\documentclass{article}

\usepackage{arxiv}

\usepackage[utf8]{inputenc} 
\usepackage[T1]{fontenc}    
\usepackage{hyperref}       
\usepackage{url}            
\usepackage{booktabs}       
\usepackage{amsfonts}       
\usepackage{nicefrac}       
\usepackage{microtype}      
\usepackage{lipsum}

\usepackage{mathrsfs,amssymb,graphicx,multirow,color,amsthm}
\usepackage{amsmath}
\usepackage{algorithm}
\usepackage{algorithmic}  
\usepackage{multirow}
\usepackage{color}
\usepackage{ulem}
\usepackage{caption}

\newtheorem{lemma}{Lemma}
\newtheorem{theorem}{Theorem}
\newtheorem{definition}{Definition}
\newtheorem{proposition}{Proposition}
\newtheorem{corollary}{Corollary}
\newtheorem{example}{Example}

\newcommand{\diag}{{\rm diag}}
\newcommand{\spann}{{\rm span}}
\newcommand{\tr}{{\rm tr}}
\newcommand{\rank}{{\rm rank}}
\newcommand{\sgn}{{\rm sign}}

\newcommand{\I}{{\cal I}}
\newcommand{\J}{{\cal J}}
\newcommand{\K}{{\cal K}}
\newcommand{\R}{{\cal R}}
\renewcommand{\S}{{\cal S}}
\newcommand{\grad}{{\rm grad}}

\title{Minimal Sample Subspace Learning: Theory and Algorithms}

\author{
  Zhenyue Zhang\\
  School of Mathematics Science \\
  Zhejiang University, Yuquan Campus \\
  Hangzhou 310027, China \\
  \ \\
  Zhejiang Laboratory \\
  1818 Wenyixi Road \\
  Hangzhou 311122, China \\
  \texttt{zyzhang@zju.edu.cn} \\
   \And
  Yuqing Xia \\
  School of Mathematics Science \\
  Zhejiang University, Yuquan Campus \\
  Hangzhou 310027, China\\
  \texttt{xiayq@zju.edu.cn} \\
}

\begin{document}
\maketitle

\begin{abstract}
Subspace segmentation, or subspace learning, is a challenging and complicated task in machine learning. This paper builds a primary frame and solid theoretical bases for the minimal subspace segmentation (MSS) of finite samples. The existence and conditional uniqueness of MSS are discussed with conditions generally satisfied in applications. Utilizing weak prior information of MSS, the minimality inspection of segments is further simplified to the prior detection of partitions. The MSS problem is then modeled as a computable optimization problem via the  self-expressiveness of samples. A closed form of the representation matrices is first given for the self-expressiveness, and the connection of diagonal blocks is addressed. The MSS model uses a rank restriction on the sum of segment ranks. Theoretically, it can retrieve the minimal sample subspaces that could be heavily intersected. The optimization problem is solved via a basic manifold conjugate gradient algorithm, alternative optimization and hybrid optimization, therein considering solutions to both the primal MSS problem and its pseudo-dual problem. The MSS model is further modified for handling noisy data and solved by an ADMM algorithm. The reported experiments show the strong ability of the MSS method to retrieve minimal sample subspaces that are heavily intersected. 
\end{abstract}

\keywords{  Subspace learning \and Clustering \and Rank restriction \and Sparse optimization \and Self-expressiveness}

\section{Introduction} \label{sec:intro}

Given a collection of vectors sampled from the union of several unknown low-dimensional subspaces that might intersect with each other,  subspace learning, or subspace segmentation, aims to partition the samples into several segments such that each segment contains samples within the same subspace. If the segmentation is correct, the unknown subspaces are estimated well by the segments. The problem of subspace segmentation occurs in several applications. For instance, in single rigid motion, the trajectories of feature points lie in an affine subspace with a dimension of at most 3 \cite{costeira1998a}.\footnote{A $d$-dimensional affine subspace can be embedded into a $(d+1)$-dimensional linear subspace by adding 1 as a new entry to each vector.} Moreover, facial images of an individual under various lighting conditions lie in a linear subspace of dimension up to 9 \cite{basri2003lambertian}. Detecting multiple rigidly moving objects from videos, and recognizing multiple individuals from facial images are potential subspace segmentation tasks.

Algorithms for subspace segmentation can be traced back to early studies on algorithms such as RANSAC \cite{Fischler1981Random}, K-subspace \cite{Bradley2000k,Tseng2000Nearest}, and generalized principal component analysis (GPCA, \cite{vidal2005generalized}, \cite{ma2008estimation}). In recent years, self-expressiveness methods such as low-rank representation (LRR, \cite{LRR2013}) and sparse subspace clustering (SSC, \cite{Elhamifar2013Sparse}) have attracted a great deal of attention because of their state-of-art empirical performance. Given a set of column vectors $X$ sampled from the union of several subspaces, such algorithms search for a representation matrix $C$ of $X$, that is, $X=XC$, and try to detect the subspaces based on $C$. Ideally, the representation matrix $C$ is block-diagonal as that $C = \diag(C_1,\cdots,C_K)$ under permutation. In this case, the samples are also partitioned into $K$ pieces such as $X=[X_1,\cdots,X_K]$, and each $X_k$ contains samples from the same subspace. 

To estimate such a representation with the block-diagonal structure in $C$ as much as possible, the SSC minimizes the $\ell_1$-norm of $C$ as follows:
\begin{align}\label{opt:SSC}
	{\rm (SSC)}	\quad \min_C \|C\|_1 \quad \mbox{s.t.}\ X=XC, \ \diag (C)=0.
\end{align}
The restriction on the diagonals avoids a trivial and meaningless solution. The LRR method searches for a representation matrix with approximate low-rank structure that can take the so-called `global structure' of the samples into account. Thus, it minimizes
the nuclear norm $\|C\|_*$ (sum of the singular values) of $C$ as follows:
\begin{align}\label{opt:LRR}
	{\rm (LRR)} \quad	
	\min_C \|C\|_* \quad \mbox{s.t.} \ X=XC.
\end{align}
The main difference is that SSC searches for the sparsest nontrivial representation matrix, while LRR searches for the representation matrix with the lowest rank. SSC and LRR seemingly impede subspace retrieval at two ends: an overly sparse solution may be block-diagonal with a greater-than-expected number of diagonal blocks, and a solution with an overly low-rank may not be block-diagonal or have less blocks. In these two cases, the ground-truth subspaces cannot be detected via classical spectral clustering. To control the number of blocks, \cite{wang2013provable} combines the two objective functions. More purposefully, \cite{XiaZhangMVAP} combines the $\ell_1$-norm function with a logarithmic-determinant function to balance the sparsity and rank of the solution. 
\cite{Li2017Structured} modifies the $\ell_1$-norm function $\|C\|_1$ to the $\ell_1$-norm of the off-diagonal blocks of $C$ with a partition that should also be optimized. This strategy may help to increase the connection of the diagonal blocks in some sense.\footnote{ We say that $C_k$ is connected if the undirected graph constructed from $|C_k|+|C_k|^T$ is connected.}
Since the number of zero eigenvalues of the Laplacian matrix of a block-diagonal matrix is equal to the number of connected diagonal blocks \cite{vonluxburg2007a}, \cite{BDR2018} minimizes several of the smallest eigenvalues of the Laplacian matrix of $C$ to achieve a block-diagonal solution with connected diagonal blocks.

The effectiveness of these methods has scarcely been exploited in the literature. For instance, \cite{Soltanolkotabi2012A} gave sufficient conditions for SSC to retrieve a representation matrix that can detect subspaces. \cite{LRR2013} proved that LRR can recover mutually independent subspaces.\footnote{LRR cannot recover dependent subspaces; see Theorem \ref{thm:LRR-SSC} in Subsection \ref{subsect:LRR-SSC}.} The above modified methods require equivalent or stronger conditions, which are generally very strict in applications. 

The latent subspaces we wish to retrieve from samples are not well-defined mathematically, which may explain why theoretical progress in subspace learning has been slow. Subspace segmentation is practically ambiguous and unidentified in the literature. It is also highly possible that the segmented subspaces found by an algorithm may be defined by the algorithm used. For instance, in SSC, segmenting $X$ into $\{X_k\}$ is equivalent to separating a constructed graph $A$ into connected subgraphs $\{A_k\}$ via the following procedure: 
Let $c_i$ minimize $\|c\|_1$ subjected to $x_i = X_{(i)}c$, where $X_{(i)}$ is the $X$ whose $i$-th column is reset to zero, $i=1,\cdots,n$. Graph $A$ takes $\{x_i\}$ as its nodes and has an edge between nodes $x_i$ and $x_j$ if the $j$-th entry of $c_i$ or the $i$-th entry of $c_j$ is nonzero. Theoretically, $C=\diag(C_1,\cdots,C_K)$ is block-diagonal of connected blocks under permutation if and only if $A$ has $K$ connected subgraphs $\{A_k\}$. In that case, $X_k$ consists of the samples as nodes involved in $A_k$. Clearly, the spanning subspaces $\{\spann(X_k)\}$ depend on the connection structure of the constructed graph and cannot be predicted. In addition, the number of subspaces cannot be predicted.
LRR gives a coarse segmentation corresponding to the independent subspaces, each a sum of several ground-truth subspaces, assuming that the ground-truth subspaces can be separated into several classes such that the subspace sums within classes are independent.\footnote{This claim can be also concluded from Theorem \ref{thm:LRR-SSC}.}

This paper aims to build a theoretical basis for subspace learning from a mathematical viewpoint. The basic, important, and key issues that we keep in mind include the following:

(1) Identifiability of the subspaces that we wish to detect from a finite number of samples. The related basic issues for noiseless samples may include the definition of subspaces that are solely determined by samples, the uniqueness of the corresponding segmentation, the sufficient conditions for uniquely identifying the segmentation, and the consistency of the defined segmentation with the groundtruth segmentation that we expect in applications.

(2) Computability of the defined subspace segmentation. For application purpose, we may be required to formulate the defined segmentation as an optimization problem that should be computable with an acceptable cost. Related issues may include the uniqueness of the solution or conditions of the uniqueness, and the ability of addressing complicated segmentation wherein subspaces intersect with each other heavily, or samples are located near such intersected subspaces.

(3) Efficient algorithms for solving the optimization problem. We may also encounter efficiency issues with the adopted algorithms, such as computational complexity and local optimums. 

(4) Stability of solutions and robustness of algorithms. It may be  difficult but absolutely worth addressing these issues to further our understanding of subspace learning. 

(5) Extension to noisy samples which may be more important in applications. Certain necessary modifications  are required to this end, together with perturbation theory on subspace segmentation. 

In this paper, we partially address the above issues. Below, we briefly describe the main contributions of this paper and our related motivations.

1. The concept  of minimal sample space is introduced and used to define a minimal sample segmentation (MSS) of a given set of samples. The existence of the MSS is guaranteed but may not be unique in some special cases; thus, we show that the MSS is conditionally unique. Two kinds of sufficient conditions for this uniqueness are given that focus on data quantity and  quality, respectively. These conditions are weak since they are always satisfied in applications with randomly chosen samples from ground-truth subspaces. Hence, the minimal sample subspaces should generally be ground-truth subspaces.

2. It is difficult to check the minimality of a segmentation. We further study how to simplify detection under following the prior information of an MSS: The number of minimal segments, the sum of the segment ranks, and the minimal rank of the segments. We focus on the set of partitions with the same number of pieces and the restrictions of rank sum and minimal rank. Conditions for the singleness of such a set are given based on discreet rank estimations on each segment. These conditions permit subspaces to be heavily intersected within reasonable sense. Singleness means that the MSS can be detected. 

3. The sufficient conditions for singleness of the above partition set are tight. We further exploit the properties of the sample segments when the sufficient conditions are incompletely satisfied, leading to two types of partition refinements under weaker conditions: Segment reduction and segment extension.  

4. Based on solid theoretical analyses, we formulate the detection of minimal subspace segmentation as a computable optimization problem that adopts the  self-expressiveness of samples. The closed-form structure of the representation matrix is given. MSS detection requires a connected and block-diagonal structure of the solution partitioned as the considered MSS. We prove that all the connected diagonal blocks are guaranteed only if the rank sum of the diagonal blocks is equal to that of the minimal sample segments. Under this restriction, the optimization problem gives a minimal subspace detectable representation (MSDR) of the MSS. 

5. The objective function of the proposed optimization problem contains discrete variables from index partition $J$ and continuous variables from the representation matrix $C$ over a nonconvex feasible domain. To solve this minimization problem, we alternatively optimize $J$ and $C$, slightly modifying $J$ to an active index set $\Omega$ and adding a penalty on the diagonals of $C$. A manifold conjugate gradient (MCG) method is used for optimizing $C$, and an update rule is given for both $\Omega$ and $J$. Combining the two types of update rules yields an alternative optimization for detecting an MSS. An equivalent pseudo-dual problem of the primal problem is further considered and solved via subspace correction. These two kinds of MSS algorithms may drop into local minima, but they seldom have the same local minimizers. Hence, alternatively using these algorithms is an efficient strategy for escaping a local minimum, yielding a hybrid optimization method for the minimal subspace segmentation.

6. We further extend the MSS optimization problem to handle noisy samples. An ADMM method is simply considered for solving this extended optimization problem, and detailed formulas are given for solving the subproblems involved in the ADMM method.

It should be pointed out that we require the sum of subspace dimensions in our sparse model. It is an additional prior as a restriction to the rank of $C$, compared with algorithms given in the literature for subspace learning. The restriction is not necessary for uniquely determining the MSS (see Theorem \ref{thm:unique2} given in Section \ref{subsection:unique} for the sufficient conditions). However, it is necessary for guaranteeing the connection of a block-diagonal $C$ (Theorems \ref{thm:unconnect} and \ref{thm:rank_connect} in Section \ref{subsection:MSDR}). It is also helpful for simplifying the detection of minimal subspace or MSS as shown in Theorem \ref{thm:minimal partition} of Section \ref{subsection:singleness}.

The remainder of the paper is organized as follows. Sections \ref{sec:min seg}-\ref{sec:algorithm} cover the analysis of noiseless samples, while Section \ref{sec:extension} discusses the extended model on noisy samples. The definition of minimal sample subspaces and discussions on uniqueness are given in Section \ref{sec:min seg}. In Section \ref{sec:equivalent}, we discuss the problem of detecting an MSS, focusing on conditions for the singleness of rank-restricted index partitions. The theoretically supported refinement of conditioned partitions is further discussed in Section \ref{sec:refine}. Based on these theoretical analyses, we model the MSS problem as a computational optimization problem in Section \ref{sec:modeling}, covering a closed form of representation matrices, the connectivity of diagonal blocks, slight modifications to the model, and a comparison with related work.  
The MSS algorithms, together with a manifold conjugate gradient method for solving the basic optimization problem, alternative optimization strategies and hybrid optimization, are given in Section \ref{sec:algorithm}. In addition, we present an extended model for handling noisy samples and a detailed ADMM algorithm for solving the optimization problem. Finally, we report our numerical results and compare our method with existing algorithms on both noiseless synthetic data and real-world data in Section \ref{sec:experiment}. Comments on further research directions are given in the conclusion section. 

\section{Minimal Subspace Segmentation}\label{sec:min seg}

The subspaces that we expect to identify from the set of a finite number of samples may be quite different from those that naturally fit these samples. This occurs when the samples from an expected subspace are exactly located in the expected subspaces' several smaller subspaces. Thus, the basic issue for subspace learning is: what subspaces can we reasonably expect based on the given data points?

In this section, we introduce the concept of a minimal sample subspace and use it to define a segmentation of samples termed minimal subspace segmentation.
For the sake of discussion, we refer to $X$ as a data matrix consisting of data points $\{x_i\}$ as its columns, which is also referred to as the set of the data points. In the following discussion, $r(X)$ and $n(X)$ refer to the rank and column number of $X$, respectively.  

\subsection{Minimal Sample Subspace}

Naturally, the subspace spanned by a set of samples should have a smaller dimension than the sample number. Equivalently, the spanning samples are not linearly independent. We say that a sample-spanned subspace is minimal if it does not have a smaller subspace spanned by a subset of the samples. That is, any linearly dependent partial set of samples spans the same subspace. Below is an equivalent definition of the minimal sample subspace in linear algebra.

\begin{definition}\label{def:minimal-subspace}
A sample subspace $\spann(X)$ is minimal, if 
\begin{itemize}
    \item[(1)] $X$ is rank deficient, that is, $n(X)>r(X)$, and 
    \item[(2)] $X$ is nondegenerate, that is, any subset with a rank smaller than $r(X)$ is of full rank.
\end{itemize}
A sample subspace, $\spann(X)$, is pure if $X$ is of full column rank, {\it i.e.}, $n(X) = r(X)$. 
\end{definition}

Nondegeneracy specifies that $r(X')\!=\!\min\{n(X'),r(X)\}$ for any subset $X'$ of a nondegenerate $X$. Hence, any subset of a nondegenerate $X$ must also be nondegenerate. This property implies that, for a minimal subspace $\spann(X)$, any rank deficient subset $X'$ of $X$ cannot span a subspace with a smaller dimension. Equivalently, if $\spann(X)$ contains a minimal subspace $\spann(X')$, the two subspaces must be equal. 

With respect to a given data set $X$ if its spanning subspace $\spann(X)$ is neither pure nor minimal, that is, if $X$ is rank-deficient but degenerate, then there is a rank-deficient subset $X'$ of smaller rank. Thus, it makes sense to partition the data set $X$ into several nonoverlapping segments $X_0,\cdots, X_K$ such that $\spann(X_0)$ is pure and the other $\{\spann(X_k)\}$ are minimal.

\begin{definition}\label{def:partition}
A segmentation $\{X_0,X_1,\cdots,X_K\}$ of vector set $X$ is called a minimal subspace segmentation (MSS) of $X$ if 
\begin{itemize}
	\item[(1)] $\spann(X_0)$ is pure, and each $\spann(X_k)$ is minimal, $k=1,\cdots,K$, if it exists; 
	\item[(2)] $\spann(X_k) \neq \spann(X_\ell)$ for $k,\ell=0,\cdots,K$, $k\neq\ell$;
	\item[(3)] If $X_0$ exists, for any $x_j\in X_0$, $x_j\notin \spann(X_k)$, $k = 1, \cdots, K$.
\end{itemize}
\end{definition}

We also call $X_0$ a pure segment if $\spann(X_0)$ is pure. In applications, a pure segment $X_0$, if it exists, could be a set of outliers. Condition (3) is necessary since some samples may be redundant for spanning a minimal sample space.

\begin{theorem}\label{thm:min seg}
Any vector set with nonzero columns has an MSS.
\end{theorem}
\begin{proof}
The basic idea of the proof was mentioned above. 
If $\spann(X)$ is pure, we set $X = X_0$ and $K=0$. If $\spann(X)$ is minimal, we set $X_1=X$ and $K=1$, and $X_0$ disappears. In the other cases, we have a minimal subspace $\spann(X_1)$ that has the smallest dimension, where $X_1$ is a subset of $X$. $X_1$ can be the set containing all the samples belonging to the subspace $\spann(X_1)$, since adding these samples does not change the minimality of the subspace because there is no minimal subspace of lower dimension. That is, $\spann(X_1)$ remains minimal after adding samples. Repeating the above procedure on the remaining samples, we can complete the proof.
\end{proof}

However, the MSS of a given set $X$ may be not unique. The following example illustrates an example of nonuniqueness.  

\begin{example}\label{exmp:mult-MSS}
Let $X$ be the union of 4T five-dimensional vectors in the pieces
\[
	X_{1,j} = \left[\begin{array}{cc} a_{j,1} & a_{j,2} \\ b_j & b_j \\ o & o\end{array}\right],\quad
	X_{2,j} = \left[\begin{array}{cc} a_{j,1} & a_{j,2} \\ o & o\\ b_j & b_j \end{array}\right] \quad{\rm with }\quad
	o = \left[\begin{array}{c} 0 \\ 0 \end{array}\right],\quad j=1,\cdots,T,
\]
where the scales $a_{j,1}$ and $a_{j,2}$ and the vector $b_j\in R^2$ are arbitrarily chosen such that $a_{j,1}\neq a_{j,2}$, each pair $(b_i,b_j)$ pair is linearly independent, and $\left[\begin{array}{ccc}a_{i,s_i} &a_{j,s_j} &a_{k,s_k}\\ b_i &b_j & b_k\end{array}\right]$ is of full rank for different $i,j,k$ and any $s_i,s_j,s_k=1,2$.
Then $X$ has two types of segmentations without a pure segment,
\begin{itemize}
    \item[(1)] $X_1' = [X_{1,1},\cdots,X_{1,T}]$, $X_2' = [X_{2,1},\cdots,X_{2,T}]$, $K$ = 2;
    \item[(2)] $X_k'' = [X_{1,k},X_{2,k}]$, $k = 1,\cdots,K=T$.
\end{itemize}
Here, each $\spann(X_k')$ or $\spann(X_k'')$ is minimal. Hence, both $\{X_k'\}$ and $\{X_k''\}$ are minimal. 
\end{example}

This example partially explains why subspace learning is complicated. First, a sample set may have multiple segmentations, and each is an MSS. Second, the segments of an MSS may be very small. A small minimal segment $X_k$ may have the smallest number of samples needed to span a minimal subspace, {\it i.e.}, $n(X_k)=r(X_k)+1$. Obviously, if each segment is small in an MSS, this MSS may have a large number of segments. That is, the samples can be clustered into many small classes. Third, two different MSSs may have an equal number of segments with equal ranks. This case occurs when $T=2$, where each segment has rank $3$ with $4$ samples. 

Fortunately, the MSS is generally unique in applications. In the next subsection, we discuss the conditions of uniqueness.

\subsection{Uniqueness of Minimal Subspace Segmentation}\label{subsection:unique}

The following condition is obviously necessary for a unique MSS of $X=[x_1,\cdots,x_n]$. 
\begin{align}\label{non-intersection}
	x_j \notin \spann(X_k) \cap \spann(X_\ell),\quad  \forall j\  \mbox{and} \ \forall k\neq \ell.
\end{align}
Otherwise, a sample belonging to the intersection of two spanning subspaces could be arbitrarily assigned to any one of the two sample sets spanning the subspaces. In this subsection, we describe two types of sufficient conditions that guarantee the uniqueness of an MSS based on either sample quantity or  quality.

\begin{theorem}\label{thm:unique1}
If $X$ has an MSS with $K$ minimal sample subspaces satisfying the condition (\ref{non-intersection}) and $n(X_k)>(K+1)(r(X_k)-1)$, then a different MSS of $X$ satisfying (\ref{non-intersection}), if it exists, must have more minimal sample subspaces.
\end{theorem}
\begin{proof}
The theorem is obviously true if $\spann(X)$ is pure. If $\spann(X)$ is not pure, which implies $K\geq 1$, and there is another MSS $\{X_0',X_1',\cdots,X_{K'}'\}$ of $X$ satisfying (\ref{non-intersection}) with $K'\leq K$, let 
\[
    n_1 = n(X_1),\quad 
	r_1 = r(X_1),\quad
	X_{1,k}=X_1\cap X'_k,\quad
	n_{1,k} = n(X_{1,k}),\quad
	k_1 = \arg\max_{k\leq K'} n_{1,k}.
\] 
Because $(K+1)(r_1-1) < n_1 = \sum_{k=0}^{K'} n_{1,k}\leq (K'+1)n_{1,k_1} \leq (K+1)n_{1,k_1}$, we have $n_{1,k_1}  \geq r_1$. Hence, $\spann(X_1)=\spann(X_{1,k_1}) \subset \spann(X_{k_1}')$. By (\ref{non-intersection}), $X_{k_1}' = X_1$. Deleting $X_1$ from $X$, the remainder samples have two minimal segmentations $\{X_0,X_2,\cdots,X_K\}$ and $\{X_0',\cdots,X_{k_1-1}',X_{k_1+1}',\cdots,X_{K'}'\}$. Using induction on $K$, these two minimal segmentations should be equal. Hence, the theorem is proven.
\end{proof}

Theorem \ref{thm:unique1} basically says that an MSS is unique if it is `fat', that is, each segment has enough samples. Multiple minimal subspace segmentations may exist only if some segments have a small number of samples compared with the number of segments. Among those minimal segments with few samples, a union of partial samples from different segments can also form a new minimal segment. This may be the main reason for the multiplicity of minimal subspace segmentations. However, this multiplicity will disappear if the samples are well-distributed. Below, we introduce a new concept to define such a good distribution.  For the sake of simplicity, $X\!\setminus\!X'$ indicates the remaining samples of $X$ when the subset $X'$ is removed.

\begin{definition}\label{def:non-degenerate:int}
A segmentation $\{X_0,X_1,\cdots,X_K\}$ is intersected nondegenerately, if for any subset $X_k'$ of $X\!\setminus\!X_k$, the splitting 
\[
    X_{k} = Y_k+Z_{k},\quad Y_k\subset\S_k',\quad Z_k\subset \S_k''
\]
according to the direct sum $\spann(X_k) = \S_k'\oplus \S_k''$, always gives a zero or nondegenerate $Z_k$,
where $\S_k' = \spann(X_k)\cap\spann(X_k')$, and $\S_k''$ is the orthogonal complement of $\S_k'$ restricted in $\spann(X_k)$.
\end{definition}

\begin{figure}[t]
\centering
\includegraphics[width=2.5in]{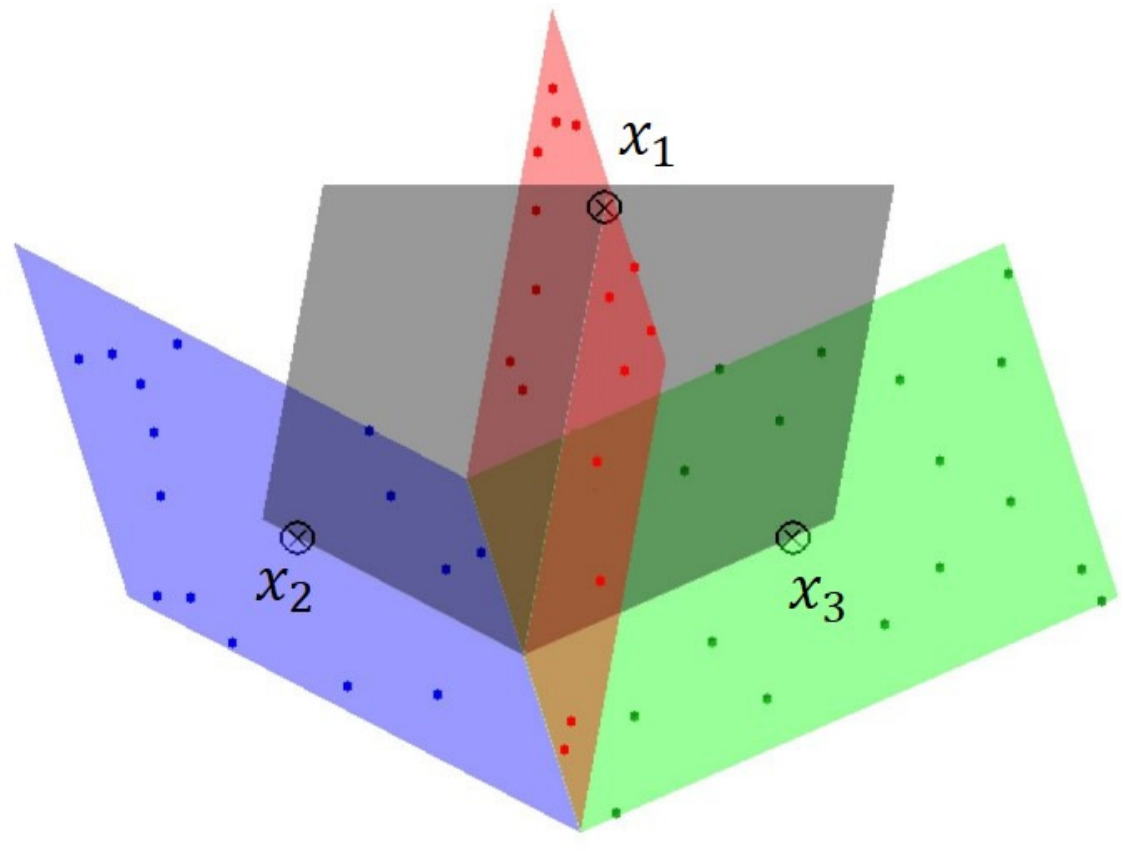}
\caption{Illustration of nondegenerate intersection. The segmentation of samples according to the blue, red, and green planes is intersected degenerately because of the points $x_1,x_2,x_3$ that are sampled from the three planes but can span a new minimal subspace marked in gray. After deleting these three points, the segmentation is intersected nondegenerately.}
\label{fig:inter_degenerate}
\end{figure}

Figure \ref{fig:inter_degenerate} illustrates the nondegenerate intersection, in which the minimal segmentation becomes to be intersected degenerately if we add three special points $x_1$, $x_2$, and $x_3$ into the three planes. In this case, the whole sample set has a new segmentation with 4 minimal segments; However the intersection of two segments may be not empty. In addition, after merging the new points into the original three segments, respectively, the three extended segments also form an MSS. This illustration is limited because of its low dimension, in which the newly added sample $x_i$ belongs to an intersection of two minimal segments. This intersection phenomenon should be removed in uniqueness analysis as we assumed in (\ref{non-intersection}). In a higher dimensional space, some MSSs may be intersected degenerately, although the intersection phenomenon in Figure \ref{fig:inter_degenerate} does not occur. For example, (\ref{non-intersection}) is satisfied for the degenerately intersected  segmentation $\{X''_k\}$ given in Example \ref{exmp:mult-MSS}.

\begin{theorem}\label{thm:unique2}
If $X$ has a nondegenerately intersected MSS satisfies (\ref{non-intersection}), then its MSS satisfying (\ref{non-intersection}) is unique.
\end{theorem}
\begin{proof}
We use the same notation as that used in the proof of Theorem \ref{thm:unique1}. Assume that $X$ has a nondegenerately intersected MSS $\{X_0, \cdots, X_{K}\}$ and another MSS $\{X_0', \cdots, X_{K'}'\}$, and both satisfy (\ref{non-intersection}). We first show that there is a segment $X_{k_1}'$ equal to $X_1$.  Uniqueness is then achieved after applying the method of induction to the number of segments since $\{X_0, X_2,\cdots, X_{K}\}$ is a minimal segmentation of the remaining samples.
To this end, we consider $n_{1,k_1} = \max_{k\neq 0} n_{1,k}$ that is most possible for the equality $X_{k_1}'=X_1$ since $X_{k_1}'$ has the largest intersection with $X_1$. For the sake of simplicity, we can assume that $k_1=1$. The equality $X_1'=X_1$ holds if $\spann(X_1')\subseteq \spann(X_1)$ or $\spann(X_1)\subseteq \spann(X_1')$ by (\ref{non-intersection}).  

Assume $\spann(X_1')\nsubseteq \spann(X_1)$ and $\spann(X_1)\nsubseteq \spann(X_1')$ conversely. That is, both $X_1^c= X_1'\!\setminus\! X_{1,1}$ and $X_1\!\setminus\! X_{1,1}$ are not empty. Obviously, $X_1\nsubseteq \S_0'=\spann(X_1^c)$.
Hence, splitting $X_1$ according to the direct sum $\spann(X_1) = \S' +\S''$, where $\S'=\S_0'\cap \spann(X_1)$ and $\S''$ is the orthogonal complement of $\S'$ restricted in $\spann(X_1)$, we can rewrite $X_1 = Y_1+ Z_1$, where $Y_1\subset \S'$ and nonzero $Z_1\subset \S''$ that should be nondegenerate since $\{X_0, \cdots, X_K\}$ is intersected nondegenerately. Thus, in the splitting $X_{1,1} = Y_{1,1}+ Z_{1,1}$ of the subset $X_{1,1}$ of $X_1$, where $Z_{1,1}\subset Z_1$, the nondegeneracy of $Z_1$ gives that $r(Z_{1,1}) = \min\big\{n_{1,1},r(Z_1)\big\}$. However, whether $r(Z_{1,1})=r(Z_1)$ is true or not, it always leads to a contradiction as shown below. 
If $r(Z_{1,1})=r(Z_1)$, then $\spann(Z_1)=\spann(Z_{1,1})\subset\spann(X_{1,1})\subset \spann(X_1')$, and we get $\spann(X_1)\subseteq \spann(Y_1)+\spann(Z_1) = \S'+\spann(X_1')\subseteq\spann(X_1')$, a contradiction of the hypothesis $\spann(X_1)\nsubseteq \spann(X_1')$. If $r(Z_{1,1})\neq r(Z_1)$, then $r(Z_{1,1})=n_{1,1}$ and 
\[
    r(Y_{1,1})+r(Z_{1,1})=r(X_{1,1})\leq n_{1,1} = r(Z_{1,1}).
\]
Hence, $Y_{1,1}=0$, {\it i.e.}, $X_{1,1}=Z_{1,1}\subseteq S''\,\bot\,\big(\S_0'\cap\spann(X_{1,1}) \big)$. We conclude that 
$\spann(X_{1,1})$ is orthogonal to $\S_0'=\spann(X_1^c)$. Thus, $r(X_1^c) < r(X_1')$. By the minimality of $X_1'$, $X_1^c$ is of full column rank, and
$n_1>r(X_1') = r(X_{1,1})+r(X_1^c) = n_{1,1}+n(X_1^c) = n_1$, which is also
a contradiction.
\end{proof}

The sample quantity condition of Theorem \ref{thm:unique1} is generally satisfied in many applications since the number of subspaces that we want to be recognized is quite small, compared with the number of samples. In addition, the sample quality condition of Theorem \ref{thm:unique2} is also satisfied with probability 1 if the samples are randomly chosen from given subspaces.

\begin{theorem}\label{thm:ndgt}
Given different subspaces $\S_1,\cdots,\S_K$, if the columns of $X_k$ are randomly chosen from $\S_k$ with $n_k>\dim(\S_k)$ for $k=1,\cdots,K$, then $\{X_1,\cdots,X_K\}$ is intersected nondegenerately and (\ref{non-intersection}) is satisfied with probability 1.
\end{theorem}
\begin{proof}
The condition (\ref{non-intersection}) is obviously satisfied with probability 1. Let $U_k$ be an orthogonal basis matrix of $\S_k$, and let $X_k = U_kH_k$. By the assumption, $H_k$ is a random matrix whose entries are i.i.d.
To show the nondegenerate intersection of $\{X_0,X_1,\cdots,X_K\}$, we consider an arbitrary subset $X_k'$ of $X_k^c = X\!\setminus\!X_k$, and the splitting 
\[
    X_k = Y_k+Z_k,\quad
    Y_k\subset\S_k' = \S_k\cap\spann(X_k'),\quad
    Z_k\subset \S_k''
\]
where $Z_k\neq 0$ and $\S_k''$ is the orthogonal complement of $\S_k'$ restricted to $\S_k$. Let $Y_k=U_k'H_k'$ and $Z_k=U_k''H_k''$, where $U_k'$ and $U_k''$ are orthogonal basis matrices of $\S_k'$ and $\S_k''$, respectively. We have $U_k = [U_k',U_k'']$ and $H_k^T = [H_k'^T,H_k''^T]$. Since $H_k$ is a random matrix whose entries are i.i.d., so is $H_k''$. The entry distribution implies that $H_k''$ is nondegenerate with probability 1 since a matrix with i.i.d. entries is full rank with probability 1. Therefore, $Z_k = U_k''H_k''$ is also nondegenerate with probability 1. Hence, the proof is completed.
\end{proof}

Generally, the pure segment $X_0$ vanishes in applications. The following corollary further shows that if the samples are randomly chosen from the union of subspaces $\S_1,\cdots,\S_K$, then these subspaces are just the unique minimal sample subspaces of the samples with probability 1.

\begin{corollary}
Assume that the columns of $X_k$ are randomly sampled from subspace $\S_k$ and $n_k>\dim(\S_k)$ for $k\leq K$. Then, $\{X_1, \cdots, X_K\}$ is the unique minimal segmentation with probability 1.
\end{corollary}
\begin{proof}
By Theorem \ref{thm:ndgt}, $\{X_1,\cdots, X_K \}$ is intersected nondegenerately with probability 1. Hence, with this probability, each $X_k$ is non-degenerate, and hence, $\spann(X_k) = \S_k$ is a minimal sample subspace, that is, $\{X_1 ,\cdots, X_K\}$ is minimal by definition. Since (\ref{non-intersection}) is also satisfied with probability 1, by Theorem \ref{thm:unique2}, this MSS is unique with probability 1.
\end{proof}

In summary, the MSS of a given a set of finite samples always exists. It is possible to have multiple MSSs, but a fat MSS is unique, as shown in Theorem \ref{thm:unique1}. Furthermore, if the samples are well-distributed, only one MSS exists, as shown in Theorem \ref{thm:unique2}. In applications, samples from ground-truth subspaces are generally well-distributed or the MSS is fat. Therefore, the MSS is unique and  generally represents the ground-truth. However,  segmentation minimality is  extremely difficult to  confirm. In the next section, we show how to detect the minimality in a relatively simple way, which provides an insight for MSS detection. It is very helpful for modeling  minimality as an optimization problem so that we can practically determine the minimal subspace segmentation via solving  the optimization problem.

\section{Detection of Minimal Subspace Segmentation}\label{sec:equivalent}

Clearly, it is impractical to inspect the minimality of a given segmentation $\{X_{J_k}\}$  by checking  whether each segment is nondegenerate or not, where $X_{J_k} = X(:,J_k)$. Notice that we have used the notation $X_k$ for the minimal segment for the sake of simplicity, {\it i.e.}, $X_k=X_{J^*_k}$ with the index set $J_k^*$ of $X_k$. Fortunately, this complicated task can be relatively simplified if we have a little prior information on the MSS.

The insight  for the detection of MSS is that prior information on the MSS may narrow the set of segmentations,  and thus enabling relatively easy detection. To this end, and also for the sake of simplicity, we assume that MSS $\{X_k\}$ does not have a pure segment $X_0$ and is intersected nondegenerately.  Thus, MSS $\{X_k\}$ is unique  according to Theorem \ref{thm:unique2}. Obviously, there are at least three necessary conditions for segmentation $\{X_{J_k}\}$ to be the MSS:

(a) Its segment number  equals the number of the minimal segments;

(b)  The rank sum of its segments is not larger than the rank sum $d = \sum_kr(X_k)$;

(c) Each segment size  is larger than the smallest rank $d_{\min}=\min_kr(X_k)$.

Here the rank sum $d$ is equal to the dimension sum of the minimal subspaces. We  use the prior information to narrow the feasible domain of the MSS to the subset of those satisfying the above three restrictions.  Equivalently, we focus the index partitions in the following set:
\begin{align}\label{J}
    {\cal J}(K,d,d_{\min}) = \big\{J=\{J_1,\cdots,J_K\}: \ \min |J_k| > d_{\min}, \ r(X_{J_1})+\cdots+r(X_{J_K}) \leq d\big\}.
\end{align}
Obviously, index partition $J^* = \{J_k^*\}$ of the MSS $\{X_k\}$ belongs to ${\cal J}(K,d,d_{\min})$. If ${\cal J}(K,d,d_{\min}) $ contains only one partition, the detection of the MSS becomes to simply check whether a partition has only $K$ pieces and if the two conditions 
\[
    r(X_{J_1})+\cdots+r(X_{J_K}) \leq d,\quad
    \min |J_k| > d_{\min}
\]
are satisfied. Hence, the relevant question is: Could ${\cal J}(K,d,d_{\min})$ be a singleton? 

We will give a positive answer to this question under weak conditions shown later. To this end, let $\S_k =  \spann(X_k)$, $d_k = \dim(\S_k) = r(X_k)$, $n_k = n(X_k)$ for $k=1,\cdots,K$, and 
\begin{align}\label{def:ds}
	d_{\min}=\min_k d_{k},\quad
	d_0 = \max_k \dim\big( \S_{k} \cap \sum_{j \neq k} \S_{j} \big), \quad
	d_{\rm int} = \max_{k\neq j}\dim( \S_{k}\cap \S_j). 
\end{align}
Hence, $d = \sum_k d_k$.
We say that $J = \{J_k\}$ is a minimal partition if $\{X_{J_k}\}$ is an MSS of $X$. Example \ref{exmp:mult-MSS} shows that ${\cal J}(K,d,d_{\min})$ may have multiple minimal partitions in  special cases. To guarantee a single minimal partition in ${\cal J}(K,d,d_{\min}) $,  certain conditions  must be met. In the next subsection, we offer some sufficient conditions  that guarantee the singleness of ${\cal J}(K,d,d_{\min}) $. We may use the assumptions if necessary.
\begin{align}\label{assumption}
	(i) \ x_j \notin \S_k\cap \S_\ell,\quad \forall j\leq n,  k \neq \ell;  \  
	(ii)\ \{X_1,\cdots,X_K\} \ \mbox{is intersected nondegenerately.}
\end{align}
These sufficient conditions are tight. We will give some counterexamples  in which one of  the sufficient conditions is not satisfied and  further discuss how to refine $J\in {\cal J}(K,d,d_{\min}) $ in these cases.

 The number $K$ and dimension sum $d$ of minimal subspaces are generally known in applications. The smallest dimension $d_{\min}$ may also be known if the minimal subspaces have equal dimensions. In the computational model given later, we assume that $K$, $d$, and $d_{\min}$ are known.  However, the minimal dimension restriction is relaxed in our subsequent algorithms.

\subsection{Conditions of Singleness}\label{subsection:singleness}

Our analysis on the singleness of ${\cal J}(K,d,d_{\min})  $ is based on a discreet estimation on the rank of each segment $X_{J_k} $ for a given partition $J=\{J_k\}\in{\cal J}(K,d,d_{\min}) $. The simple equality for matrix partition $A=[B,C]$
\begin{align}\label{ABC}
	r(A) = r(B)+r(C)-\dim(\spann(B)\cap\spann(C))
\end{align}
will be repeatedly used in the rank estimation. 
For the sake of simplicity, let $t_+ = \max\{0,t\}$ and 
\[
	J_{ik} = \{j\in J_i: x_j\in X_k\},\quad |J_{ik}| = n(X_{J_{ik}}), 
\]
and let $\S\setminus\S'$ be the orthogonal complement of $\S'$ restricted in $\S$ for subspace $\S'$ of $\S$.

\begin{lemma}\label{lma:rank} 
Let $J=\{J_k\}\in{\cal J}(K,d,d_{\min}) $. If $x_j \notin \S_k\cap \S_\ell $ for all $j$ and $k \neq \ell$, then for any $s\neq t$ and $J_i$,
\begin{align}\label{rank:X_Ji}
	r(X_{J_i})\geq \min\{|J_{is}|,d_s\}+\big(\min\{|J_{it}|,d_t\}-d_{\rm int}\big)_++\sum_{k\neq s,t}\big(\min\{|J_{ik}|,d_{k}\}-d_0\big)_+.
\end{align}
Furthermore, if $\{X_k\}$ is intersected nondegenerately, then 
\begin{itemize}
	\item[(a)] \hspace{2pt}$r(X_{J_i})\geq d_{\min}$ for those $J_i$ having a single nonempty piece $J_{is}$. 
	\item[(b)] $r(X_{J_i})> d_{\min}$ for those $J_i$ having at least two nonempty pieces $J_{is}$ and $J_{it}$.
\end{itemize}
\end{lemma}
\begin{proof}
For the sake of simplicity, let $X_k' = X_{J_{ik}}$. We prove (\ref{rank:X_Ji}) with $(s,t) =(1,2)$ only since one can reorder $\{X_k'\}$ to have $X_s'$ and $X_t'$ as the first two segments in the general case. 
To this end, we merge the first $k$ pieces to 
$M_k =[X_1',\cdots,X_k']$ and let $\S_k' = \spann(M_{k-1})\cap\spann(X_k')$. By (\ref{ABC}), we have the following recursion: 
\begin{align}\label{rank:Mk}
	r(M_k) =r(M_{k-1})+\zeta_k,
\end{align}
where $\zeta_k = r(X_k')-\dim(\S_k')$.
Since $X_k$ is nondegenerate, $r(X_k')= \min\{|J_{ik}|,d_{k}\}$. Combining this with $\dim(\S_k')\leq d_{\rm int}$ for $k=2$ or $\dim(\S_k')\leq d_0$ for $k>2$, we have the following estimate:
\begin{align}\label{bound:zeta_k}
 	\zeta_k\geq \left\{
	\begin{array}{ll}
		 \big(\min\{|J_{ik}|,d_{k}\}-d_{\rm int}\big)_+, & k = 2;\\
		 \big(\min\{|J_{ik}|,d_{k}\}-d_0\big)_+, & k > 2.
	\end{array}\right.
\end{align}
Thus, taking the sum of all the equalities in (\ref{rank:Mk}) and using (\ref{bound:zeta_k}), we get (\ref{rank:X_Ji}) with $(s,t) =(1,2)$.

We further show that $\zeta_k$ can be represented with $\S_k''= \S_k \cap \spann(M_{k-1})$ as
\begin{align}\label{zeta_k}
	\zeta_k=\min\{|J_{ik}|, d_k-\dim(\S_k'')\}.
\end{align}
based on the nondegeneracy of the intersection of $\{X_k\}$.
To this end, we split $X_k' = Y_k'+Z_k'$ with $Y_k' \subset \S_k'$ and $Z_k' \subset \spann(X_k') \setminus \S_k'$, and rewrite $M_k = [M_{k-1}, Y_k']+[0, Z_k']$. 
By (\ref{ABC}), and $\spann([M_{k-1}, Y_k'])\cap\spann(Z_k') = \{0\}$, we also obtain (\ref{rank:Mk}) with $\zeta_k = r(Z_k')$ since
\[
	r(M_k) = r([M_{k-1},Y_k'])+r(Z_k') = r(M_{k-1})+r(Z_k').
\]
To estimate the rank of $Z_k'$, we extend the splitting of $X_k'$ to $X_k = Y_k+Z_k$ with $Y_k \subset \S_k''= \S_k\cap\spann(M_{k-1})$ and $Z_k \subset \S_k\!\setminus \S_k''$. Obviously, $r(Y_k)\leq\dim(\S_k'')$ 
and $r(Z_k) \leq d_k-\dim(\S_k'')$. These equalities should hold since $r(Y_k)+r(Z_k) = d_k$. Furthermore, $Z_k$ is nondegenerated or $Z_k=0$ by the nondegenerate intersection of $\{X_1, \cdots, X_K\}$.
Thus, as a column submatrix of $Z_k$, $Z_k'$ should have the rank 
$\min\{|J_{ik}|,r(Z_k)\} = \min\{|J_{ik}|, d_k-\dim(\S_k'')\}$. This is (\ref{zeta_k}). 

We now prove (a) and (b) of this lemma, using (\ref{rank:Mk}) and $r(X_{J_i}) = r(M_K) = \sum_{\ell=1}^{K} \zeta_k$, comparing $|J_{ik}|$ and $\delta_k = d_k-\dim(\S_k'')$ for determining $\zeta_k$ by its definition (\ref{zeta_k}).

(1) 
If $|J_{ik}|\leq \delta_k$ for all $k$, then $\zeta_k=|J_{ik}|$. Hence, 
$r(X_{J_i})=\sum_{k=1}^{K}|J_{ik}| = |J_i|> d_{\min}$ since $J=\{J_k\}\in{\cal J}(K,d,d_{\min}) $. 

(2) 
If $|J_{i1}|>\delta_1$, and $|J_{ik}|\leq \delta_k$ for $k\geq 2$,
then $\zeta_1 = \delta_1=d_1$ and $\zeta_k = |J_{ik}|$ for $k\geq 2$. Hence, 
$r(X_{J_i}) \geq d_1 +\sum_{k=2}^K|J_{ik}|$. 

(3) 
If $|J_{ik}|> \delta_k$ for a $k\geq 2$, $\zeta_k = \delta_k=d_k-\dim(\S_k'')$. Since $\spann(M_{k-1})\nsubseteq\S_k$, we have $\S_k''\subsetneq\spann(M_{k-1})$, {\it i.e.}, $r(M_{k-1})>\dim(\S_k'')$.
By (\ref{rank:Mk}), $r(X_{J_i})\geq r(M_k) = r(M_{k-1})+\zeta_k > d_k$. 

Hence, in each of the above cases, (a) and (b) are always true.
\end{proof}

The following lemma further shows that if each minimal segment $X_k$ has  a sufficient number of samples, it must be dominated by one piece of any $J\in{\cal J}(K,d,d_{\min}) $, in the sense that there exists at least one subset $J_{ik}$ whose size is not smaller than $d_k$. We will use this lemma to prove the singleness of ${\cal J}(K,d,d_{\min}) $.

\begin{lemma}\label{lma:large J_ik}
If $x_j \notin \S_k\cap \S_\ell $ for all $j$ and $k \neq \ell$, and $n_k > d_k+(K-1)d_0 $ for all $k$, then for $J\in{\cal J}(K,d,d_{\min}) $
\[
	\max_i |J_{ik}|\geq d_k,\quad k=1,\cdots, K.
\]
\end{lemma}
\begin{proof}
Let $\K = \{k : \max_i |J_{ik}| \geq d_k\}$. This lemma is equivalent to saying that $\K = \{1,\cdots,K\}$.  we can prove this by letting $\I = \{i: \max_k |J_{ik}| \geq d_k\}$ and $\K_i = \{k : |J_{ik}| \geq d_k\}$ for each $i\in\I$. Then, $\K = \cup_{i\in \I} \K_i$. If $\K^c$ is not empty, we choose an $s\in\K^c$ and any $t\neq s$ in (\ref{rank:X_Ji}) of Lemma \ref{lma:rank} and use $d_{\rm int}\leq d_0$ to  obtain the following:
\[
	r(X_{J_i}) \geq \left\{\begin{array}{ll}
   	d_0 +\sum_{k \in \K^c}(|J_{ik}|-d_0)+ \sum_{k \in \K_i} (d_k-d_0)_+, & i\in \I; \\
    	d_0 + \sum_{k \in \K^c}(|J_{ik}|-d_0),& i\in \I^c.
    	\end{array}\right.
\]
Hence,
$	d=\sum_ir(X_{J_i})
    	 \geq Kd_0 +\sum_{k\in\K^c}\sum_{i}(|J_{ik}|-d_0)+ \sum_{i\in \I}\sum_{k\in \K_i}(d_k-d_0)_+ .
$
In the second term, $\sum_{i}(|J_{ik}|-d_0)\geq n_k-Kd_0>d_k-d_0$.
Since $\I = \cup_{k\in\K} \I_k$ with $\I_k = \{i : |J_{ik}| \geq d_k\}$, the last term becomes as follows:
\begin{align}\label{eq:ik}
	\sum_{i\in \I}\sum_{k\in \K_i}(d_k-d_0)_+ =
	\sum_{k\in \K}\sum_{i\in \I_k} (d_k-d_0)_+
	= \sum_{k\in \K}|\I_k| (d_k-d_0)_+\geq \sum_{k\in \K}(d_k-d_0).
\end{align}
Thus, $d>Kd_0 +\sum_{k}(d_k-d_0) = d$, which is
a contradiction. Therefore, $\K^c$ must be empty.
\end{proof}

 We are now ready to prove the singleness of ${\cal J}(K,d,d_{\min})$.

\begin{theorem}\label{thm:minimal partition}
Assume that $X$ has an MSS $\{X_1,\cdots,X_K\}$ satisfying the assumption (\ref{assumption}). If 
\begin{align}\label{cond:d}
	d_{\rm int} <d_{\min}, \quad
	d_0 \leq d_{\min}, \quad
	n_k > d_k+(K-1)d_0,\quad \forall k,
\end{align}
then ${\cal J}(K,d,d_{\min})$ is a singleton with the unique $J^*$.
\end{theorem}
\begin{proof}
Let $J_k^*$ be the index set of $X_k$. By Lemma \ref{lma:large J_ik}, $\I_k = \{i : |J_{ik}| \geq d_k\}$ is nonempty. We further show that $\I_k$ has only a single index $i_k$ for each $k$ and $|\I|=K$. If it is proven,  the mapping from $k$ to $i_k$ is one-to-one; hence for $i=i_k$, $r(X_{J_i})\geq r(X_{J_{ik}})=\min\{|J_{ik}|,d_k\} = d_k$. This equality holds since $d\geq\sum_kr(X_{J_{ik}}\!)\geq\sum_kd_k=d$. Thus, $\spann(X_{J_{i_k}}) = \spann(X_{J_{i_k,k}}) = \S_k$ and $J_{i_k} \!=\! J_{i_k,k}=J_k^*$. That is, $\{X_{J_{i_k}}\}$ is equal to $\{X_k\}$, so $\J(K,d,d_{\min}) $ is a singleton with the unique $J^*$.

 We now prove that $|\I_k|=1$ for each $k$ and $|\I|=K$ by Lemma \ref{lma:rank}.
For $i\in \I^c$, we have $r(X_{J_i})\geq d_{\min}+1$ by (b) of Lemma \ref{lma:rank}, and then $\sum_{i\in\I^c}r(X_{J_i})\geq|\I^c|(d_{\min}+1)$.
For $i\in\I$, we choose $(s,t)$ in (\ref{rank:X_Ji}) such that $s\neq t\in\K_i$ if $|\K_i|>1$, or $s\in\K_i$ and $t\in \K_i^c$ if $|\K_i|=1$.
 We use the indication function $\delta_{|\K_i|>1}= 1$ if $|\K_i|>1$ or $\delta_{|\K_i|>1}=0$, otherwise , and obtain that for $i\in \I$,
\begin{align}\label{ineq:rankJi}
\begin{split}
	r(X_{J_i})  & \geq d_s+ \delta_{|\K_i|>1}(d_t-d_{\rm int})+\sum_{k\in \K_i,k\neq s,t}(d_k-d_0)_+ \\
	& = d_0 + \delta_{|\K_i|>1}(d_0-d_{\rm int})+ \sum_{k\in \K_i}(d_k-d_0),
\end{split}
\end{align}
and $\sum_{i\in\I}r(X_{J_i})\geq|\I|d_0+\sum_{i\in\I}\delta_{|\K_i|>1}(d_0-d_{\rm int})+\sum_{i\in \cal I}\sum_{k\in \K_i}(d_k-d_0)$. Hence, 
\begin{align}\label{ineq:d}
	d \geq \sum_{i}r(X_{J_i}) \geq |\I| d_0 +  \sum_{i\in \cal I}\delta_{|\K_i|>1}(d_0-d_{\rm int})
	+\sum_{i\in \cal I}\sum_{k\in \K_i}(d_k-d_0)+|\I^c| (d_{\min} + 1).
\end{align}
Since 
$|\I_k|\geq 1$, $|\K_i|\geq 1+\delta_{|\K_i|>1}$ for $i\in\I$, 
and $d_{\min}\geq d_0$, we estimate the third term as
\begin{align}
	\sum_{i\in\I}\sum_{k\in \K_i}(d_k-d_0) 
	& = \sum_{i\in\I}\sum_{k\in \K_i}(d_k-d_{\min})+\sum_{i\in\I}\sum_{k\in \K_i}(d_{\min}-d_0)	\nonumber\\	
	& = \sum_k|\I_k|(d_k-d_{\min})+\sum_{i\in\I}|\K_i|(d_{\min}-d_0) 	\nonumber\\
	& \geq d - Kd_{\min}+ |\I|(d_{\min}-d_0)+\sum_{i\in \I}\delta_{|\K_i|>1}(d_{\min}-d_0). \label{ineq:sumrank}
\end{align}
Substituting (\ref{ineq:sumrank}) into (\ref{ineq:d}), we obtain that $d\geq d+|\I^c|+\sum_{i\in \cal I}\delta_{|\K_i|>1}(d_{\min}-d_{\rm int})\geq d$. Hence, $|\I^c| = 0$. Furthermore, we have $|\K_i|=1$ for each $i\in \I$ if $d_{\min}>d_{\rm int}$. Since $|\K| = |\I| = K$, $|\K_i|=1$ for each $i\in \I$ is equivalent to  
$|\I_k| = 1$ for each $k$. The theorem is then proven.
\end{proof}

The condition $n_k > d_k+(K-1)d_0 $ for all $k$ is generally satisfied in applications. The other conditions on $d_{\rm int}$ and $d_0$ are also satisfied (in some cases, naturally) basically  because $\S_k \nsubseteq \S_j$ for any $k \neq j$. In practice, since $\S_k \cap \S_j \subsetneq \S_k$ for $k \neq j$, $d_{\rm int} < \max_{k \neq j} \min \{d_k, d_j\} \leq \max_k d_k$ and $d_0  \leq \max_k d_k $. Thus, if all the $d_k$'s are equal, $d_{\rm int}< \max_k d_k = d_{\min}$ and $d_0 \leq d_{\min}$. The equality restriction on $\{d_k\}$ can be released if $K = 2$ since $d_0 = d_{\rm int}  < \min\{d_1, d_2\} = d_{\min}$. We summarize  our conclusions as a corollary.

\begin{corollary}\label{coro:same_d}
Assume that $X$ has an MSS $\{X_1,\cdots,X_K\}$ satisfying the assumption (\ref{assumption}).
If $n_k > d_k + (K-1)d_0$ for each $k$, and $d_1 = \cdots = d_K$ when $K>2$ or arbitrary $d_1$ and $d_2$ when $K=2$, then $J^*$ is the single partition in ${\cal J}(K,d,d_{\min})$.
\end{corollary}

\subsection{Necessity of the Sufficient Conditions}

The conditions of Theorem \ref{thm:minimal partition} are tight. In this subsection, we give three counterexamples to show that if one of these conditions, except $n_k > d_k + (K-1)d_0$, is not satisfied, ${\cal J}(K,d,d_{\min})$ may not be a singleton. In detail, the MSS $\{X_k\}$ in Example \ref{eg:ndg_int} is not interacted nondegenerately, and the other conditions in (\ref{cond:d}) are satisfied. Example \ref{eg:ndint} is designed such that $d_{\rm int} < d_{\min}$ is not obeyed, and in Example \ref{eg:nd0}, there exists a $d_k<d_0$. 

\begin{example} \label{eg:ndg_int}
Let $X_k = U_k H\in\R^{8\times n_k}$ with $n_k > 30$, $k=1,\cdots,6$, where $H$ is nondegenerate and its first three columns are $e, e-e_1,e-e_1-e_2$, $e_i$ is the $i$-th column of the identity matrix of order 8, and $e$ is a the column vector of all ones. Each $U_k$ consists of 5 columns of the same identity matrix,
\begin{align*}
	&U_1 = [e_1, e_2, e_3, e_4, e_5], \ U_2 = [e_1, e_2, e_6, e_7, e_8], \ U_3 = [e_3, e_4, e_1, e_2, e_6],\\
	&U_4 = [e_3, e_4, e_5, e_7, e_8], \ U_5 = [e_5, e_6, e_1, e_2, e_4], \ U_6 = [e_5, e_6, e_4, e_7, e_8].
\end{align*}
\end{example}

Obviously, $\{X_k\}$ is an MSS of $X=\cup_k X_k$ with $K=6$ and $d = 30$ since $r(X_k) = 5$ for all $k\leq K$. The inequality conditions in (\ref{cond:d}) are satisfied since $d_0 = 5$, $d_{\rm int} =4$, $d_k=5$, and $n_k > 30=d_k+(K-1)d_0$ for all $k$. However, for $\S_1' = \S_1\cap\S_2$, the splitting $X_1 = Y_1+Z_1$ with $Y_1\subset S'_1$ and $Z_1\subset\S_1\setminus\S_1'$ results in a degenerate $Z_1$ whose first three columns are equal to $e_3+e_4+e_5$. Hence, $\{X_k\}$ is not interacted nondegenerately. 
In addition to MSS $\{X_k\}$, we have another segmentation of $6$ pieces as 
\begin{align*}
	&\tilde X_1 = [X_1(:,1:3), X_2(:,1:3)],\ 
	\tilde X_2 = [X_3(:,1:3), X_4(:,1:3)],\\ 
	& \tilde X_3 = [X_5(:,1:3), X_6(:,1:3)], \ 
	\tilde X_4 = X_1(:,4:n_1),  \ 
	\tilde X_5 = X_2(:,4:n_2),  \\ 
	& \tilde X_6 = [X_3(:,4:n_3),X_4(:,4:n_4),X_5(:,4:n_5), X_6(:,4:n_6)].
\end{align*}
Since $r(\tilde X_k) = 4$ for $k\leq 3$, $r(\tilde X_4) = r(\tilde X_5) = 5$, $r(\tilde X_6) = 8$, we also have $\sum_kr(\tilde X_k) = 30$. Hence, the partition $\tilde J$ corresponding to $\{\tilde X_k'\}$ also belongs to ${\cal J}(K,d,d_{\min})$. However, $\{\tilde X_k\}$ cannot be minimal since both $\{X_k\}$ and $\{\tilde X_k\}$ satisfy $n_k > 7(d_k-1)$ for each $k$, and by Theorem \ref{thm:unique1}, the MSS of $X$ with $K=6$ is unique. 

\begin{example}\label{eg:ndint}
Let $X=[X_1,X_2,X_3]$ with $X_k = U_kH_k$, where $U_1 = [e_1, e_2]$, $U_2 = [e_2, e_3]$, and $U_3 = e_4$ are three orthonormal matrices of four rows, and $H_1$, $H_2$, and $H_3$ are three non-degenerate matrices of 5 columns with  $2$, $2$, and $1$ row(s), respectively.
\end{example}

This segmentation is minimal by definition but does not satisfy $d_{\rm int}<d_{\min}$. We have a different one $\{\tilde X_1, \tilde X_2, \tilde X_3\}$  where $\tilde X_1 = [X_1,X_2]$ and the other two pieces $\tilde X_2$ and $\tilde X_3$ split from $X_3$, each  having at least two samples $\tilde X_2$ and $\tilde X_3$  Since $\sum_k r(X_k) = \sum_k r(\tilde X_k) = 5$, both  partitions belong to $\J(3,5,1)$.
\begin{example}\label{eg:nd0}
Let $X=[X_1,\cdots,X_5]$ of columns in $\R^8$ and $X_k = U_kH_k$ with orthonormal  
\[
	U_1 = [e_1, e_2, e_3],\  
	U_2 = [e_3, e_4, e_5], \
	U_3 = [e_1, e_4, e_6], \
	U_4 = [e_2, e_5, e_6], \
	U_5 = [e_7, e_8],
\]
and let $\{H_k\}$ be intersected nondegenerately with $n(X_k) = n(H_k) >15$.
\end{example}

The segmentation $\{X_k\}$ is also minimal with $d = 15$ since each $X_k$ is nondegenerate as $H_k$. Now, the condition $d_0\leq d_{\min}$ is not satisfied since $d_0 = 3$ and $2=d_{\min}$. If we merge the first 4 segments to be $\tilde X_1$ and split $X_5$ into 4 pieces as $\tilde X_2,\cdots,\tilde X_5$ without overlap, and each $\tilde X_k$ has at least three samples, then $\sum_k r(\tilde X_k)=15$. Hence, ${\J}(5,15,2)$ has at least two different partitions. 

\section{Segmentation Refinement}\label{sec:refine}

When  either of two conditions $d_{\rm int}<d_{\min}$ or $d_0 \leq d_{\min}$ in Theorem \ref{thm:minimal partition} are not satisfied, ${\cal J}(K,d,d_{\min})$ may have multiple $K$-partitions. Hence, there may be a partition $J=\{J_k\}$ in ${\cal J}(K,d,d_{\min})$  that is not minimal. However,  certain segments $J_k$ or $X_{J_k}$ can be further refined to be minimal under some weak conditions. Let us illustrate  this scenario on the examples shown in the last subsection.

In Example \ref{eg:ndint},  we take segment of $\{\tilde X_k\}$ with the smallest rank, say $\tilde X_2$,  and extend it to be the largest  segment containing all the samples belonging to $\spann(\tilde X_2)$. This extension merges $\tilde X_2$ and $\tilde X_3$  as $X_3$; hence, $X_3$ is recovered. Then, $\{X_1,X_2\}$ is an MSS of the remaining samples $X'=X\!\setminus\!X_3$. One may  search for a segmentation from ${\cal J}(K',d',d_{\min}')$ on $X'$ with $K' = K-1 = 2$, $d'=d-r(X_3) = 4$ and $d_{\min}' = 2$. Since the conditions of Theorem \ref{thm:minimal partition} are now satisfied , ${\cal J}(K',d',d_{\min}')$ has the single segmentation $\{X_1,X_2\}$. Hence, minimal segmentation $\{X_k\}$ is  recovered.  Similar, we can refine $\{\tilde X_k\}$ in Example \ref{eg:nd0}.

In Example \ref{eg:ndg_int}, each $\tilde X_k$, $k\leq 3$, has the smallest rank but is nonextendable. However, the extension works on the larger segments $\tilde X_4$ or $\tilde X_5$. That is, if we extend $\tilde X_4$ to the largest one, $X_1$ can be recovered immediately. Similarly, when $\tilde X_5$ is extended, $X_2$ can also be recovered.  Other segments can be determined from ${\cal J}(K',d',d_{\min}')$ on the remaining samples with $K'=K-2$, $d'=d-r(X_1)-r(X_2)$ and $d_{\min}' = 5$.

Motivated by  these observations, we  offer an approach  for refining a segmentation $\{X_{J_k}\}$ for $J\in{\cal J}(K,d,d_{\min})$ if it is not minimal. The approach consists of two strategies: reduction and extension. We emphasize that, in this section  our analysis is given under the same assumption as that given in the last subsection. Hence, we no longer mention the conditions for simplicity.

\subsection{Segment Reduction}

We observe that a partition $J=\{J_k\}\in{\cal J}(K,d,d_{\min})$ has at least one piece $J_k$ such that $X_{J_k}$ is a minimal segment, even if the whole segmentation $\{X_{J_i}\}$ is not an MSS of $X$. 
The following two propositions support this observation.

\begin{proposition}\label{prop:contain}
If $r(X_{J_i}) = d_{\min}<|J_i|$, then $J_i\subseteq J_k^*$ with $d_k=d_{\min}$.
\end{proposition}
\begin{proof}
By (b) of Lemma \ref{lma:rank}, if $J_i$ has two nonempty intersection parts $J_{ik}=J_i\cap J_k^*$ with two different $k$'s, then the condition $|J_i|> d_{\min}$ implies $r(X_{J_i}) > d_{\min}$.
Therefore, if we also have $r(X_{J_i}) = d_{\min}$, $J_i$ must have a single nonempty $J_{ik}$, that is, $J_i\subseteq J_k^*$, and $X_{J_i}\subseteq X_k$. Since $X_k$ is nondegenerate, 
\[
    d_{\min}=r(X_{J_i}) = \min\{|J_i|, d_k\} \geq \min\{d_{\min}+1, d_k\}.
\]
Combining this with $d_k\geq d_{\min}$, we can conclude that $d_k=d_{\min}$. 
\end{proof}
\begin{proposition}\label{prop:min}
If $d_0 \leq d_{\min}$ and $|J_k|>d_k+(K-1)d_0$ for all k, then $\min_i r(X_{J_i}) = d_{\min}$.
\end{proposition}
\begin{proof}
If we further have that $d_{\min}>d_{\rm int}$, this proposition is obviously true since $\{J_k\}$ is unique by Theorem \ref{thm:minimal partition}. Hence, we can assume $d_{\min}\leq d_{\rm int}$, which implies $d_{\rm int} = d_0 = d_{\min}$ since $d_{\rm int} \leq d_0 \leq d_{\min}$. Thus, in the proof of Theorem \ref{thm:minimal partition}, 
we have $|\I^c|=0$ and each of the inequalities between (\ref{ineq:rankJi}) and (\ref{ineq:sumrank}) holds in equality, where we do not use $d_{\rm int} < d_{\min}$.  
The equalities $d_{\rm int} = d_0 = d_{\min}$ simplify (\ref{ineq:sumrank}) to $\sum_k|\I_k|(d_k-d_{\min})= d - Kd_{\min}$. Thus, $|\I_k|=1$ if $d_k>d_{\min}$. Consider the union of those $\I_k$ with $k\in\K' = \{k: d_k > d_{\min}\}$. The size of this union is equal to $|\K'|$ and $|\K'|<K$ since there is a $k\notin \K'$ with $d_k=d_{\min}$.
For each $k\in \K_{i_0}$ with $i_0 \notin \cup_{k \in \K'} \I_k$, we also have $i_0\in \I_k$. Hence, we conclude that for all $k\in \K_{i_0}$, $k\notin\K'$, {\it i.e.}, $d_k = d_{\min}$. Moreover, (\ref{ineq:rankJi}) becomes $r(X_{J_i}) = d_{\min} + \sum_{k \in \K_i}(d_k-d_{\min})$ for all $i$. Hence, $r(X_{J_{i_0}}) = d_{\min} \leq r(X_{J_i})$ for all $i$.
\end{proof}

Therefore, if the conditions of Proposition \ref{prop:min} are satisfied, by Proposition \ref{prop:contain}, $J_i \subseteq J_{k_i}^*$ for those $J_i$ and $J_{k_i}^*$ with the smallest rank $r(X_{J_i})=d_{\min}=d_{k_i}$. That is, these minimal segments $X_{k_i}$ have been retrieved. Let $K_0$ be the number of retrieved segments. The conditions of Proposition \ref{prop:min} remain satisfied for $J'\in{\cal J}(K',d',d_{\min}')$, where $K' = K-K_0$, $d'=d-d_{\min}K_0
$ and $d_{\min}'\geq d_{\min}$. Thus, repeating this reduction procedure, we can retrieve all the minimal segments. Therefore, we have proven the following theorem.

\begin{theorem}
If $d_0 \leq d_{\min}$ and $n_k>d_k+(K-1)d_0$ for all k, then the MSS $\{X_k\}$ of $X$ can be recovered via a reduction procedure on any $J\in{\cal J}(K,d,d_{\min})$.
\end{theorem}

\subsection{Segment Extension} 

We say that $J_i$ or $X_{J_i}$ is extendable if there is at least one $x_j\in\spann(X_{J_i})$ with $j\notin J_{i}$. The extension strategy enlarges an extendable $J_i$ as much as possible by adding all these $j$'s into $J_{i}$ similar to that given in the proof of Theorem \ref{thm:min seg} without checking for segment nondegeneracy. 
\begin{proposition}\label{prop:notextendable}
If each $J_i$ of $J\in{\cal J}(K,d,d_{\min})$ is nonextendable, then $\{J_i\}= \{J_k^*\}$.
\end{proposition}
\begin{proof}
By Lemma \ref{lma:large J_ik}, for each $k$, there exists an $i_k$ such that $|J_{i_k,k}| \geq d_k$ and hence, $\spann(X_{J_{i_k,k}}) = \S_k $. Since $J_{i_k}$ cannot be extended, we must have $J_{i_k,k}= J_k^*$ and $J_{ik}$ is empty for $i\neq i_k$. Because each $J_i$ is not empty, the mapping from $k$ to $i_k$ is one-to-one and onto. Therefore, $\{J_i\} = \{J_k^*\}$. 
\end{proof}
\begin{proposition}\label{prop:extendable}
If $J_i$ is extendable, the extended $\tilde J_i$ has $\tilde J_{ik} \!=\! J_{ik}$ or $\tilde J_{ik} = J_k^*$ for all $k$.
\end{proposition}
\begin{proof}
Assume $\tilde J_i = J_i\cup J_i'$, where $X_{J_i'}\subset \spann(X_{J_i})$. Then, $\tilde J_{ik}=J_{ik}\cup J_{ik}'$, where $J_{ik}'=J_i'\cap J_k^*$. 
If the minimal segment $X_k\subseteq\spann(X_{J_{ik}^c})\subset \spann(X_{J_i})$, where $J_{ik}^c=J_i\setminus J_{ik}$, then we must have $\tilde J_{ik} = J_k^*$.

We will show that if $X_k\nsubseteq\spann(X_{J_{ik}^c})$ and $\tilde J_{ik} \! \neq \! J_{ik}$, then $r([X_k, X_{J_{ik}^c}]) = r(X_{J_i})$, which implies $X_k\subset\spann(X_{J_i})$, and hence, $\tilde J_{ik} = J_k^*$. To this end, we split $X_k = Y_k + Z_k$ with $Y_k\subset\S_k'$ and $Z_k\subset\S_k''$, according to the direct sum $\S_k = \S_k'\oplus\S_k''$ where $\S_k' = \S_k\cap \spann(X_{J_{ik}^c})$ and $\S_k''$ is the orthogonal complement of $\S_k''$ restricted in $\S_k$. Obviously, we have $Z_k\neq 0$ when $\tilde J_{ik} \! \neq \! J_{ik}$, and hence, $Z_k$ is nondegenerate since the MSS $\{X_k\}$ is intersected non-degenerately. 
Therefore, $r([X_k, X_{J_{ik}^c}]) =r([Y_k, X_{J_{ik}^c}])+r(Z_k) = r([Z_k, X_{J_{ik}^c}])$. 
Similarly, 
\[
	r(X_{J_i}) =r([X_{J_{ik}}, X_{J_{ik}^c}])= r([Z_k',X_{J_{ik}^c}]),\quad
	r(X_{\tilde J_i}) =r([X_{\tilde J_{ik}}, X_{J_{ik}^c}])= r([\tilde Z_k',X_{J_{ik}^c}]),
\]
where $Z_k'$ and $\tilde Z_k'$ are two subsets of $Z_k$ corresponding to the subsets $X_{J_{ik}}$ and $X_{\tilde J_{ik}}$ of $X_k$, respectively. We have $r(\tilde Z_k')=r(Z_k')$ since $r(X_{\tilde J_i})=r(X_{J_i})$. By the nondegeneracy of $Z_k$, $r(Z_k')=r(Z_k)$. Therefore, $r([X_k, X_{J_{ik}^c}]) = r(X_{J_i})$, and the proof is completed. 
\end{proof}

\begin{theorem}\label{thm:divide-conquer}
Assume that $n_k > d_k+(K-1)d_0 $ for all $k$. After segment extension on all the extendable segments of $J$ one-by-one, then each nonempty segment $\tilde J_i$ of the resulting segmentation $\tilde J$ must be a union of several segments of the MSS $\{X_k\}$. Furthermore, $\tilde J=J^*$ if all the $\tilde J_i$'s are nonempty.
\end{theorem}
\begin{proof}
Consider the changes of $J_{ik}$ during the extension process. In the step involving an extendable $J_i$, $J_{ik}$ is unchanged or changed to $J_k^*$ by Proposition \ref{prop:extendable}. After this extension step, for other $j\neq i$, $J_{jk}$ is unchanged or changed to the empty set. That is, after an extension step, $J_{ik}$ is unchanged or becomes to $J_k^*$ or the empty set. By Lemma \ref{lma:large J_ik}, for each $k$, there is an $i_k$ such that $|J_{i_k,k}| \geq d_k$ for the original $J_{i_k}$. Hence, $J_{i_k,k}$ becomes $J_k^*$ in the extension of $J_{i_k}$ if it is unchanged in the earlier extension steps. Otherwise, $J_{i_k,k}$ has already been changed to the empty set. Therefore, after all the extension steps, for each $k$, the eventually modified $\tilde J_{ik}$ must be empty or $J_k^*$. That is, each $\tilde J_i$ must be a union of some $J_k^*$ or the empty set.
\end{proof}

Note that the extension results may depend on the extending order of each $J_i$. We suggest extending in the ascending order of $\{r(X_{J_i})\}$ because if $X_{J_i}$ has a smaller rank, most of its pieces $\{J_{ik}\}$ are likely to be empty or small. These small pieces will be removed, leaving the largest pieces after extension. This strategy reduces the risk of merging multiple minimal segments ($X_k$'s) together into a single segment of $\tilde J$.

The extension procedure cannot increase the rank sum of the segments, and the rank sum is decreased if a $\tilde J_i$ is empty. In this case, $\tilde J$ has a smaller number of nonempty segments. Hence, $\tilde J\notin {\cal J}(K,d,d_{\min})$. Assume that $\tilde J$ has $\ell$ nonempty segments, say $\tilde J_1,\cdots,\tilde J_\ell$, and let $\K_t$ be the index set of those $X_k$ that are merged to $X_{\tilde J_t}$ because of the extension. MSS $\{X_k\}$ is then partitioned into $\ell$ smaller groups $X^{(t)}$, $t=1,\cdots,\ell$, each of which is a union of the minimal segments $\{X_k: k\in {\cal K}_t\}$. Therefore, one may further determine the MSS of $X^{(t)}$ via determining ${\cal J}(K^{(t)},d^{(t)}, d_{\min}^{(t)})$ on $X^t$, where $K^{(t)} = |\K_t|$, $d^{(t)} = \sum_{k\in\K_t} r(X_k)$, and $d_{\min}^{(t)} = \min_{k \in \K_t} r(X_k)$. This amounts to a divide-and-conquer approach. We do not touch upon this technique further in this paper.

\section{Computable Modeling for Minimal Subspace Segmentation}\label{sec:modeling}

Our study of MSS detection simplifies its inspection. However, because detecting the MSS works on $K$-partitions of indices, it is difficult to implement efficiently. Thus, computable modeling is needed. To this end, we adopt the commonly used self-expressiveness approach. 

The self-expressiveness method looks for a matrix $C$ with special structures to represent the sample matrix $X$ as $X = XC$, hoping that subspace clustering is well-determined via spectral clustering on the graph matrix $|C|+|C|^T$. The effectiveness of the self-expressiveness method is conditioned by two issues: (1) the correctness of the learned partition $J = \{J_k\}$ under which $C$ has a block-diagonal form, and (2) the connection of each diagonal block $C_k = C(J_k^*,J_k^*)$ of $C$ in the minimal partitions $\{J_k^*\}$. 
As mentioned before, the connection of matrix $C_k$ refers to the connection of the undirected graph constructed from $|C_k|+|C_k|^T$. Our previous analysis addresses the first issue for theoretically detecting the MSS. 

In this section, we address the issue of connection to support a computable optimization problem that we will propose for determining the MSS. Closed-form representation matrices are first given. Based on these closed-form representation matrices, we then exploit the conditions of connected diagonal blocks of
a representation matrix in block-diagonal form. In addition, we discuss solutions of SSC and LRR.

\subsection{Structures of Representation Matrices}

Obviously, the representation matrix $C$ of $X$ is not unique since adding a matrix of null vectors of $X$ to $C$ results in another representation matrix of $X$. Notice that because $C$ solves the linear system $X = XC$, it should have a closed-form structure. We use the singular value decomposition (SVD) of $X$ in thin form:
\begin{align}\label{SVD of X}
	X = U\Sigma V^T,
\end{align}
to represent the closed-form representation matrices, where $U$ and $V$ are the orthonormal matrices of the left and right singular vectors of $X$ corresponding to its nonzero singular values $\sigma_1\geq \cdots\geq \sigma_r>0$, where $r = r(X)$, which are given in the diagonals of the diagonal matrix $\Sigma$. If $r<n$, $V$ has an orthogonal complement $V_\bot$ for forming an orthogonal matrix $[V,V_\bot]$. We use the SVD together with orthogonal complement $V_\bot$ to characterize the representation matrix. 

\begin{lemma}\label{lma:C}
$C$ is a representation matrix of $X$ if and only if it has the following form 
\begin{align}\label{def:C}
	C = VV^T+V_{\bot}H,
\end{align}
with a matrix $H\in{\cal R}^{(n-r)\times n}$. Thus, $r(C)\geq r(X)$ and $\|C\|_*\geq r(X)$. Furthermore,
\begin{itemize}
	\item[(a)] If $C$ is symmetric, $H = SV_{\bot}^T$, that is, $C = VV^T+V_{\bot}SV_{\bot}^T$ with a symmetric $S$.
	\item[(b)] If $\tr(C)=0$, then $r(C)\geq r(X)+1$.
\end{itemize}
\end{lemma}
\begin{proof}
Based on the SVD given in (\ref{SVD of X}), $C$ is a representation matrix of $X$, {\it i.e.}, $X=XC$, if and only if $V^T = V^TC$. Hence, 
$C = VV^TC+V_\bot V_\bot^TC = VV^T+V_\bot H$ with an arbitrarily $H$. 
We rewrite $H = TV^T+SV_{\bot}^T$ with arbitrary $T$ and $S$. Then $C = [V, V_\bot]\left[\begin{array}{cc} I_r & 0\\ T & S\end{array}\right][V, V_\bot]^T$, where $r=r(X)$. Thus, $r(C)\geq r$ and $\|C\|_*\geq r$.

Furthermore, if $C$ is symmetric, $T = 0$ and $S$ is symmetric obviously. That is (a). 
Since $\tr(C) = r+\tr(S)$ by the proposition $\tr(AB) = \tr(BA)$, if $C$ is imposed the restriction $\tr(C) = 0$, then $\tr(S) = -r$. Hence, $S\neq0$ and $r(C) = r+r(S)\geq r+1$. That is (b).
\end{proof}

We note that a representation matrix $C$ of $X$ could be of arbitrary rank $r'$ varying from $r(X)$ to $n(X)$. Practically, if we choose $H = \diag(H_{r'},0)V_{\bot}^T$ with any nonsingular matrix $H_{r'}$ of order $r'-r(X)$ in (\ref{def:C}), then obviously $r(C) = r'$. 

\subsection{Minimal Subspace Detectable Representation}\label{subsection:MSDR}

The self-expressiveness approach seeks a block-diagonal representation matrix $C$. That is, there is a permutation matrix $\Pi$ such that, within a given or existing partition $J = \{J_\ell\}$,
\[
    C = \Pi\diag(C_{J_1},\cdots,C_{J_{|J|}})\Pi^T,
\]
where $|J|$ defines the number of partition pieces.
Simultaneously, $X$ is also partitioned as $X = [X_{J_1},\cdots,X_{J_{|J|}}]\Pi^T$.
A given partition $\{J_\ell\}$ is not naturally assumed to contain all the connected diagonal blocks $\{C_{J_\ell}\}$. 
In this subsection, we inspect the rank propositions of the state-of-art LRR and SSC, when their solution has a block-diagonal form. The risk of non-connected diagonal blocks is discussed even when partition $\{J_\ell\}$ is ideally chosen as a minimal partition for detecting the MSS. Finally, we prove that the connection is guaranteed under a rank restriction similar to that in the set ${\cal J}(K,d,d_{\min})$, and hence, the MSS can be correctly detected. 

\subsubsection{Propositions of LRR and SSC}\label{subsect:LRR-SSC}

LRR is known to give a representation matrix that has the smallest nuclear norm, which implies that $H=0$ in Lemma \ref{lma:C}, and hence, it also has the smallest rank. Meanwhile, an SSC solution has a larger rank or nuclear norm due to a nonzero $H$. The following lemma further characterizes the solutions of LRR and SSC.\footnote{The sufficient condition for LRR was given by \cite{LRR2013}.}

\begin{theorem}\label{thm:LRR-SSC}
LRR provides a block-diagonal solution if and only if $r(X)=\sum_\ell r(X_{J_\ell})$ with a partition $\{J_\ell\}$. If SSC provides a block-diagonal $C$ with a total of $T$ connected blocks, then $r(C)\geq \sum_\ell r(X_{J_\ell})+T$.
\end{theorem}
\begin{proof}
For the sake of simplicity, let $C_\ell' = C_{J_\ell}$ and $X_\ell' = X_{J_\ell}$. If LRR has a block-diagonal $C$ with diagonal blocks $\{C_\ell'\}$ by $X_\ell'=X_\ell'C_\ell'$, we have $r(X_\ell')\leq r(C_\ell')$ and $r(C)=\sum_\ell r(C_\ell')\geq \sum_\ell r(X_\ell')\geq r(X)$. On the other hand, since the LRR solution is uniquely given by $C = VV^T$, we have $r(C) = r(X)$. Thus, $r(X)=\sum_\ell r(X_\ell')$. 

Conversely, if $r(X)=\sum_\ell r(X_\ell')$ for a segmentation $\{X_\ell'\}$ of $X$, we partition $V^T\Pi=[B_1^T,\cdots,B_{|J|}^T]$ as $\{J_\ell\}$. Based on the thin SVD (\ref{SVD of X}), we get $X_\ell'= U\Sigma B_\ell^T$ and $r(B_\ell) = r(X_\ell')$. Let $B_\ell = Q_\ell R_\ell$ be the QR decomposition of $B_\ell$ with an orthonormal $Q_\ell$ of $r(X_\ell')$ columns and a matrix $R_\ell$ of order $r(X_\ell')\times r(X)$. Then, $V = \Pi QR$, where $Q = \diag(Q_1,\cdots,Q_{|J|})$ and $R^T = [R_1^T,\cdots,R_{|J|}^T]$. The condition $r(X)=\sum_\ell r(X_\ell')$ means that $R$ is a square matrix. Since $V$ is orthonormal, $R$ must be orthogonal. Therefore, the LRR solution $C = VV^T$ can be rewritten as follows
\[
    C = \Pi(\Pi^TVV^T\Pi)\Pi^T
      = \Pi Q^TQ\Pi^T
      = \Pi\diag\big(Q_1Q_1^T,\cdots,Q_{|J|}Q_{|J|}^T\big)\Pi^T.
\]
That is, $C$ is block-diagonal.

If SSC provides an MSDR of $X$ with a block-diagonal $C$ of $T$ connected diagonal blocks $\{C_\ell'\}$, then $r(C_\ell')\geq r(X_\ell')+1$ by Lemma \ref{lma:C}(b) since $X_\ell' = X_\ell'C_\ell'$ and $\tr(C_\ell') = 0$. 
Hence, the lower bound of $r(C)$ follows immediately since $r(C)=\sum_\ell r(C_\ell')$.
\end{proof}

Strict sufficient conditions are given by \cite{Soltanolkotabi2012A} for SSC to have a block-diagonal representation matrix according to ideal segmentation $\{X_{J_k^*}\}$. These conditions are very strict and may be difficult to satisfy in applications. We will briefly discuss these sufficient conditions in Section \ref{subsec:comparison}. In addition, the block-diagonal form does not guarantee the connection of all the diagonal blocks. This phenomenon was reported by \cite{Nasihatkon2011Graph}. There is a notably large gap between ranks $r(C) = r(X)$ and $r(C)\geq\sum_{\ell}r(X_\ell')+T\gg r(X)$ of the possible block-diagonal solutions of LRR and SSC, respectively.
In the next subsection, we show how such a block-diagonal representation may be unconnected, even if it is ideally partitioned.

\subsubsection{Nonconnectivity}\label{sec:connection}
Even if we have a block-diagonal representation matrix $C$ in the ideal partition $J^* = \{J_k^*\}$, it is possible to have unconnected diagonal blocks in $C$, mainly because the solution does not have a suitable rank. This observation stems from the closed-form structure given in Lemma \ref{lma:C}. Practically, because of the block-diagonal form of $C$, $X_k = X_kC_k$, where $X_k = X(:,J_k^*)$ and $C_k = C(J_k^*,J_k^*)$ as before. Hence, each $C_k$ has the form $C_k = V_kV_k^T+(V_k)_\bot H_k$, where $V_k$ is based on the SVD of segment $X_k$: $X_k = U_k\Sigma_kV_k^T$. 
An unsuitable $H_k$ may result in an unconnected $C_k$. To further verify this observation, we consider the construction of an unconnected representation matrix of a given subset $X_k$, no matter whether it spans a minimal subspace or not. For the sake of simplicity, let 
\[
    d_k = r(X_k), \quad
    n_k = n(X_k), \quad 
    d = d_1+\cdots+d_K. 
\]
The following lemma shows how to construct such an unconnected $C_k$ with a given rank.

\begin{lemma}\label{lma:C_k}
Given $r'\in(d_k,n_k]$, there is an unconnected representation $C_k$ of $X_k$ with $r(C_k) = r'$. 
\end{lemma}
\begin{proof}
Write the integer $r'\in(d_k,n_k]$ as $r' = pd_k+t$ with $t\in [0, d_k)$. If $t=0$, we partition $X_k=[X_{k,1},\cdots,X_{k,p}]$, where each $X_{k,j}$ has at least $d_k$ columns. As metioned below the proof of Lemma \ref{lma:C}, we have a representation $C_{k,j}$ of $X_{kj}$ with rank $r(C_{k,j}) = d_k$ since $d_k\leq n(X_{k,j})$. Thus, $C_k = \diag(C_{k,1},\cdots,C_{k,p})$ is a representation of $X_k$. 
If $t\neq 0$, we partition $X_k=[X_k',X_k'']$, where $X_k'$ has $n_k-t$ columns. Since $pd_k = r'-t\leq n_k-t$, as mentioned below Lemma \ref{lma:C} again, we have a representation matrix $C_k'$ of $X_k'$ with rank $pd_k$. Thus, $D_k = \diag(C_k',I_t)$ is a representation of $X_k$. Obviously, $r(C_k)=r'$ and $C_k$ is unconnected in both cases. 
\end{proof}

Similarly, we can construct an unconnected block-diagonal representation matrix $C$ of $X$ in a given segmentation $\{X_1,\cdots,X_K\}$ of $X$.

\begin{theorem}\label{thm:unconnect}
Given a segmentation $\{X_k\}$ of $X$ and an integer $d_+ \in (d, n]$, there is a block-diagonal representation matrix $C$ of $X$ such that $r(C)=d_+$ and some diagonal blocks partitioned as per (\ref{XC}) are not connected. 
\end{theorem}
\begin{proof}
Since $d_1+\cdots+d_K = d<d_+\leq n = n_1+\cdots+n_K$, we can write $d_+ = r_1+\cdots+r_K$ with $\{r_k\}$ satisfying $r_k\in(d_k,n_k]$ for $k\leq \ell$ and $r_k = d_k$ for $k = \ell+1,\cdots,K$ with a suitable $\ell\geq 1$. Applying Lemma \ref{lma:C_k} to the first $\ell$ segments, we obtain the unconnected representation matrices $C_k$ of $X_k$ with rank $r_k$ for $k\leq \ell$. For each $k>\ell$, we also have a $C_k$ with rank $d_k$ via the SVD of $X_k$ as previously mentioned. Thus, $C = \Pi^T\diag(C_1,\cdots,C_K)\Pi$ is obviously a representation matrix of $X$ with rank $d_+$. 
\end{proof}

\subsubsection{MSDR: Minimal Subspace Detectable Representation}

We seek a representation matrix of $X$ that can be used to correctly detect the minimal subspace segmentation. This representation matrix should be partitioned block-diagonally as the MSS and all the diagonal blocks are connected. We call such a matrix the minimal subspace detectable representation (MSDR) of $X$.

\begin{definition}\label{def:msdr}
A representation matrix $C$ of $X$ is minimal subspace detectable if there is a permutation matrix $\Pi$ such that:
\begin{align}\label{XC}
    X\Pi = [X_1,\cdots,X_K],\quad
    C = \Pi\diag(C_1,\cdots,C_K)\Pi^T,
\end{align}
where $\{X_k\}$ is an MSS of $X$, and each $C_k$ is connected.
\end{definition}

Theorem \ref{thm:unconnect} also shows that, if $X$ has an MSS $\{X_k\}$, its representation matrix $C$ with rank $r(C) > \sum_k r(X_k)$ may be not an MSDR for the MSS, even if $C$ has a block-diagonal form partitioned according to the MSS, since some of the diagonal blocks may be unconnected. In such a case, the unconnected diagonal blocks can be divided into smaller (connected) ones. Thus, $C$ is also a block-diagonal representation matrix with greater than $K$ connected diagonal blocks. In addition, the $K$-partition learned by spectral clustering may give a nonminimal segmentation.
Fortunately, connection issues can be addressed if the representation matrix has a rank equal to $\sum_kr(X_k)$.

\begin{theorem} \label{thm:rank_connect}
Under the same assumptions of Theorem \ref{thm:minimal partition}, if $C$ is a block-diagonal representation matrix of $X$ with $\rank(C) = d$ and $K$ diagonal blocks, each greater than $d_{\min}$ in size, then $C$ is an MSDR. 
Furthermore, $C$ is unique if it is restricted to be symmetric.
\end{theorem}
\begin{proof}
Let $J = \{J_k\}$ be the partition with $K$ pieces corresponding to the block-diagonal form and let $X_k = X_{J_k}$. 
Obviously, $X_{J_k} C_k = X_{J_k}$. Since $d = r(C) = \sum_{k} r(C_k) \geq \sum_k r(X_{J_k})$ and the size $|J_k|$ of $C_k$ is greater than $d_{\min}$, we have $J = \{J_k\} \in \J(K,d,d_{\min}) $. By Theorem \ref{thm:minimal partition}, $\J(K,d,d_{\min}) $ is a singleton and thereby $\{X_{J_k}\}$ is an MSS.

Next, we show that all the diagonal blocks of $C$ are connected. 
Assume, inversely, that there is an unconnected $C_i$, which can be further partitioned into block-diagonal form with at least two diagonal blocks. That is, there is a permutation $\Pi_i$ such that: 
\[
    \Pi_i C_i \Pi_i^T = \diag(C_i', C_i''), \quad X_{J_i}\Pi_i = [X_i',X_i''].
\]
Let $n_i' = n(X_i')$ and $n_i'' = n(X_i'')$. Since $X_{J_i}$ is minimal, $r(X_i') = \min \{d_i, n_i'\}$ and $r(X_i'') = \min \{d_i, n_i''\}$. Hence, using $n_i'+n_i'' = |J_i| > d_i$,
\begin{align*}
    r(C_i)  = r(C_i')+r(C_i'') 
    & \geq r(X_i')+r(X_i'') = \min \{d_i, n_i'\} + \min \{d_i, n_i''\} \\
    & \geq \min \{d_i+n_i', d_i+n_i'', d_i+d_i, n_i'+n_i'' \} > d_i.
\end{align*}
Combining this inequality with $r(C_k) \geq d_k$ for $k \neq i$, we obtain $r(C) = \sum_k r(C_k) > \sum_k d_k = d$, a contradiction to $\rank(C) = d$. Hence, $C$ is an MSDR of $X$. If $C$ is symmetric, each $C_k$ is also symmetric, and $C_k = V_kV_k^T$ by $r(C_k) = r(X_{J_k}) = d_k$. Hence, $C$ is unique.
\end{proof}

\subsection{Computable Modeling for MSDR}\label{sec:model}

We are now ready to model the MSDR as an optimization problem, mainly motivated by Theorem \ref{thm:rank_connect}.
As in previous sections, we also assume that the number $K$ and the rank sum $d$ of segments are known for the MSS $\{X_k\}$. By Theorem \ref{thm:rank_connect}, we restrict the feasible representation $C$ to be of rank $d$. Since the symmetric MSDR is unique, we further restrict it to be symmetric as $C = C(S)$ with the following: 
\[
	C(S) = VV^T+V_\bot SV_\bot^T, \quad r(S) = d-r, \quad S^T = S,
\]
where $r = r(X)$ is known. The mapping from a symmetric $S$ to a symmetric $C$ is one-to-one. That is, given a symmetric $C$ with $\rank(C) = d$ in the imaging domain, there is a unique $S = V_\bot^TCV_\bot$ satisfying $C(S) = C$. To enforce $C$ to be block-diagonal reasonably, we hope that the off-block-diagonal part of $C$, defined as $C_{{\rm off}(J)}(S)$ with entries
\[
    \big(C_{{\rm off}(J)}(S)\big)_{ij} 
    = \left\{\begin{array}{cl}
         0 & i,j\in J_k,\, \forall k; \\
         c_{ij} & i\in J_s,\,j\in J_t,\,\forall s\neq t,
    \end{array}
    \right.
\]
is as small as possible with $J=\{J_k\}\in{\cal J}(K,d,d_{\min})$. We adopt the $\ell_1$-norm for minimizing this off-block-diagonal part of $C$. That is, we solve the following optimization model
\begin{align} \label{prob:MSDR_0} 
    \min_{J\in{\cal J}(K,d,d_{\min})}\min_{S\in\mathbb S} \ \| C_{{\rm off}(J)}(S)\|_1
\end{align}
for determining an MSDR $C$ of $X$, where $\mathbb S$ is a feasible domain of symmetric
matrices of order $n-r$. 
Theorem \ref{thm:rank_connect} supports the model (\ref{prob:MSDR_0}) to give an MSDR, because the solution $C$ of (\ref{prob:MSDR_0}) should be a block diagonal matrix of $K$ diagonal blocks, with the size of each block greater than $d_{\min}$.

There are various options for the feasible domain $\mathbb S$. By Theorem \ref{thm:rank_connect}, the symmetric MSDR should be a positive semidefinite matrix and orthogonal projection operator with rank $d$. A feasible domain should contain such a matrix. For example, choose $\mathbb S$ as the set of orthogonal projection matrices as follows:
\[
    {\cal P}_{d-r}^{n-r}= \{S=QQ^T: Q \in {\cal O}_{d-r}^{n-r}\big\},
\]
where ${\cal O}_{d-r}^{n-r}= \{Q\in {\cal R}^{(n-r)\times(d-r)}: Q^TQ = I_{d-r}\big\}$ is a Stiefel manifold. 
Obviously, for an $S=QQ^T$ with $Q\in {\cal O}_{d-r}^{n-r}$, $C(S) = GG^T\in {\cal P}^n_d$ with an orthonormal $G=[V,V_\bot Q]\in{\cal O}_d^n$. However, the Stiefel manifold ${\cal O}_{d-r}^{n-r}$ is strongly nonconvex; thus, one may encounter a local optimum with a solution far from the MSDR, taking $Q$ as a variable in ${\cal O}_{d-r}^{n-r}$. The largely flat domain is the subspace ${\cal R}^{(n-r)\times(n-r)}$, but it misses the special structure of the MSDR. 
In this paper, we choose the feasible domain as the set of symmetric positive semidefinite matrices with rank $d$ for $C=C(S)$,  
\[
	{\cal S}_{d-r}^{n-r} = \big\{S=WW^T: W\in {\cal R}_+^{(n-r)\times(d-r)}\big\},
\]
where ${\cal R}_+^{(n-r)\times(d-r)}$ is the set of full rank matrices in ${\cal R}^{(n-r)\times(d-r)}$. ${\cal S}_{d-r}^{n-r}$ is a slightly larger manifold than ${\cal P}_{d-r}^{n-r}$. However, it is much flatter than ${\cal P}_{d-r}^{n-r}$, which benefits convergence when we iteratively solve the optimization problem (\ref{prob:MSDR}) presented later in the paper.

\subsection{Comparison with Related Work}\label{subsec:comparison}

The optimization model (\ref{prob:MSDR}) can handle cases wherein the minimal sample subspaces are intersected with each other and the intersections of pairwise subspaces are potentially significant and variant. To showcase this advantage, we compare the sufficient conditions in Theorem \ref{thm:minimal partition} with the conditions for LRR, iPursuit \cite{Mostafa2017Innovation}, SSC, and LRSSC.

As shown in Theorem \ref{thm:LRR-SSC}, the LRR obtains the MSDR if and only if the subspaces are independent, that is, $\dim(\sum_j \S_j) = \sum_j \dim(\S_j)$. This condition implies that each subspace does not intersect with the sum of the other subspaces, or equivalently, $d_{\rm int} = d_0 = 0$, which is a much stricter condition than that given in Theorem \ref{thm:minimal partition}.
It was proven by \cite{Mostafa2017Innovation} that the iPursuit can separate two subspaces ($K=2$) with high probability. This amounts to one of the special cases shown in Corollary \ref{coro:same_d}. The condition for LRSSC is similar to that of SSC in the same form, yet it is stricter. We omit a comparison of our method's sufficient condition with that of LRSSC, but a detailed comparison with SSC is given below.

For the SSC, \cite{Soltanolkotabi2012A} showed that if the samples are uniformly distributed in the union of subspaces $\{\S_k\}$, and, for the basis matrices $\{U_k\}$ of $\{\S_k\}$, $\|U_k^TU_\ell\|_F^2< \min\{\gamma_{k,\ell}, \gamma_{\ell,k}\}$ with the following:
\begin{align}\label{def:gamma}
	\gamma_{k,\ell} = \frac{1}{256} \frac{d_\ell\log((n_k-1)/d_k)}{\left(\log(n_k(n_\ell+1)) + \log (K) + t \right)^2},
\end{align}
where $t>0$ is a given parameter, then SSC can give a block-diagonal solution partitioned as the ideal subspace segmentation with a probability approximately equal to one, depending on $t$, $n_k$, $d_k$, and $K$. We note that this claim does not imply a connected solution as we have explained earlier and mentioned by \cite{Nasihatkon2011Graph}. 

Obviously, a small $\|U_k^TU_\ell\|_F$ implies approximate orthogonality between $\S_k$ and $\S_\ell$. The following lemma further shows that the inequality $\|U_k^TU_\ell\|_F^2<1$ implies that the two subspaces are not intersected with each other.

\begin{lemma} \label{lemma:aff_dim}
Let $\S_1$ and $\S_2$ be two arbitrary subspaces with basis matrices $U_1$ and $U_2$, respectively.  Then,
\[
	\dim \left( \S_1 \cap \S_2 \right) \leq \|U_1^TU_2\|_F^2.
\]
\end{lemma}
\begin{proof}
The orthogonal basis matrices of two intersected subspaces can be extended via the basis of their intersected subspace. That is, using
the basis $U_0$ of $\S_1 \cap \S_2$, the orthogonal basis matrices $U_1$ and $U_2$ of 
$\S_1$ and $\S_2$, respectively, can be represented as $U_1 = [U_0, \hat U_1]G_1$ and $U_2 = [U_0, \hat U_1]G_2$ with orthogonal $G_1$ and $G_2$, and orthonormal $\hat U_1$ and $\hat U_2$ satisfying $U_0^T \hat U_1= 0$, $U_0^T \hat U_2=0$. Hence,
\[
	\|U_1^TU_2\|_F^2 = \|[U_0, \hat U_1]^T [U_0, \hat U_2]\|_F^2 = \|U_0^TU_0\|_F^2 + \|\hat U_1^T \hat U_2\|_F^2 
	= \dim(\S_1 \cap \S_2) + \|\hat U_1^T \hat U_2\|_F^2,
\]
which implies $\dim \left( \S_1 \cap \S_2 \right) \leq \|U_1^TU_2\|_F^2$.
\end{proof}

Therefore, if $\|U_1^TU_2\|_F<1$, $\dim \left( \S_1 \cap \S_2 \right)=0$, that is, $\S_1$ and $\S_2$ are independent by Lemma \ref{lemma:aff_dim}. Since the upper bound $\gamma_{k,\ell}$ tends to zero quickly as $n_k$ or $t$ increases, even with a small $K$ such as $K=2$, the sufficient conditions $\|U_k^TU_\ell\|_F^2< \min\{\gamma_{k,\ell}, \gamma_{\ell,k}\}$ for all $k,\ell$ with $\gamma_{\ell,k}$ defined in (\ref{def:gamma}) generally
imply the existence of pairwise independent subspaces. In practice, if $d_\ell\leq 1937$, then
\[
	\gamma_{k,\ell}\leq \frac{1}{256} \frac{d_\ell}{\log(n_k(n_\ell+1))}\frac{\log((n_k-1)/d_k)}{\log(n_k(n_\ell+1))}
	<\frac{1}{256} \frac{d_\ell}{\log(n_k)+\log(n_\ell+1)}
	<\frac{1}{256} \frac{d_\ell}{\log d_\ell}<1
\]
since $n_\ell>d_\ell$. Furthermore, if $n_k\leq d_\ell^2n_\ell$, then $\gamma_{k,\ell}<\frac{1}{512} \frac{d_\ell}{\log d_\ell}$, and hence, $\gamma_{k,\ell}<1$ if $d_\ell\leq 4281$. In real applications, the subspace dimensions are generally much smaller than 1937. Hence, if the conditions for SSC are satisfied, then $d_{\rm int} = 0$, which differs from the independence condition, {\it i.e.}, $d_{\rm int} = d_0 =0$, for LRR.

In the conditions of Theorem \ref{thm:minimal partition}, we permit a relatively large intersection that does not depend on the number of samples in each subspace. By Theorems \ref{thm:ndgt} and \ref{thm:minimal partition}, the minimal sample subspaces can be recovered by the optimal solution of (\ref{prob:MSDR_0}) with probability 1.

\section{Algorithms}\label{sec:algorithm}

We encounter several computational difficulties when we try to solve the problem (\ref{prob:MSDR_0}). First, it is difficult to check whether the restriction $|J_k|>d_{\min}$ is satisfied since $d_{\min}$ is unknown if $\{d_k\}$ are variant.\footnote{If all the minimal segments have equal rank, $d_{\min} = d/K$ is known.} Second, it is inconvenient to check the restriction $\sum_k r(X_{J_k}) \leq d$. Third, the objective function of (\ref{prob:MSDR}) is neither continuous nor convex, and contains discrete and continuous variables with respect to the partition $J$ and $S$ or the factor in its symmetric factorization. 
Fortunately, the strict condition $|J_k|>d_{\min}$ can be implicitly satisfied when the strategy of normalized cutting is adopted for updating $J$ generally. In the case when $C$ is block-diagonal with $K$ diagonal blocks, the inequality $\sum_k r(X_{J_k})\leq d$ holds automatically since $r(X_{J_k})\leq r(C_k)$ and $\sum_kr(C_k) = r(C) = d$. Hence, we can remove the restrictions $|J_k|>d_{\min}$ and $\sum_k r(X_{J_k})\leq d$ in ${\cal J}(K,d,d_{\min})$. That is, we relax ${\cal J}(K,d,d_{\min})$ to the set ${\cal J}(K)$ of all $K$-partitions, and (\ref{prob:MSDR_0}) is slightly modified to 
\begin{align} \label{prob:MSDR} 
	\min_{J \in{\cal J}(K)}\min_{S\in\mathbb S} \ \| C_{{\rm off}(J)}(S)\|_1, 
\end{align}
where ${\mathbb S} = \S^{n-r}_{d-r} = \big\{S=WW^T: W\in {\cal R}_+^{(n-r)\times(d-r)} \big\}$. 

The difficulty of mixing discrete and continuous variables can be addressed via alternatively optimizing $J$ and $C(S)$. However, special strategies should be considered to improve the efficiency of this computation. We offer two types of alternative algorithms for this purpose. One algoritm solves (\ref{prob:MSDR}) directly based on a manifold conjugate gradient (MCG) method for optimizing $C$. The other algorithm solves an equivalent pseudo-dual problem of (\ref{prob:MSDR}) based on subspace estimation. Both methods solve the problem using the alternative rule: Optimize $C$ given $J$, and update $J$ according to the current $C$. 

However, these two methods cannot guarantee a globally optimal solution in any case. Thus, We hybridize them by taking the solution of one method as the initial guess for the other. The motivation for this strategy is the rarity of falling into a common local minimizer of the both problems. Using this hybrid strategy, we can obtain the true minimal segmentation in our experiments if the subspaces are not heavily-intersected with each other.

\subsection{Alternative Method for the Primal Problem}

In the literature, alternative strategies are commonly used for optimizing multiple variables. For instance, an alternative strategy is adopted by \cite{Li2017Structured} for minimizing the similar objective function $\|C_{{\rm off}(J)}\|_1+\alpha\|C\|_1$. 
It is potentially easy to optimize $C$ given partition $J$, and $J$ can be updated via normalized spectral clustering on the symmetric graph $|C|+|C|^T$ given $C$. However, if the spectral clustering is unstable, it may give an undesired partition when $C$ is far from the ideal solution. Conversely, a poor partition also leads to an unacceptable solution. To decrease instability, a soft version is also considered by \cite{Li2017Structured}, in which the function $\|C_{{\rm off}(J)}\|_1$ is modified to the weighted $\ell_1$-norm function $\sum_{ij}w_{ij}|c_{ij}|$ with weights $w_{ij}=\frac{1}{2}\|u_i-u_j\|^2$, where $u_i$ is a the vector of $i$-th components of the $K$ eigenvectors corresponding to the $K$ smallest eigenvalues of $|C|+|C|^T$. 
However, this method blurs block separation, and hence, it may also result in an unacceptable $C$.

We apply two types of modifications for solving the primal problem (\ref{prob:MSDR_0}) using an alternative strategy. The first modification acts on the $K$-partition $J$. 
Different from the commonly used weight strategy,
we slightly extend the support domain ${\rm off}(J)$ in the objective function $\|C_{{\rm off}(J)}(S)\|_1$ to an active index set $\Omega$
that covers ${\rm off}(J)$. For the sake of simplicity, $\Omega$ also refers to an indication matrix whose entries $\omega_{ij}$ are 1 for the indices in $\Omega$ and zero otherwise.
Hence, the function $\|C_{{\rm off}(J)}(S)\|_1$ becomes 
$\|\Omega\odot C(S)\|_1 =\sum_{(i,j)\in \Omega}|c_{ij}(S)|$. 
This modification can significantly reduce the risk of obtaining an incorrect partition $J$, especially in the initial case when $J$ is poorly estimated. 
Initially, we choose $\Omega$ to be the coarsest $\Omega_c$ with $\omega_{ij} = 1$ for $i\neq j$ and $\omega_{ii} = 0$, {\it i.e.}, $\|\Omega_c\odot C(S)\|_1 = \sum_{i\neq j}|c_{ij}|$. 
In a later subsection, we discuss how to update the active index set $\Omega$ so that it can approach
the subdomain ${\rm off}(J)$ as soon as $J$ is approximately optimal.

The second modification aims to reduce the degree of nonconvexity of the function $\|\Omega\odot C(S)\|_1$ given $\Omega$ to render the modified function a bit flatter so that an iteration algorithm is less likely to fall into a local minimizer. To this end, we add the prior term $\frac{\lambda}{2}\|c(S)\|_2^2$ onto the diagonal vector $c(S)$ of $C(S)$ with parameter $\lambda>0$. This strategy also benefits the search for a block-diagonal solution. Since we relax the strict zero-restriction on the diagonals, the prior term penalizes the diagonals of $C(S)$, and hence, the diagonals are uniformly small in general, which helps to increase the connections within each subspace in the representation $X=XC$.\footnote{If there is a diagonal $c_{ii}\approx 1$, the connections of sample $x_i$ to the others nearly vanish.} 

Combining the two modifications, we modify  
$\min_{S\in \mathbb S}\|C_{{\rm off}(J)}\!(S)\|_1$ to the following:
\begin{align}\label{prob:basic}
	\min_{S\in \mathbb S}\Big\{\|\Omega\odot C(S)\|_1+\frac{\lambda}{2}\|c(S)\|_2^2\Big\}.
\end{align}
Since $S$ (or $C(S)$) and $\Omega$ are updated alternatively, the penalty parameter $\lambda$ should balance the two terms $\|\Omega\odot C(S)\|_1$ and $\frac{1}{2}\|c(S)\|_2^2$. Thus, it makes sense to set 
\begin{align}\label{def:lambda}
	\lambda = \min\big(\lambda_0, 2\|\Omega\odot\hat C\|_1/\|\hat c\|_2^2\big)
\end{align}
adaptively, using the solution $\hat C$ corresponding to the previous $(\hat\Omega,\hat \lambda)$, $\hat c=\diag(\hat C)$, and $\lambda_0$ is an initial setting. This strategy is efficient in our experiments.

The basic model (\ref{prob:basic}) works well on some but not all complicated subspaces---it can recover the minimal segmentation of samples from some intersected subspaces if they are not heavily intersected with each other. We show the performance of this basic model compared with other state-of-art methods in the experiment section of this paper.

\subsection{MCG: Manifold Conjugate Gradient Method}\label{sec:algorithm_MCG}

The problem (\ref{prob:basic}) can be solved using a manifold conjugated gradient (MCG) method, but some computational issues should be addressed before applying MCG on (\ref{prob:basic}). First, the objective function in (\ref{prob:basic}) is not derivable. A subgradient is used as a substitute of the gradient in our analysis. Second, MCG convergence analysis requires the objective function to be smoothed. The gradient of this smooth function is a good approximate of a subgradient of the original function. Third, the gradient vectors should be projected onto the tangent space of the manifold at a point in MCG. However, only a smaller subspace of the tangent space benefits linear searching in MCG. For efficient computation, this subspace must be detected. In this subsection, we give a detailed MCG algorithm for solving (\ref{prob:basic}), taking into account the above concerns and the technique of linear searching, together with convergence analysis. We also discuss some computational details of the MCG.

\subsubsection{Subgradients}\label{sec:algorithm_subgrad}

Writing $S=WW^T\in\S$ with $W\in\R^{(n-r)\times (d-r)}$, the objective function of (\ref{prob:basic}) is as follows:
\[
	f(W) = \|\Omega\odot C(WW^T)\|_1+\frac{\lambda}{2}\|c(WW^T)\|_2^2.
\] 
It is known that a subgradient of the function $|x|$ at a real variable $x$ is $\sgn(x)$ if $x\neq 0$ or any real $r\in [-1,1]$ when $x=0$. Since the function $\|\Omega\odot C\|_1$ is separable on its variables, the set of subgradients of function $f(C)$ at $C$ is
\[
	\partial_C \|\Omega\odot C\|_1 
    = \big\{\Omega\odot(\sgn(C)+R):\ R\in {\cal R}_C \big\}
\]
where ${\cal R}_C = \{R: \ R\odot C = 0, \|R\|_\infty \leq 1\}$. 

For $C = C(S) = VV^T+V_\bot S V_\bot^T$ with symmetric $S$ specially, the definition of subgradients gives 
the inequality
$
	\|\Omega\odot Y\|_1-\|\Omega\odot C\|_1
    \geq \langle B, Y-C\rangle
$
for a fixed $B\in\partial_{C(S)} \|\Omega\odot C\|_1$ and all $Y$.
Hence, choosing $Y = C(T)$ with any $T$, we obtain 
\begin{align}\label{partial_S}
	\|\Omega\odot C(T)\|_1-\|\Omega\odot C(S)\|_1 
    \geq \langle B, C(T)-C(S)\rangle
    = \langle V_\bot^TBV_\bot, T-S\rangle.
\end{align}
Obviously, the subgradients of convex function $\|\Omega\odot C(S)\|_1$ at $S$ are $B_V = V_\bot^TBV_\bot$ with symmetric $B\in\partial_{C(S)}\|\Omega\odot C\|_1$, which is concluded by setting symmetric $T$ in (\ref{partial_S}). 

The subgradients of the non-convex function $\|\Omega\odot C(WW^T)\|_1$ at $W$ can be also concluded from (\ref{partial_S}), based on Definition 8.3 in 
\cite{Rockafellar1998Variational} for a non-convex function $g(x)$ at $x_0$, via the inequality 
\[
	g(x)-g(x_0)\geq \langle s,x-x_0\rangle + o(\|x-x_0\|).
\]
Let $S(W)=WW^T$ as a mapping of $W$, and choose $S = S(W)$ and $T = S(W+\Delta)$ in (\ref{partial_S}). We see that
$
	\|\Omega\odot C(S(W+\Delta))\|_1-\|\Omega\odot C(T)\|_1 
    \geq \langle 2V_\bot^TBV_\bot W, \Delta\rangle 
    + O(\|\Delta\|_F^2)
$.
Hence, 
\[
	\partial_W \|\Omega\odot C(WW^T)\|_1 
    = \big\{2V_\bot^TBV_\bot W:\ 
    B\in\partial_{C(S)}\|\Omega\odot C\|_1\big\}.
\]
Combining it with the gradient $2\lambda V_\bot^T\diag(c(WW^T))V_\bot W$ of $\frac{\lambda}{2}\|c(WW^T)\|_2^2$, we get
\begin{align*}
	\partial f(W) 
    &= \big\{2\big(\hat B_V+\lambda V_\bot^T\diag(c(S))V_\bot\big)W:\ \hat B_V\in\partial_S\|\Omega\odot C(S)\|_1,\ S = WW^T \big\}\\
    &= \big\{2V_\bot^T\big(\Omega\odot(\sgn(C)+R)+\lambda\diag(c)\big)V_\bot W:\ R^T = R\in{\cal R}_C,\ C = C(WW^T)\big\}.
\end{align*}

Convergence analysis requires a differentiable objective function. However, $f(W)$ is continuous but not differentiable on the zero entries of $\Omega\odot C(WW^T)$. To polish $f(W)$, we use the derivable function 
\begin{align}\label{def:q_delta}
	q_\delta(t) = \left\{\begin{array}{ll}
	    |t| & |t|>\delta;\\
        \frac{1}{2\delta}(t^2+\delta^2) & |t|\leq \delta
	\end{array}\right.
\end{align}
that polishes $|t|$ within a small threshold $\delta>0$. Its derivative $q_\delta'(t) = \sgn(t)\min\{|t|/\delta,1\}$ is an approximate of subgradient $s_t = \sgn(t)$ of $|t|$ with error $s_t-q_\delta'(t)= \sgn(t)(1-|t|/\delta)_+$, where $t_+ = \max\{t,0\}$.
Hence, 
\begin{align}\label{f_delta}
	f_\delta(W) = \|\Omega\odot C(WW^T)\|_{1,\delta}
    	+\frac{\lambda}{2}\|c(WW^T)\|_2^2
\end{align}
is a polishing function of $f(W)$, where $\|C\|_{1,\delta} =\|q_\delta(C)\|_1 = \sum_{ij}q_\delta(c_{ij})$.
Since $\|C\|_{1,\delta}$ has the gradient $\grad\|C\|_{1,\delta} = \sgn(C)\odot\min\{|C|/\delta,1\}$,  we obtain that, with $C = C(WW^T)$,
\begin{align}
	\grad f_\delta(W) 
    & = 2V_\bot^T\big(\nabla \|\Omega\odot C\|_{1,\delta}
    	+\lambda \diag(c\big)\big)V_\bot W \nonumber\\
    & = 2V_\bot^T\big(\Omega\odot\sgn(C)\odot\min\{|\Omega\odot C|/\delta,1\}
    	+\lambda \diag(c)\big)V_\bot W. \label{df_delta}
\end{align}
Here, we have used the equality $\Omega\odot\sgn(\Omega\odot C) = \Omega\odot\sgn(C)$. 

The gradient of $f_\delta(W)$ is an approximation of subgradiant $S_W=2V_\bot^T\big(\Omega\odot\sgn(C)+\lambda\diag(c)\big)V_\bot W$ of $f(W)$ corresponding to $R=0$. 
The error matrix is as follows:
\[
	E = S_W-\grad f_\delta(W) 
    = 2V_\bot^T\big(\Omega\odot\sgn(C)\odot
    	(1-|\omega_{ij}c_{ij}|/\delta)_+\big)V_\bot W.
\]
If $\delta$ is small enough such that $\delta \leq \min_{\omega_{ij}c_{ij} \neq 0} |\omega_{ij}c_{ij}|$, then $E = 0$, that is, $\grad f_\delta(W)$ is a subgradient of $f(W)$ at $W$. Furthermore, $\grad f_\delta(W) = 0$ means $0 \in \partial f(W)$.
We use $f_\delta(W)$ with a small $\delta$ in our MCG algorithm. The following lemma shows that a local or global optimizer of $f_\delta$ is also an approximately local or global optimizer of $f$ with an approximate error $O(\delta)$ in terms of the following:
\[
	n_\delta(W) = \sum_{|c_{ij}(W)|<\delta}\frac{(\delta-|c_{ij}(W)|)^2}{2\delta}
    \leq \frac{\delta}{2}\sum 1_{|c_{ij}(W)|<\delta}.
\]

\begin{lemma}
Let $W_\delta$ and $W_*$ be the minimizers of $f_\delta(W)$ and $f(W)$, respectively. Then
\[
	f(W_\delta)\leq f(W_*)
    + n_\delta(W_*).
\]
\end{lemma}
\begin{proof}
Since $q_\delta(t)=|t|$ if $|t|\geq \delta$, or $0\leq q_\delta(t)-|t|=\frac{(\delta-|t|)^2}{2\delta}$ otherwise. By definition, we have that 
$
   	f(W) \leq f_\delta(W)\leq f(W)+ n_\delta(W)
$
for any $W$. Hence, $f(W_\delta) \leq f_\delta(W_\delta) \leq f_\delta(W_*) \leq f(W_*) + n_\delta(W_*)$.
\end{proof}

\subsubsection{Tangent Space}
The method of nonlinear conjugate gradient (NCG) updates the current $W$ via a linear searching as $\hat W = W+\alpha \Delta$ for minimizing $f_\delta(W)$, where $\Delta$ is a conjugate gradient direction involved as a sum of a gradient $\grad f_\delta(W)$ and a conjugate gradient direction at the previous point. However, the NCG does not take into account the manifold ${\cal S}_{d-r}^{n-r}$ in our case. To take advantage of the manifold structure, it is required to slightly modify the conjugate gradient formula on the one hand. We will mention it in the next subsection.

On the other hand, the conjugate gradient $\Delta$ should be further modified \cite{journee2010low}. Practically, the modified point in the manifold ${\cal S}_{d-r}^{n-r}$, 
\[
    S(W+\alpha \Delta) 
    = S(W)+\alpha(W\Delta^T+\Delta W^T) + \alpha^2\Delta\Delta^T 
\]
contains the tangent component $T_W(\Delta) = W\Delta^T+\Delta W^T$ of the manifold at $S(W) = WW^T$. Obviously, the component of $\Delta$ that belongs to the null space ${\cal N}_W$ of the linear map $T_W(\Delta)$ does not contribute to the tangent space. If we split $\Delta = \Delta_N+\Delta_H$ with $\Delta_N\in {\cal N}_W$ and $\Delta_H$ in the orthogonal complement ${\cal H}_W$ of ${\cal N}_W$, Thus,
\[
    W\Delta^T+\Delta W^T = W\Delta_H^T +\Delta_H W^T,\quad
    \|\Delta\|_F^2 = \|\Delta_N\|_F^2 + \|\Delta_H\|_F^2.
\]
Clearly, condensing $\Delta$ into the horizontal set ${\cal H}_W$ does not change the tangent component, but it yields a new point $S(W+\alpha \Delta_H)$. It is closer to the tangent space than $S(W+\alpha \Delta)$. Notice that both $S(W+\alpha \Delta_H)$ and $S(W+\alpha \Delta)$ are retractions of the same modified point $S(W)+T_W(\Delta_H)$ onto the manifold ${\cal S}_{d-r}^{n-r}$ in the technique of manifold conjugate gradient (MCG). Therefore, the updating of $W$ should be modified as
\begin{align}\label{update:W}
	W_{\rm new} = W+\alpha_W {\cal P}_W(\Delta)
\end{align}
with a suitable step length $\alpha_W$ for linear searching for the (modified) conjugate gradient $\Delta$, where ${\cal P}_W(\Delta)=\Delta_H$ is the projection of $\Delta$ onto horizontal set ${\cal H}_W$. 

It is not difficult to determine the projection ${\cal P}_W(\Delta)$, via characterizing the subspaces ${\cal N}_W$ and ${\cal H}_W$. Practically, writing each $\Delta\in {\cal N}_W$ as $\Delta = WN+W_\bot F$ with an orthogonal complement $W_\bot$ of $W$,\footnote{We assume that $W$ is of full column rank for simplicity.} and using the equality $\Delta W^T+W\Delta^T=0$, we have that
\[
    0 = W^T(\Delta W^T+W\Delta^T)W
    = W^TW(N+N^T)W^TW.
\]
Hence, $N+N^T=0$ since $W$ is of full column rank. Moreover, 
\begin{align*}
    0 &= -WNW^T-WN^TW^T
    = (W_\bot F-\Delta)W^T+W(W_\bot F-\Delta)^T\\
    &= W_\bot FW^T+WF^TW_\bot^T
    = [W,W_\bot]\left(\begin{array}{cc}
        0 & F^T\\ F & 0
    \end{array}\right)[W,W_\bot]^T,
\end{align*}
which implies $F = 0$. Hence, ${\cal N}_W = \big\{WN: \ N^T=-N\in{\cal R}^{(d-r)\times(d-r)}\big\}$. Furthermore, its orthogonal complement is
${\cal H}_W = \big\{H\in{\cal R}^{(n-r)\times(d-r)}:\  W^TH = H^TW \big\}$ obviously since $\langle WN,H\rangle = 0$ for all skew-symmetric $N$ of order $d-r$.

To determine a skew-symmetric $N$ and $H\in {\cal H}_W$ from the splitting $\Delta = WN+H$, at first, we eliminate the symmetric $W^TH$ in the equality 
$W^T\Delta = W^TWN+W^TH$, by taking the skew-symmetric part of $W^T\Delta$. It yields the equation $W^TWN+NW^TW = E$, where $E = W^T\Delta-\Delta^TW$ is known. Thus, using the eigen-decomposition $W^TW = Q\Sigma Q^T$ and setting $\tilde N = Q^TNQ$ and $\tilde E=Q^TEQ$, this equation is simplified to $\Sigma \tilde N+\tilde N\Sigma = \tilde E$, and $\tilde N$ and $N$ can be easily obtained as that
\begin{align}\label{proj}
	\tilde N = \Big(\frac{\tilde e_{ij}}{\sigma_i+\sigma_j}\Big), \quad
    N = Q\tilde NQ^T.
\end{align}
Therefore, the linear projection of $\Delta$ is ${\cal P}_W(\Delta) = H = \Delta-WN$.

\subsubsection{Manifold Conjugate Gradients}

The conjugate gradient direction $\Delta$ in the NCG is recursively defined. In our case, we set $\Delta = G_W$, where the recursive definition of $G_W$ is  slightly modified as:
\[
    G_W = -\grad f_\delta(W) + \beta_W {\cal P}_{W_{\rm old}}(G_{W_{\rm old}}),
\]
and $W_{\rm old}$ is a previous point. Let $P_W={\cal P}_W\big(\grad f_\delta(W)\big)$ be the projection of $\grad f_\delta(W)$ onto ${\cal H}_W$.
The projection of $G_W$ onto ${\cal H}_W$, {\it i.e.}, the conjugate direction $H_W$ is also recursively defined \cite[Algorithm 13]{absil2009optimization}, 
\begin{align}\label{update:cg}
	H_W = 
    -P_W+\beta_W{\cal P}_W\big(H_{W_{\rm old}}\big).
\end{align}
Initially, $H_W = -P_W$. Thus, the iteration (\ref{update:W}) with $\Delta = G_W$ becomes 
\begin{align}\label{update:CG}
	W_{\rm new} = W+\alpha_W H_W,
\end{align}
which is an iteration of the manifold conjugate gradient method.

We use the following formula for setting the $\beta_W$ in (\ref{update:cg}) for updating the conjugate direction in the Riemannian manifold
\begin{align}\label{beta:HZ}
	\beta_W = \begin{cases}
		\frac{\langle P_W,Y_W\rangle}{\langle Y_W, Z_W \rangle} 
        - \frac{2\langle P_W,Z_W\rangle}{\langle Y_W, Z_W \rangle^2}\|Y_W\|_F^2, 
        & \mbox{if }\ \langle Y_W, Z_W\rangle\neq 0; \\
	0, & \mbox{otherwise},
	\end{cases}
\end{align}
where $Y_W = P_W - {\cal P}_W(P_{W_{\rm old}})$ and $Z_W = {\cal P}_W (H_{W_{\rm old}})$,  
a slight adaptation of that for the CG method in Euclidean space \cite{Hager2005A}. When $\langle Y_W, Z_W\rangle=0$, the iteration is restarted. Obviously, rescaling $H_{W_{\rm old}}$ does not change the updating process (\ref{update:cg}). 
Hence, one can normalize each $H_W$ in (\ref{update:CG}) to have a unit Frobenius norm if necessary for numerical stability.

\subsubsection{Line Searching}

One strategy for linear searching is to choose $\alpha_W$ satisfying the Armijo condition on $f(W)$
\begin{align}\label{cond:armijo}
	f(W_{\rm new}) \leq f(W) + \tau \alpha_W 
    \inf_{B_W\in\partial f(W)}\langle {\cal P}_W(B_W), {\cal P}_W(\Delta)\rangle.
\end{align}
with $\tau\in(0,1)$. Mathematically, $\langle {\cal P}_W(Y), {\cal P}_W(Z) \rangle = \langle Y, {\cal P}_W(Z) \rangle = \langle {\cal P}_W(Y), Z \rangle$ for any $Y$ and $Z$. Hence, only one projection is required in the inner production. In numerical computation, $f(W)$ is replaced by the smooth $f_\delta(W)$, and (\ref{cond:armijo}) is changed to that 
\begin{align}\label{cond:armijo_delta}
	f_\delta(W_{\rm new}) \leq f_\delta(W) 
    + \tau \alpha_W \big\langle P_W, H_W \big\rangle,
\end{align}
as suggested in Section 4.2 by \cite{absil2009optimization}.
Once the Armijo condition (\ref{cond:armijo_delta}) is satisfied and $H_W$ is a descending direction, {\it i.e.}, $\langle P_W, H_W \rangle<0$, then $f_\delta(W_{\rm new}) < f_\delta(W)$ is guaranteed. 

We note that the computational cost of checking for the Armijo condition is much lower than that of other strategies for determining an $\alpha_W$. For example, for the strong Wolfe conditions \cite{Iwai2015A}
\begin{align}\label{cond:wolfe2}
	f_\delta(W_{\rm new}) \leq f_\delta(W) + c_1 \alpha_W \langle P_W, H_W\rangle, \quad 
    |\langle P_{W_{\rm new}}, {\cal P}_{W_{\rm new}}(H_W)\rangle| \leq | c_2 \langle P_W, H_W \rangle|,
\end{align}
where the constants $c_1$ and $c_2$ satisfy $0 < c_1 < c_2 < 1$, an additional condition must be checked. For the convergence ${\rm liminf}_{k \to \infty} \|P_{W_k}\|_F = 0$ of MCG under the strong Wolfe conditions, \cite{Iwai2015A} suggests another rule for choosing $\beta_W$. 

\subsubsection{Convergence}

The following lemma benefits convergence analysis of the MCG with $\beta_W$ in (\ref{beta:HZ}) and linear searching satisfying the Armijo condition (\ref{cond:armijo_delta}).

\begin{lemma}\label{lma:-7/8}
If $\beta_W$ is chosen as (\ref{beta:HZ}), then for arbitrary $W$, 
\begin{align} \label{ineq:HZ}
	\langle P_W, H_W \rangle 
    \leq -\frac{7}{8} \|P_W\|_F^2.
\end{align}
\end{lemma} 
\begin{proof}
From updating (\ref{update:cg}) of $H_W$, we have
$
	\langle P_W, H_W \rangle
	= -\|P_W\|_F^2 + \beta_W \langle P_W, Z_W \rangle.
$
If $\langle Y_W, Z_W \rangle = 0$, then $\beta_W=0$ by (\ref{beta:HZ}), and $\langle P_W, H_W \rangle = -\|P_W\|_F^2$. Otherwise, 
\begin{align*}
	\beta_W\langle P_W, Z_W \rangle 
  = & \frac{1}{8}\Big(8\frac{\langle P_W, Z_W \rangle}{\langle Y_W, Z_W \rangle}\langle P_W,Y_W\rangle
     -16\frac{\langle P_W, Z_W \rangle^2}{\langle Y_W, Z_W \rangle^2}
     	\langle Y_W,Y_W\rangle\Big)\\
  = & \frac{1}{8}\Big(\|P_W\|_F^2 
  	- \big\| P_W - 4\frac{\langle P_W, Z_W\rangle} 
                {\langle Y_W, Z_W \rangle} Y_W\big\|_F^2\Big)
  \leq \frac{1}{8}\|P_W\|_F^2.
\end{align*}
Hence, 
$
	\langle P_W, H_W \rangle \leq -\|P_W\|_F^2 + \frac{1}{8}\|P_W\| _F^2 
    = -\frac{7}{8}\|P_W\|_F^2.
$
\end{proof}

Thus, if the Armijo condition (\ref{cond:armijo_delta}) is satisfied with $\alpha_W$, we have the decreasing property
\begin{align}\label{eqn:f_decreasing}
 	f_\delta(W_{\rm new}) \leq f_\delta(W)-\frac{7\tau\alpha_W}{8}\|P_W\|_F^2
    \leq f_\delta(W).
\end{align}
This equality holds only if $P_W=0$. Hence, 
starting with any point, the MCG converges in the sense that $P_{W_k}=0$ at a $W_k$ or 
\begin{align}\label{conv}
	\lim_{k \to \infty} P_{W_{k}} = 0.
 \end{align}
That is, the MCG converges globally.

Theoretically, for a sufficiently small $\delta$, the minimizer of $f_\delta$ is also a local minimizer of $f$, as previously mentioned. However, a smaller $\delta$ might yield slower convergence of the MCG algorithm, which frequently occurs in numerical experiments. We use the stepped strategy of decreasing $\delta$ and use the minimizer $W_\delta$ as an initial guess for the MCG with a smaller $\delta$. This strategy can accelerate convergence. 

\begin{theorem}\label{thm:decreasing}
Let $\{\delta_\ell\}$ be a decreasing sequence and $W^{(\ell)}$ be a solution of the manifold conjugate gradient method with (\ref{beta:HZ}), starting with the previous $W^{(\ell-1)}$ and satisfying the Armijo condition. Then, $\big\{f_{\delta_\ell}\big(W^{(\ell)}\big)\big\}$ is monotonously decreasing.
\end{theorem}
\begin{proof} Obviously, if $\delta'<\delta''$, then $q_{\delta'}(t)\leq q_{\delta''}(t)$ for $q_\delta(t)$ given in (\ref{def:q_delta}) and all $t$. Hence, $f_{\delta'}(W)\leq f_{\delta''}(W)$ for all $W$, and 
$
	f_{\delta_{\ell+1}}\big(W^{(\ell)}\big)
    \leq f_{\delta_\ell}\big(W^{(\ell)}\big).
$ 
This equality holds only if $\min_{ij}|c_{ij}^{(\ell)}|\geq \delta_{\ell}$ since $q_{\delta'}(t)=q_{\delta''}(t)$ only if $|t|\geq \delta''$.
By (\ref{eqn:f_decreasing}),
$f_{\delta_{\ell+1}}\big(W^{(\ell+1)}\big) 
\leq f_{\delta_{\ell+1}}\big(W^{(\ell)}\big)$. This equality holds only if $P_{W^{(\ell)}} = 0$. Therefore, 
$
    f_{\delta_{\ell+1}}\big(W^{(\ell+1)}\big) 
    \leq f_{\delta_\ell}\big(W^{(\ell)}\big)
$
for all $\ell$, and this strict inequality holds if $P_{W^{(\ell)}} \neq 0$ or $\min_{ij}|c_{ij}^{(\ell)}|< \delta_{\ell}$. 
\end{proof}

\subsubsection{Computational Details}
Several computational issues may affect the efficiency of the MCG algorithm: the stopping condition of the inner iteration of $W$ given $\delta$, the rule for choosing a suitable $\alpha_W$ satisfying the Armijo condition, the choice of the initial testing value of $\alpha_W$, and the choice of $\delta$. We offer details on these computational issues below.

{\it Stopping criterion.} Given $\delta>0$, we normalize $H_W$ to have a unit Frobenius norm prior to linear searching. Since $\alpha_W = \alpha_W\|H_W\|_F = \|W_{\rm new}-W\|_F$, a simple stopping criterion of the iteration of $W$ is that $\alpha_W\leq \epsilon_{\alpha}$ 
with a small constant $\epsilon_{\alpha}$. 

{\it Choosing $\alpha_W$.} To guarantee convergence by Corollary 4.3.2 of \cite{absil2009optimization}, we determine an $\alpha_W$ such that $\alpha_W$ satisfies the Armijo condition but $\alpha_W' = \alpha_W/\rho$ does not. This is accomplished via repeatedly testing $\alpha$ in the rule: $\alpha:=\alpha/\rho$ if (\ref{cond:armijo}) holds or $\alpha:=\alpha*\rho$ otherwise, starting with an initial value $\alpha_0$. This is basically an estimation of 
the largest $\alpha$ satisfying the Armijo condition. Taking $\alpha_W^*$ as a good approximation of the minimizer of $\phi(\alpha)=f_\delta(W+\alpha H_W)$, the relative approximation error is bounded,
\[
	0\leq \frac{\alpha_W^*- \alpha_W}{\alpha_W}
    \leq \frac{\alpha_W/\rho-\alpha_W}{\alpha_W} = \frac{1-\rho}{\rho}.
\]
Hence, a $\rho$ closer to 1 yields a better approximate $\alpha_W$ to $\alpha_W^*$, and hence, a smaller value of $f_\delta$, roughly speaking. We typically choose $\rho\in[0.5,1)$.

{\it Initial guess of $\alpha_W$.} For simplifying the discussion, we normalize $H_W$ to have a unit Frobenius norm prior to linear searching. Since $\alpha_W$ tends to zero as the iteration of $W$ converges, 
a good estimate for $\alpha_W$ is $\alpha_{W_{\rm old}}$ if the previous $\alpha_{W_{\rm old}}$ is available. This initial guess works well in our experiments---it only takes twice testings for each update of $W$ in general, but may fail when the curvature of $\phi(\alpha) = f_\delta(W+\alpha H_W)$ achieves a local minimum near $\alpha = 0$, which may result in a very small $\alpha_W\approx 0$. 
This phenomenon happens when $W$ is close to a local minimizer or when the direction $H_W$ is unsuitable, causing very slow descent.
Thus, we change $H_W$ back to $-P_W$ if the Armijo condition is unsatisfied under at most $k_{\max}$ testings in case the computational cost becomes prohibitive. In our experiments, we generally set $k_{\max} = 10$. 

{\it Setting $\delta$.} In practice, a sequence of decreasing $\{\delta_\ell\}$ is used. We simply choose $\delta_\ell = \gamma^{\ell-1}\delta_0$ with $\gamma<1$. Let $W^{(\ell)}$ be the solution corresponding to $\delta_\ell$. We terminate the outer iteration if $\|\Omega\odot\big(C_{W^{(\ell)}}-C_{W^{(\ell-1)}}\big)\|_{\infty}<\epsilon_C$ with a given accuracy or $\delta_\ell\leq\epsilon_\delta$, where $\epsilon_\delta$ is a small constant such as $\epsilon_\alpha$.

The MCG algorithm is summarized in Algorithm \ref{alg:mcg}.

\begin{algorithm}[t]
\caption{MSS using manifold conjugated gradients (MSS$_-$MCG)}\label{alg:mcg}
Input: $V$, $V_\bot$, $\Omega$, initial guess $W$, $\delta_0$, and $\alpha_{\rm init}$, 
	parameters $\rho$, $\gamma$, $\epsilon_\alpha$, $\epsilon_C$, $\epsilon_\delta$, $\ell_{\max}$, and $k_{\max}$\\
Output: $W$ and $C$.
\begin{algorithmic}[1]
    \STATE Compute $C_v \!=\! VV^T$, $V_w \!=\! V_\bot W$, and $C \!=\! C_v+V_wV_w^T$, set $\delta \!=\! \delta_0$, and save $W^{(0)} \!=\! W$.
    \STATE For $\ell=1,2,\cdots, \ell_{\max}$
    	\STATE \hspace{15pt} Save $C_{\rm old} = C$, and compute $f = f_\delta(W)$ as (\ref{f_delta}) with the current $C$. 
	\STATE \hspace{15pt} For $k=0,1,2,\cdots,k_{\max}$
	\STATE \hspace{30pt} Compute $\grad f_\delta(W)$ as (\ref{df_delta}) and $P = {\cal P}_{W}(\grad f_\delta(W))$ as (\ref{proj}).
	\STATE \hspace{30pt} Set $H \!=\! -P$ if $k\!=\!0$, or compute $H$ as (\ref{update:cg}) and (\ref{beta:HZ}), and $\frac{H}{\|H\|_F}\!\to\! H$ if $k\!>\!0$. 
    \STATE \hspace{30pt} Starting with $\alpha_{\rm init}$, choose $\alpha$ satisfying the Armijo condition, but $\alpha/\rho$ does not. 
    \STATE \hspace{30pt} Update $W:= W+\alpha H$, $V_w = V_\bot W$, and reset $\alpha_{\rm init} = \alpha$.
    \STATE \hspace{30pt} Update $C = C_v+V_wV_w^T$, $f = f_\delta(W)$.
    \STATE \hspace{30pt} If $\alpha<\epsilon_\alpha$, terminate the inner iteration.
    \STATE \hspace{15pt} End
    \STATE \hspace{15pt} If $\|\Omega\odot (C-C_{\rm old})\|<\epsilon_C$ and $\delta < \epsilon_\delta$, terminate, otherwise, reduce $\delta: = \gamma\delta$.
    \STATE End
\end{algorithmic}
\end{algorithm}

\subsection{Active Set Updating}\label{sec:update Omega}

Once we obtain a solution $C=C(S)$ of (\ref{prob:basic}) with an active set $\Omega$, as an estimated solution of $\min_{S\in \mathbb S}\|C_{{\rm off}(J)}(S)\|_1$, we must update the current $\Omega$ together with $\lambda$ as (\ref{def:lambda}). 
In this subsection, we provide an effective approach for updating the active set $\Omega$, that addresses two issues in the unnormalized spectral clustering for estimating the $K$-partition $J$: small segments and instability of classical $k$-means. 

There is an implicit restriction $|J_k|>d_{\min}$ with unknown $d_{\min}$ for partition $J\in {\cal J}(K)$ in practice. This restriction implies that $J_k$ should not be small. Hence, we adopt normalized cutting \cite{Shi2000Normalized} to avoid small blocks in learning $J$. For the sake of completeness, we briefly describe the approach taken in this paper, which is similar to that of \cite{vonluxburg2007a}. 

Normalized cutting modifies $\|C_{{\rm off}(J)}\|_1 = \frac{1}{2}\sum_{ij} |c_{ij}|\| u_i-u_j\|_2^2$ to $\frac{1}{2}\sum_{ij} |c_{ij}| \|v_i - v_j\|_2^2$,
by just changing the assignment vectors $u_i = e_k$ of $J$ to the rescaled vector $v_i = e_k/\sqrt{\sum_{j \in J_k}\alpha_j}$ for $i\in J_k$, where $e_k$ is the $k$-th column of $I_K$, the identity matrix of order $K$, and $\alpha_j=\sum_i |c_{ji}|$. Hence, $u_i = v_i/\|v_i\|$ can be determined by the solution of the equivalent problem $\min\tr(VLV^T)$ subjected to $V=[v_1,\cdots,v_n]$ with discrete entries and $VDV^T= I_K$, where $L = D-(|C|+|C|^T)/2$ and $D$ is a diagonal matrix of scales
$\big\{\sum_{j=1}^n\frac{|c_{ij}|+|c_{ji}|}{2}\big\}$.
The discrete restriction is released for computation, and hence, $V$ is estimated by the solution of $\min \tr(YLY^T)$ subjected to $YDY^T=I_K$, which is $Y = Q^TD^{-1/2}$ with $Q$ of $K$ unit eigenvectors of $D^{-1/2}(|C|+|C|^T)D^{-1/2}$ corresponding to the $K$ largest eigenvalues. 
Therefore, $\{u_i\}$ is estimated by $\{\tilde y_i=y_i/\|y_i\|\}$, or equivalently, partition $J$ is estimated by the $k$-means clustering of $\{\tilde y_i\}$. That is, we assign labels for $\{\tilde y_i\}$ according to the centroids $\{b_k\}$ given by $k$-means as follows:
\[
	\ell(\tilde y_j)= \arg\min_k\|\tilde y_j-b_k\|_2, \quad j=1,\cdots,n
\]
where 
$J = \{J_1,\cdots,J_K\}$ with $J_k = \{j:\ \ell(\tilde y_j)=k\}$.

\begin{algorithm}[t]
\caption{Construct active set $\Omega$ }\label{alg:omega}
Input: a symmetric graph $A$ and the parameter $\tau$.\\
Output: active set $\Omega$ and partition $J$.
\begin{algorithmic}[1]
    \STATE Compute the sum $a = \sum_j a_j$ of all columns of $A$ and set $D = \diag(a)$.
    \STATE Compute $K$ unit eigenvectors $Q$ of $D^{-1/2}AD^{-1/2}$ with the largest eigenvalues. 
    \STATE Apply $k$-means on the normalized rows $\! \{\tilde y_j\} \! $ of $Q$ to get centroids $\{b_k\}$ and partition $J$.
    \STATE Compute $q_{i\ell}  = \frac{\psi(\tilde q_{i\ell})}{\sum_k \psi(\tilde q_{i \ell})}$ via (\ref{q:centriod}) 
    	and $\psi(t)=\left\{\begin{array}{ll} 1,& {\rm if }\ t < \tau;\\ 0,&{\rm if}\ t\geq \tau.\end{array}\right.$ 
    \STATE Set the active set $\Omega = \{(i,j): q_{i\ell}q_{j\ell}<1\}$.
\end{algorithmic}
\end{algorithm}

However, faulty assignment may occur via $k$-means clustering, especially when some $\{\tilde y_i\}$ are located between two centers. A hard assignment strategy may mislead the partition. To address its effect on the optimization of $C$, we suggest using the soft strategy of setting the active set $\Omega$ based on a probability estimation $p_{ij}$ of points $\tilde y_i$ and $\tilde y_j$ belonging to different subspaces: $\omega_{ij} = 1$ if $p_{ij}\geq \gamma$ with a constant $\gamma \in(0,1]$, or $\omega_{ij} = 0$ otherwise. 
By the law of total probability, we write $p_{ij} = 1 - \sum_\ell q_{i \ell} q_{j \ell}$, where $q_{i \ell}$ is the probability of sample $x_i$ belonging to the estimated subspace $\spann(X_{J_\ell})$. Hence, the probability of $x_i$ and $x_j$ belonging the same subspace is $\sum_\ell q_{i \ell} q_{j \ell}$.
We set $q_{i\ell}  = \frac{\psi(\tilde q_{i\ell})}{\sum_k \psi(\tilde q_{i \ell})}$ with the rescaled distance to a centroid, 
\begin{align}\label{q:centriod}
	\tilde q_{i \ell} 
    = \frac{\|\tilde y_i - b_\ell\|_2 - \min_k \|\tilde y_i - b_k\|_2}
    	   {\max_k \|\tilde y_i - b_k\|_2-\min_k \|\tilde y_i - b_k\|_2}, 
\end{align}
where $\psi$ is a nonincreasing function. For example, $\psi(t) = 1$ for $t \leq \tau$ and $\psi(t) = 0$ otherwise, where $\tau\in(0,1)$ is a given constant. 
In our experiments, we simply set $\tau=1/2$. Algorithm \ref{alg:omega} lists the detailed steps of the construction of $\Omega$.

Algorithm \ref{alg:primal_JC} summarizes the alternative rule of updating $C(S)$ and $\Omega$ for solving (\ref{prob:MSDR}). Compared with other state-of-art methods, this algorithm provides improved segmentation, especially when the minimal subspaces are significantly intersected with each other. We show the relevant comparisons in the experiment section of this paper. It is possible that the computed solution is locally optimal. In the next subsection, we further consider algorithmic improvements to avoid such localization as much as possible.

\subsection{The Pseudo-dual Problem and Solver}\label{sec:dual}

The alternative method for solving the primal problem (\ref{prob:MSDR}) provided in previous subsections may obtain only a locally optimal solution in some cases due to nonconvexity. In this subsection, we consider an algorithm for solving an equivalent pseudo-dual problem of (\ref{prob:MSDR}) for improved capability to jump out of local minima.  

\begin{algorithm}[t]
\caption{Minimal subspace segmentation via alternative optimization (MSS$_-$AO)}\label{alg:primal_JC}
Input: number of subspace $K$, $d$, initial active set $\Omega$, $\lambda_0$, $\tau$, maximal iteration number $t_{\max}$.\\
Output: $J$ and $C$.
\begin{algorithmic}[1]
    \STATE Initially set $W = [I_d,\ 0]^T$.
    \STATE For $t=1,2,\cdots, t_{\max}$
    \STATE \hspace{15pt} If $\ell = 1$, set $\lambda = \lambda_0$. Otherwise, set $\lambda$ as (\ref{def:lambda}).
    \STATE \hspace{15pt} Save $\Omega_{\rm old} = \Omega$ and compute $C$ by Algorithm \ref{alg:mcg} with the current active set $\Omega$.         
    \STATE \hspace{15pt} Update the current $\Omega$ and $J$ by Algorithm \ref{alg:omega} with $A = (|C|+|C|^T)/2$.
    \STATE \hspace{15pt} If $\Omega = \Omega_{\rm old}$, terminate the iteration.
    \STATE End
\end{algorithmic}
\end{algorithm}

\subsubsection{The Pseudo-dual Problem}

Changing the objective function $\|C_{{\rm off}(J)}\|_1$ of the primal problem (\ref{prob:MSDR}) as per restriction $C_{{\rm off}(J)} = 0$ while changing its restriction $X = XC$ as per function $\|X-XC\|_F^2$ for minimizing and keeping the same restrictions $C^T=C$ and $r(C) = d$, we can easily obtain the following pseudo-dual problem
\begin{align}\label{prob:dual}
	\min_{J \in \J(K), C} \|X-XC\|_F^2 \quad {\rm s.t.} \ C_{{\rm off}(J)} = 0,\ C = C^T,\ r(C) = d.
\end{align}
The pseudo-dual problem is equivalent to the primal problem under the conditions of Theorem \ref{thm:minimal partition}, because both problems have the same unique solution by Theorem \ref{thm:rank_connect}.

The pseudo-dual problem can be further simplified because the off-diagonal blocks of $C$ are zero. Let $C_{J_k} = C(J_k,J_k)$ as before. We see that
\[
    \|X-XC\|_F^2 = \sum_k \|X_{J_k} - X_{J_k}C_{J_k}\|_F^2, \quad {\rm and} \quad
    r(C) = \sum_k r(C_{J_k}).
\]
Hence, (\ref{prob:dual}) becomes 
\begin{align}\label{prob:dual_J}
	\min_{\{C_k\}, \{J_k\} \in \J(K)} \sum_{k=1}^K \|X_{J_k} - X_{J_k}C_{J_k}\|_F^2 \quad
	 {\rm s.t.}\ C = C^T,\ \sum_k r(C_{J_k}) = d.
\end{align}
It is convenient to optimize the block-diagonal $C$ in the above problem, since this is equivalent to solving the $K$ independent subproblems 
\begin{align}\label{prob:C_k}
    \min_{r(C_{J_k})=d_k'} \|X_{J_k} - X_{J_k}C_{J_k}\|_F^2, \quad k=1,\cdots, K,
\end{align}
on a smaller scale,
provided that $d$ can be split as $d = \sum_k d_k'$ with a good estimate $d_k'$ of the true $d_k = r(X_k)$ for each $k$. We discuss how to split $d$ and how to optimize the partition given $C$ in the next subsection.

\subsubsection{Subspace Correction}

Since $r(C_{J_k})=d_k'$, $r\big(X_{J_k}C_{J_k}\big)\leq d_k'$ and $\|X_{J_k} - X_{J_k}C_{J_k}\|_F^2\geq \min_{r(Z)\leq d_k'}\|X_{J_k}-Z\|_F^2$. It is known that the minimum is given by the truncated SVD of $X_{J_k}$ with rank $d_k'$. That is, the minimizer $Z_k = G_kD_kQ_k^T$, where $G_k$ and $Q_k$ consist of the $d_k'$ left and right singular vectors of $X_{J_k}$, respectively, corresponding to the $d_k'$ largest singular values, and $D_k$ is a diagonal matrix of the $d_k'$ largest singular values. If we choose $C_{J_k} = Q_kQ_k^T$, then $X_{J_k}C_{J_k} = G_kD_kQ_k$. That is, $Q_kQ_k^T$ solves the subproblem $\min_{r(C_{J_k})=d_k'}\|X_{J_k} - X_{J_k}C_{J_k}\|_F^2$, and
\[
    \min_{r(C_{J_k})=d_k'}\|X_{J_k} - X_{J_k}C_{J_k}\|_F^2
    = \|X_{J_k} - G_kD_kQ_k^T\|_F^2
    = \sum_j\sigma_{k,j}^2-\sum_{j\leq d_k'}\sigma_{k,j}^2,
\]
where $\sigma_{k,1}\geq \cdots \geq \sigma_{k,d_k}$ are all the singular values of $X_{J_k}$. 

The splitting $d = \sum_k d_k'$ can be easily determined. Since  
\[
    \sum_k\|X_{J_k}-X_{J_k}C_k\|_F^2 = \sum_k\sum_j\sigma_{k,j}^2 
- \sum_k\sum_{j\leq d_k'}\sigma_{k,j}^2,
\]
minimizing $\sum_k\|X_{J_k}-X_{J_k}C_k\|_F^2$ is equivalent to collecting the $d$ largest values of $\{\sigma_{k,j}\}$. Once the selection is completed, the splitting $d = \sum_k d_k'$ is immediately available by setting $d_k'$ as the number of selected $\{\sigma_{k,j}\}$ in the $d$ largest values. 

We now consider how to update partition $J=\{J_k\}$ given $\{C_k\}$. 
Because we have obtained the spanning subspaces $\{\spann(G_k)\}$, partition $J$ can be updated by the new partition $\tilde J=\{\tilde J_1,\cdots,\tilde J_K\}$ according to the rule of the nearest subspace for each sample, that is, 
\begin{align}\label{def:labeling}
	\tilde J_k = \big\{j:\ k = \arg\min_\ell \|x_j-G_\ell G_\ell ^Tx_j\|_2^2\big\}.
\end{align}

Our subspace correction method for solving the pseudo-dual problem (\ref{prob:dual}) is summarized in Algorithm \ref{alg:dual_C}. We note that the above method is a bit similar to the $K$-Subspace algorithm proposed by \cite{Bradley2000k} in which the dimension $d_k$ of each true subspace $\S_k$ is known, and each $J_k$ is assumed to match dimension $d_k$ correctly. These two assumptions cannot be satisfied in the complicated case that we consider because the spanning subspaces of the minimal segments are unknown.

\begin{algorithm}[t]
	\caption{Subspace correction for solving the pseudo-dual problem (\ref{prob:dual})}\label{alg:dual_C}
	Input: $X$, $d$, initial $K$-partition $J$, and max iteration number $s_{\max}$ \\
	Output: $\{J_k\}$ and $\{C_k\}$.
	\begin{algorithmic}[1]
		\STATE For $s=1,2,\cdots, s_{\max}$
		\STATE \hspace{15pt} Save $J_{\rm old}=J$ and compute the $d$ largest singular triples $\{g_{ki},\sigma_{ki},q_{ki}\}$ of each $X_{J_k}$.
		\STATE \hspace{15pt} Pick up $d$ largest values from $\{\sigma_{ki}\}$, containing $d_k'$ selected ones for each $k$.
		\STATE \hspace{15pt} Set $G_k=[g_{k1},\cdots,g_{kd_k'}]$ and update $J_{\rm old}$ to $J=\{J_k\}$ according to (\ref{def:labeling}).
		\STATE \hspace{15pt} If $J=J_{\rm old}$, set $C_k = Q_kQ_k^T$ with $Q_k=[q_{k1},\cdots,q_{kd_k'}]$ for each $k$, and terminate.
		\STATE End
	\end{algorithmic}
\end{algorithm}

\subsubsection{Convergence}

Algorithm \ref{alg:dual_C} decreases the objective function of (\ref{prob:dual_J}). On the one hand, given $J=\{J_k\}$, the optimal blocks $\{C_k\}$ are provided by $C_k = Q_kQ_k^T$ as shown above. Hence,
\begin{align*}
    \sum_k \|X_{J_k} \!-\! X_{J_k} C_{J_k} \|_F^2 
	& = \sum_k \|X_{J_k} \!-\! G_kG_k^TX_{J_k}\|_F^2 \\
	& = \sum_k \sum_{j \in J_k} \|x_j \!-\! G_kG_k^Tx_j\|_2^2
	\geq \sum_j \min_{\ell} \|x_j \!-\! G_\ell G_\ell^Tx_j\|_2^2.
\end{align*}
On the other hand, as $X_{J_k}C_{J_k} = G_kD_kQ_k$, we also have $X_{\tilde J_k}\tilde C_{\tilde J_k} = \tilde G_k\tilde D_k\tilde Q_k^T$ for the updated pairs $\{\tilde C_{\tilde J_k}, \tilde J_k\}$ of $\{C_{J_k},J_k\}$ since $\tilde G_k\tilde D_k\tilde Q_k^T$ is a truncated SVD of $X_{\tilde J_k}$. 
Hence,
\begin{align*}
	&\ \sum_j \min_{\ell} \|x_j \!-\! G_\ell G_\ell^Tx_j\|_2^2 
	= \sum_{k}\sum_{j \in \tilde J_k} \|x_j \!-\! G_kG_k^T x_j\|_2^2 
   	= \sum_k \|X_{\tilde J_k} \!-\! G_kG_k^T X_{\tilde J_k}\|_F^2 \\
    \geq&\ \sum_k \|X_{\tilde J_k} \!-\! \tilde G_k \tilde D_k\tilde Q_k^T\|_F^2
    = \sum_k \|X_{\tilde J_k} \!-\! X_{\tilde J_k} \tilde C_k\|_F^2.
\end{align*}
Therefore, $\sum_k \|X_{J_k} \!-\! X_{J_k} C_{J_k} \|_F^2\geq \sum_k \|X_{\tilde J_k} \!-\! X_{\tilde J_k} \tilde C_k\|_F^2$. The alternative iteration converges in the sense of decreasing the value of the objective function. Because only a finite number of partitions exist, the alternative iteration can be terminated within a finite number of steps as the function value is unchanged, though $J$ may have a differently modified $\tilde J$.

\begin{theorem}
The algorithm of subspace correction yields a decreasing sequence of objective values and terminates within a finite number of iterations. 
\end{theorem}

It should be pointed out that multiple partitions achieving the same objective values may exist if some samples have the same minimal distances to different estimated sample subspaces. If this happens at a terminated partition $J$ and its modified $\tilde J$, 
the equalities of objective values imply the equalities
\[
    \sum_k \sum_{j \in J_k} \|x_j \!-\! G_kG_k^Tx_j\|_2^2
	= \sum_j \min_{\ell} \|x_j \!-\! G_\ell G_\ell^Tx_j\|_2^2
	= \sum_k \sum_{j \in \tilde J_k} \|x_j \!-\! \tilde G_k\tilde G_k^Tx_j\|_2^2.
\]
One may understand the difference between $J$ and $\tilde J$ by the arbitrary labeling of such samples because of their equal distances. To clearly show this, let 
\[
    M_j = \big\{k: \|x_j \!-\! G_kG_k^Tx_j\|_2 = \min_{\ell} \|x_j \!-\! G_\ell G_\ell^Tx_j\|_2\big\}
\]
for each $j$. Then, each $J_k$ can be split as $J_k = J_k^0\cup J_k'$, where $J_k^0$ consists of the $j$'s with a singleton $M_j = \{k\}$, and $J_k'$ is a set of partial $j$'s whose $M_j$ has at least two indices, one of which is $k$. Similarly, $\tilde J_k = J_k^0\cup\tilde J_k'$. Obviously, $J_k\neq \tilde J_k$ is equivalent to $J_k'\neq \tilde J_k'$. Randomly labeling these $j$'s according to the multiple $k$'s in $M_j$ results in multiple partitions. That is, there are multiple options for setting $\tilde J$ in this case. It is unclear whether there is a $\tilde J$ among the multiple choices that achieves a smaller value of the objective function. Choosing such a $\tilde J$ may obtain better convergence but requires a complicated labeling rule, rather than the simple one (\ref{def:labeling}). We do not intend to further exploit the multiplicity of partitions because of the nonsingletons $\{M_j\}$.

\subsection{Hybrid Optimization}

Both Algorithms \ref{alg:primal_JC} and \ref{alg:dual_C} may fall into local minimizers, but exhibit their own convergence behaviors. Algorithm \ref{alg:primal_JC} is relatively stable on the initial setting of partition $J$ or active set $\Omega$, in the sense that the convergent solution $C$ or $J$ always has good accuracy with respect to the minimal partition, although the solution may not be completely correct. Algorithm \ref{alg:dual_C} heavily depends on the initial guess of $J$ and may give a completely incorrect solution if the initial partition is poor. A good initial $J$ for Algorithm \ref{alg:dual_C} should ensure that each $J_k$ dominates samples from the same minimal segment. In this case, the algorithm \ref{alg:dual_C} converges to the true minimal partition quickly. 

In this subsection, we consider a hybrid strategy for minimal subspace learning that combines primal and pseudo-dual optimization, which we term hybrid optimization. Essentially, starting with the coarsest active set $\Omega_c$ covering all index pairs $(i,j)$ except the diagonal indices $\{(i,i)\}$, the hybrid strategy first solves the primal problem (\ref{prob:MSDR}) with an active set $\Omega$ and then solves the pseudo-dual problem (\ref{prob:dual}) using the primal solution as its initial guess. This procedure is repeated if necessary. 

The key issue for hybrid optimization is constructing an initial guess for the primal (pseudo-dual) algorithm from the solution of the pseudo-dual (primal) algorithm. It is easy to construct an initial partition for the pseudo-dual algorithm (Algorithm \ref{alg:dual_C}) based on subspace correction using the solution given by the primal algorithm (Algorithms \ref{alg:primal_JC}). Here, we focus on constructing a suitable $\Omega$ for Algorithms \ref{alg:primal_JC}, based on partition $J=\{J_\ell\}$ given by Algorithm \ref{alg:dual_C}. Here, $\Omega$ means the matrix with entries $\omega_{ij}$. We may slightly change the $0-1$ setting of the entries to that with one of the three values $0, 1, \beta$ because of the property of subspace correction. 

If all the $K$ subsets are not empty, it is highly possible that each $J_\ell$ is dominated by a single true segment. Thus, we modify $\Omega$ as follows:
\begin{align}\label{Omega1}
	\omega_{ij} = \left\{\begin{array}{ll}
		1	&  {\rm if}\ i\in J_s, j\in J_t,\ s\neq t;\\
		0  	&  {\rm if}\ i, j\in J_k.
	\end{array}\right.
\end{align}
However, if there are some empty $J_k$, without loss of generality, let $J_1,\cdots,J_\ell\ $ be all the nonempty subsets with $\ell<K$. Since some true minimal segments are approximately merged together into a nonempty $J_k$ because of the subspace correction, the entries in the off-diagonal block $C(J_s,J_t)$ decrease faster than those $C(J_{k'}^*,J_{k''}^*)$ with $k'\neq k''$ if $J_{k'}^*$ and $J_{k''}^*$ are merged together. We slightly modify the coarse $\Omega_c$ to $\Omega_\beta = (\omega_{ij})$ with the following
\begin{align}\label{Omega2}
	\omega_{ij} = \left\{\begin{array}{ll}
		\beta	&  {\rm if}\ i\in J_s, j\in J_t,\ s\neq t;\\
		1  	&  {\rm if}\ i, j\in J_k,\  i\neq j;\\
		0	&  {\rm if}\ i=j .
	\end{array}\right.
\end{align}

\begin{algorithm}[t]
	\caption{Minimal subspace segmentation via hybrid optimization (MSS$_-$HO)} \label{alg:combine}
	Input: $X$, $K$, parameter $\beta$, maximal iteration number $h_{\max}$ of HO\\
	Output: $J$ and $C$
	\begin{algorithmic}[1]
		\STATE Initially set $\Omega=\Omega_c$ with $\omega_{ij} = 1$ if $i\neq j$ and $\omega_{ii}=0$.
		\STATE For $h=1,2,\cdots, h_{\max}$
		\STATE \hspace{15pt} Save $\Omega_{\rm old} = \Omega$, and 
		solve the primal problem (\ref{prob:MSDR}) to get $J_{\rm prim}$ and $C$ by Algorithm \ref{alg:primal_JC} with $\Omega$.
		\STATE \hspace{15pt} Solve the pseudo-dual problem (\ref{prob:dual}) to get $J$ 
			by Algorithm \ref{alg:dual_C} using $J_{\rm prim}$ initially.
		\STATE \hspace{15pt} Update $\Omega$ as (\ref{Omega2}) if there is $J_k =\emptyset$, or as (\ref{Omega1}) otherwise.
			If $\Omega = \Omega_{\rm old}$, terminate. 
		\STATE End
	\end{algorithmic}
\end{algorithm}

The constant $\beta$ plays a special role in controlling the convergent behavior of $C=C(S)$ in the iteration of Algorithm \ref{alg:primal_JC} using $\Omega_\beta$. Compared with the iteration of $C$ in Algorithm \ref{alg:primal_JC} before shifting to Algorithm \ref{alg:dual_C} for updating the partition, a larger $\beta>1$ can accelerate the decreasing of the blocks $C(J_s,J_t)$ with $s\neq t$ because of the larger weights in the function
\[
	\|\Omega_\beta\odot C\|_1 = \sum_k\sum_{i\neq j\in J_k}|c_{ij}|+\beta\sum_{s\neq t}\|C(J_s,J_t)\|_1.
\] 

Once these $C(J_s,J_t)$ are small, block $C(J_{k'}^*,J_{k''}^*)$ begins decreasing, a bit similar to the result of applying Algorithm \ref{alg:primal_JC} on the smaller block $C(J_k,J_k)$. Thus, a larger $\beta$ helps to turn off the decreasing early. Note that a smaller $\beta<1$ can balance $\{C(J_s,J_t)\}$ and $\{C(J_{k'}^*,J_{k''}^*)\}$ since it can delay the decreasing of the $C(J_s,J_t)$, or equivalently, relatively accelerate the decreasing of $C(J_{k'}^*,J_{k''}^*)$. When such balance occurs, the merging of multiple minimal segments might also be delayed in Algorithm \ref{alg:dual_C} using such a locally optimal solution of Algorithm \ref{alg:primal_JC}. However, the value of $\beta$ must be carefully chosen to balance the decreasing of these blocks. For the sake of simplicity, we just suggest using a $\beta>1$. In our experiments, we always set $\beta=1.25$. See Algorithm \ref{alg:combine} for the procedure of our hybrid optimization procedure.

Figure \ref{fig:alg_combine} plots four indication matrices $E_J$ of four partitions obtained by the hybrid algorithm, where $(E_J)_{ij}=1$ if $i,j\in J_k$, or $(E_J)_{ij}=0$ if $i,j$ belong to different $J_k$'s. The data set has minimal segmentation consisting of five minimal segments of equal size. Starting with the coarsest $\Omega_c$, Algorithm \ref{alg:primal_JC} obtains a solution $J$ whose indication matrix is plotted on the left. $J$ contains five nonempty pieces $\{J_k\}$; three of them have relatively dominant indices from a single minimal segment and other two are mixed by multiple minimal segments. Due to this mixture, the pseudo-dual step gives a partition with four nonempty pieces, in which two of them are very small, one almost contains the indices of a minimal segment, and the largest one is dominated by other minimal segments. The initial setting (\ref{Omega2}) can significantly reduce the mixture at the primal step; see the indication matrix plotted third in Figure \ref{fig:alg_combine}. Due to the improvement, the second pseudo-dual step correctly recovers all the minimal segments.

\begin{figure}[t]
		\centering
		\includegraphics[width=1.5in]{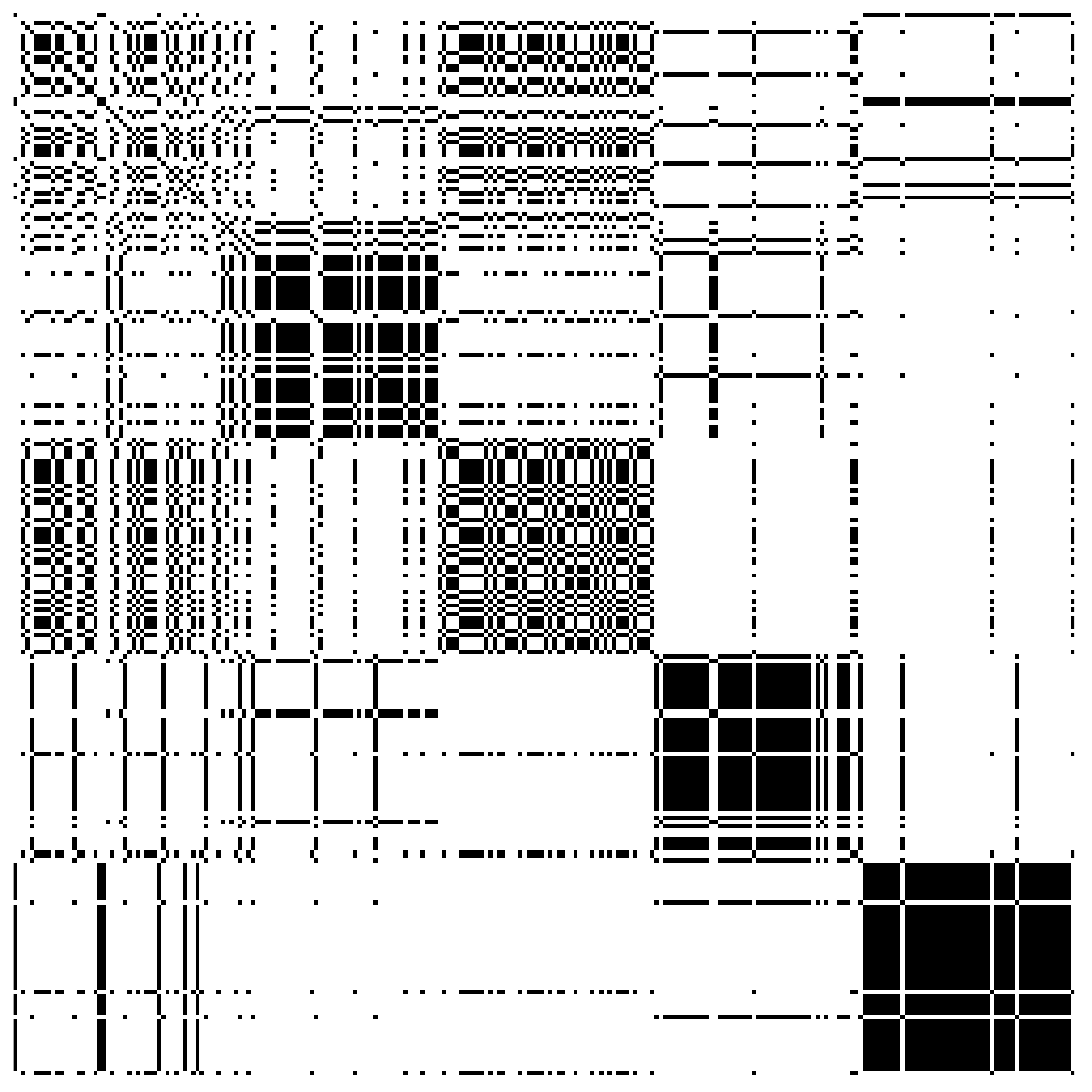}
		\includegraphics[width=1.5in]{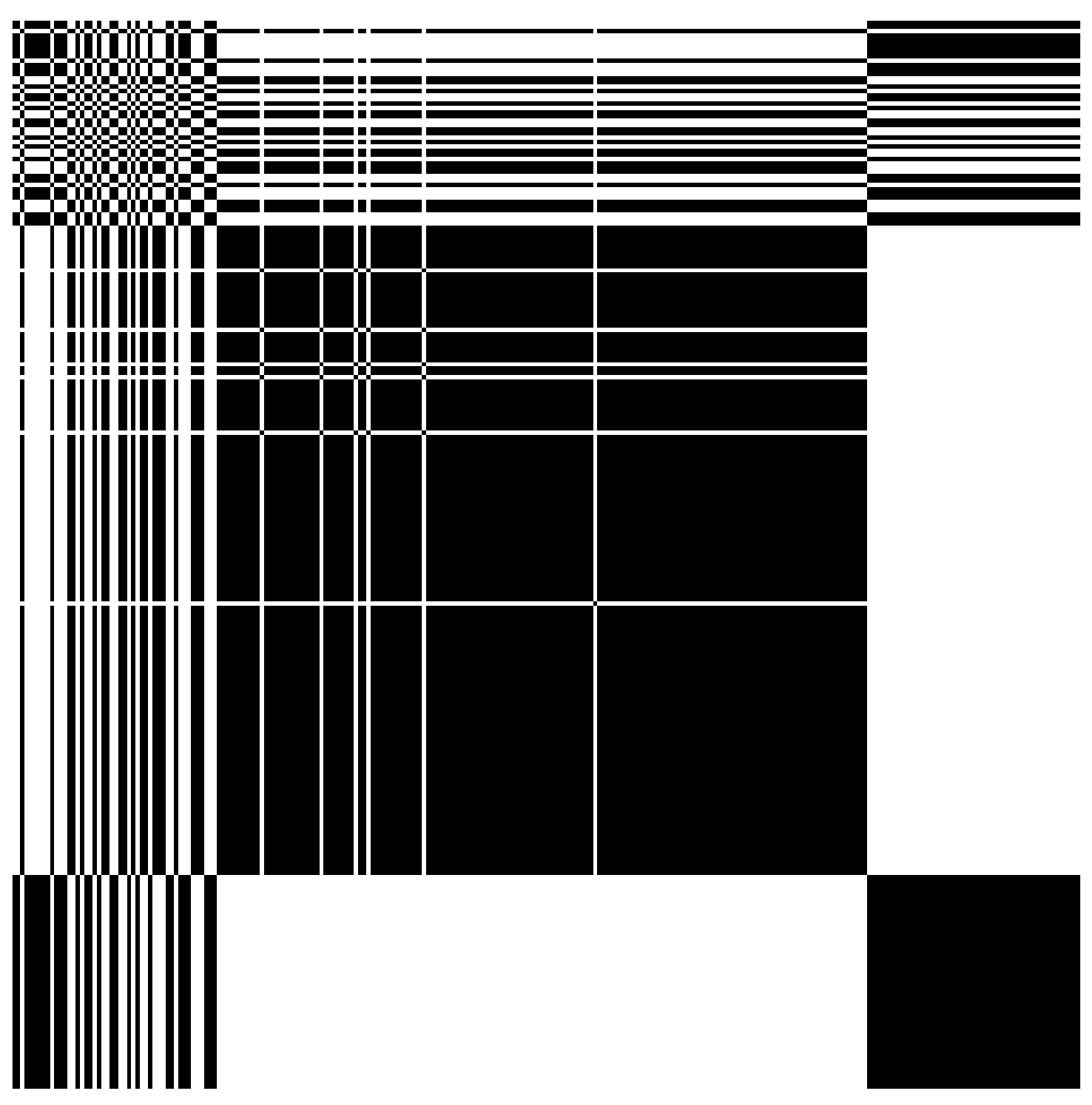}
		\includegraphics[width=1.5in]{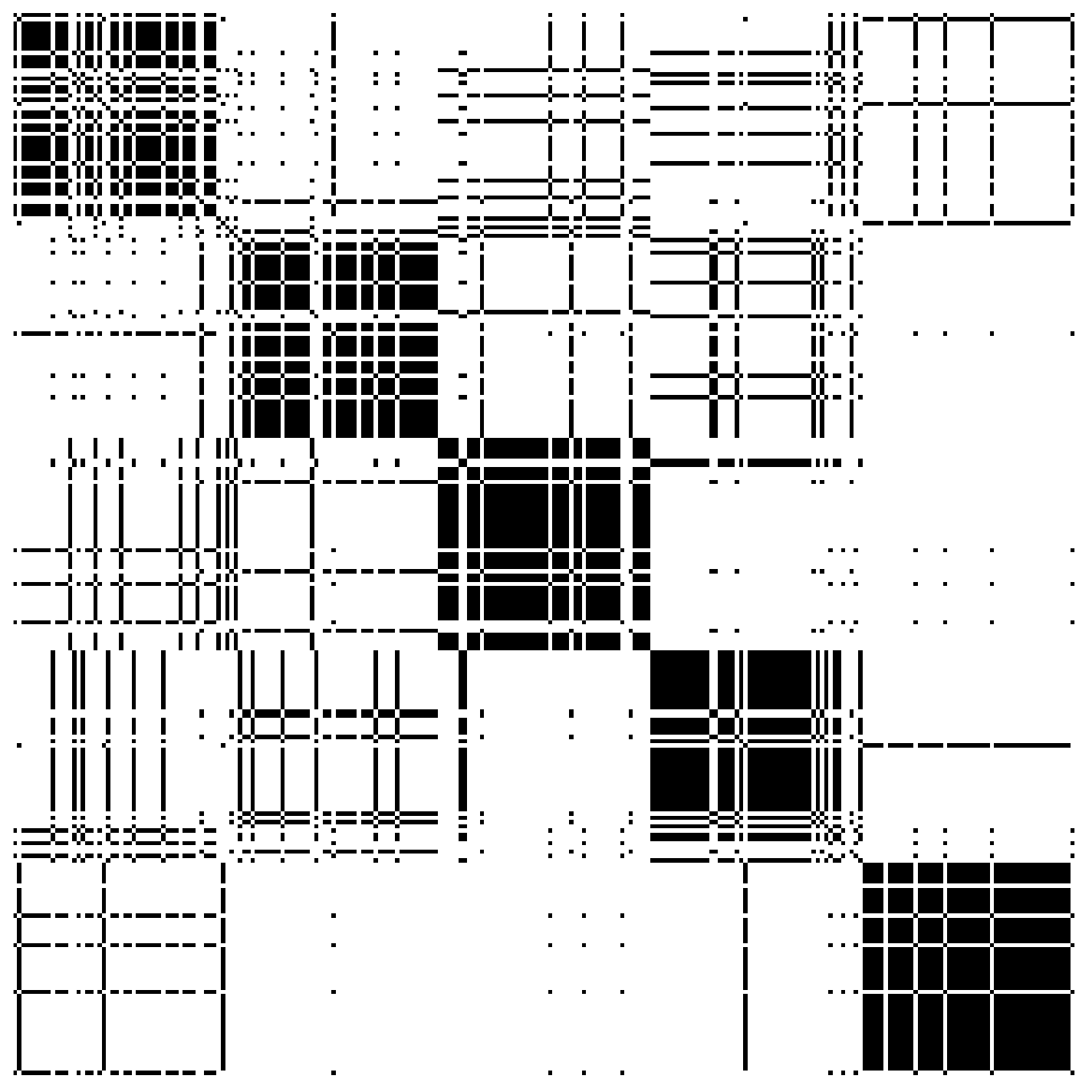}
		\includegraphics[width=1.5in]{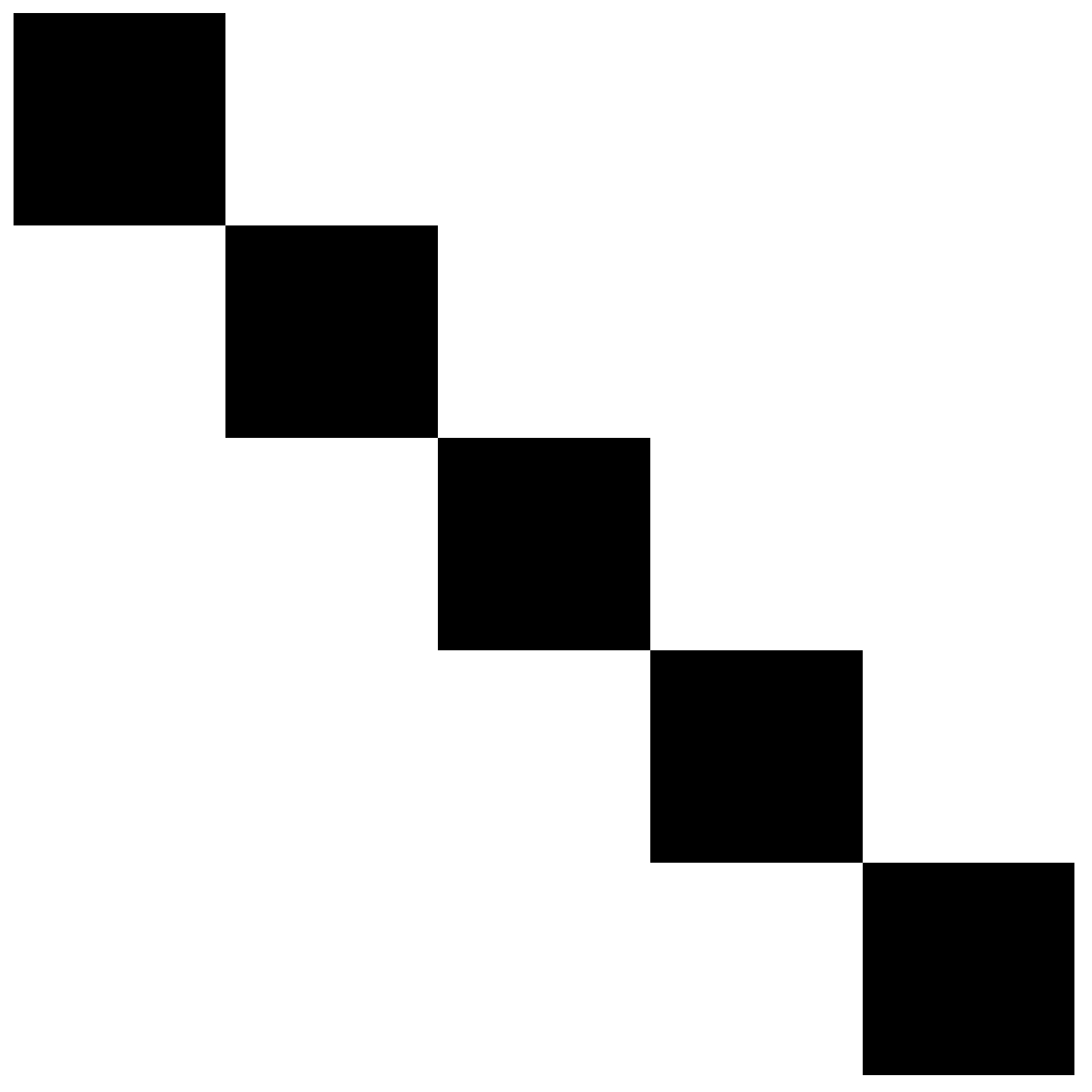}
	\caption{Indication matrix $E_{J^{(t)}}$ of the computed partition of Algorithm \ref{alg:primal_JC} using $\Omega_c$ (left) or $\Omega_\beta$ (middle right) corresponding to the solution (middle left) of Algorithm \ref{alg:dual_C} starting with the left $J$. The right $J$ is a solution of Algorithm \ref{alg:dual_C} initially using the middle right $J$.}
	\label{fig:alg_combine}
\end{figure}

\section{Generalization for Noisy Samples}\label{sec:extension}

Given a noisy sample set $X = X_*+E$, where $X_*$ is a set of unknown clean samples and $E$ is a noise set, we seek an MSDR with respect to the subspaces $\{\S_k\}$ spanned by the minimal segments of $X_*$ theoretically. Our algorithms discussed in previous sections are based on the orthonormal matrix $V$ of the right singular vectors of $X_*$ corresponding to nonzero singular values and its orthogonal complement $V_\bot$. If we can estimate $V$ from noisy samples with high accuracy, these algorithms can also be used for subspace segmentation from noisy samples. Numerically, the required orthonormal matrix $V$ can be estimated by the dominant right singular vectors of $X$. The classical perturbation theory of singular values shows that the estimation of $V$, or more preciously, the estimation of the subspace spanned by the dominant singular vectors, is robust if the noise is relatively smaller than the smallest singular value. This means that if the smallest singular value of $X_*$ is large, the data noise could be somewhat relatively large. Further, we could apply the MSS algorithms on the estimated $V$ and its orthogonal complement $V_{\bot}$ for segmenting the noisy samples. Perturbation theory on solutions is an interesting topic. We leave it as further work. 

In this section, we consider a robust approach for the complicated cases of considerable noise or an incorrectly estimated dimension sum $d$. A generalized sparse model is given to handle noisy samples. We solve this sparse problem by applying an ADMM method.

\subsection{The Relaxed Optimization}

Let $C_*$ be an MSDR of the latent samples $X_*$, that is, $X_* = X_*C_*$ and $C_*$ is a block-diagonal matrix with rank $d$. By Theorem \ref{thm:rank_connect} on the one hand, $C_* = G_*G_*^T$ with an orthonormal $G_*$ of $d$ columns. On the other hand, since $X_* = X-E$, the required SDR also satisfies $X = XC_*+E(I-C_*)$. That is, the self-expressive error $X - XC_*$ should be as small as the noise in magnitude. Therefore, it makes sense for approximating the MSDR to minimize the sparsity $\|C_{\rm off}\|_1$ subjected to $C = GG^T$ with an orthonormal $G = V_*+(V_*)_\bot Q$ of $d$ columns, if the correct sum $d$ of subspace dimensions is known, we should also simultaneously minimize the self-expressive error $X-XC$ in a suitable measurement $\phi(X-XC)$. That is, we may consider the optimization problem
\[
	\min_{C =GG^T, G^TG=I_d}\Big\{\|C_{{\rm off}}\|_1 + \alpha \phi(X-XC)\Big\}.
\]

Some relaxations on the above problem are required for efficient computation. Here, we adopt four kinds of relaxations:

1) relax the first term to $\|\Omega\odot C\|_1$ with an $\Omega$ that can be modified iteratively, 
 
2) add a penalty term on the diagonal vector $c$ of $C$ to decrease the nonconvexity,

3) relax the special restriction on $G=V_*+(V_*)_\bot Q$ to be a general $G$ of $d$ columns, and

4) relax the strict expressive matrix $C=GG^T$ to $C\approx GG^T$. 

The first two relaxations are similar to those discussed in previous sections without noise, as 
in the basic model (\ref{prob:basic}). The third one can avoid estimating the dimension $r$ of data space, its orthogonal basis $V_*$, and orthogonal complement of $V_*$. The last one considers the case where $d$ may be approximately estimated. 
The approximation $C\approx GG^T$ can be implicitly obtained by minimizing the error function $\|C - GG^T\|_F^2$. Combining these relaxations together results in the following sparse problem,
\begin{align} \label{opt:noise_rcsc}
	\min_{C, \Omega, G} \Big\{{\cal F}(C, \Omega, G) = \|\Omega\odot C\|_1 + \frac{\lambda}{2}\|c\|_2^2 + \alpha \phi(X-XC) + \frac{\beta}{2} \|C - GG^T\|_F^2\Big\}.
\end{align}
We may choose $\phi$ as the norm $\|\cdot\|_1$, $\|\cdot\|_{2,1}$, or $\|\cdot\|_F^2$, depending on the noise distribution. 
It is reasonable to set the parameters $\lambda$, $\alpha$, and $\beta$ proportional to the number of samples in each subspace. If the number of samples in each subspace are approximately equal, we can set the parameters proportional to $n/K$. We will show that this strategy works well numerically.

As in the noiseless case, we also solve (\ref{opt:noise_rcsc}) via alternatively optimizing $(C,G)$ and $\Omega$: We solve 
\begin{align}\label{prob:F_CG}
	\min_{C,G}{\cal F}(C, \Omega, G)
\end{align}
with a fixed $\Omega$, and then update $\Omega$ when $(C,G)$ is updated. The problem (\ref{prob:F_CG}) is also addressed via alternatively optimizing $C$ and $G$ because we have the following two sound propositions:

(1) The subproblem $\min_{G}{\cal F}(C, \Omega, G)$ is equivalent to $\min_{G \in \R^{n \times d}} \|C - GG^T\|_F^2$ whose solution is as follows:
\begin{align}\label{G}
	GG^T = P\diag((\lambda_1)_+,\cdots,(\lambda_d)_+)P^T, 
\end{align}
where $\{\lambda_i\}$ are the $d$ largest eigenvalues of $(C+C^T)/2$ and $P$ consists of the corresponding eigenvectors. 

(2) The subproblem $\min_{C}{\cal F}(C, \Omega, G)$ is convex and has a unique solution. We can solve this problem using the ADMM method given in the next subsection. 

Therefore, the algorithm for solving (\ref{prob:F_CG}) consists of an inner-outer iteration scheme. In addition, because the objective function is monotonously decreasing, the algorithm is convergent. Numerically, it is unnecessary to solve the inner problem for updating $C$ with high accuracy. An inaccurate solution by ADMM is sufficient if it can decrease the value of ${\cal F}$, which results in an inaccurate inner-outer iteration method for solving (\ref{prob:F_CG}). 

As soon as $(C,S)$ are updated, we modify $\Omega$ using the rule given in Subsection \ref{sec:update Omega}. That is, we construct a new $\Omega$ by Algorithm \ref{alg:omega} with input graph $A$. The simple input $A=(|C|+|C|^T)/2$ is no longer suitable for noisy data since it is not positive semidefinite. We give an approach for constructing a symmetric graph $A$ as input to Algorithm \ref{alg:omega}. The construction combines the advantages of that based on SSC solutions or LRR solutions as shown in Algorithm \ref{alg:G}. 

\begin{algorithm}[t]
\caption{Graph construction from $C$}\label{alg:G}
Input: $C$, $\gamma$, $\sigma$, and $s\geq 1$.\\
Output: graph matrix $A$.
\begin{algorithmic}[1]
    \STATE Cut off small entries of $C$ to get a sparse $C_\gamma$ with column norms $\|C_\gamma e_j\|\approx\gamma\|C e_j\|$.
    \STATE Compute the left singular vectors $\{u_i\}$ of $C_\gamma$ corresponding to singular values $\sigma_i\!\geq\!\sigma$.
    \STATE Construct $A=\big(\langle g_i,g_j\rangle^s\big)$ with the nomalized rows $\{g_j\}$ of $[\cdots,\sqrt{\sigma_i}u_i,\cdots]$.
\end{algorithmic}
\end{algorithm}

\subsection{ADMM Approach}

We rewrite the subproblem of (\ref{opt:noise_rcsc}) for optimizing $C$ given $(\Omega,G)$ as the ADMM form
\[
	\min_{C,E,Z}\Big\{ \|\Omega\odot C\|_1 + \frac{\lambda}{2}\|c\|_2^2 + \alpha \phi(E) + \frac{\beta}{2}\|Z - GG^T\|_F^2\Big\}
	 \quad {\rm s.t. } \ E = X-XZ,\ Z = C.
\]
Its augmented  Lagrangian function is
\begin{align*}
	L(C, E, Z) & = \|\Omega\odot C\|_1 + \frac{\lambda}{2}\|c\|_2^2 + \alpha \phi(E) + \frac{\beta}{2}\|Z - GG^T\|_F^2 
		+ \frac{\rho'}{2}\|C-Z\|_F^2 \\
		& + \frac{\rho''}{2}\|X-XZ-E\|_F^2 
    	         + \langle C-Z, Y' \rangle + \langle XZ+E-X, Y'' \rangle.
\end{align*}
Hence, the ADMM scheme is given by the following alternative rule.
\begin{align}
	(\hat C, \hat E ) & =  \arg\min_{C, E} L(C, E, Z); \label{noiseADMM:ZE} \\
   	\hat Z & = \arg\min_Z L(\hat C, \hat E,Z); \label{noiseADMM:C}\\
    \hat Y' & = Y' + \rho'(\hat C-\hat Z); \label{noiseADMM:Y'}\\
    \hat Y'' & = Y'' + \rho''(X\hat Z+\hat E-X) .\label{noiseADMM:Y''}
\end{align}
Notice that the objective function of (\ref{noiseADMM:ZE}) is separable as $L(C, E, \cdot) = L_1(C)+L_2(E) + {\rm constant}$, where
\begin{align*}
	L_1(C) &\ = \|\Omega\odot C\|_1 + \frac{\lambda}{2}\|c\|_2^2 + \frac{\rho'}{2}\|C-Z + \frac{1}{\rho'}Y' \|_F^2, \\
    L_2(E) &\ = \alpha \phi(E) + \frac{\rho''}{2}\|XZ+E-X+\frac{1}{\rho''} Y''\|_F^2,
\end{align*}
and the constant means a term not depending on the variables $C$ or $E$. 
Hence, $\hat C = \arg\min_C L_1(C)$ and $\hat E = \arg\min_E L_2(E)$.
Below, we provide solutions of the above three subproblems in closed form.
The convergence of this ADMM iteration is guaranteed by \cite{Boyd2011Distributed}.

\subsubsection{Updating the Representation Matrix}

The step of updating $C$ is separable with respect to its entries. That is, updating each entry of $C$ is an independent procedure. The entry $c_{ij}$ is updated by $\hat c_{ij}$, the solution  
\[
	t^* = \arg\min_t \Big\{\omega|t| + \frac{a}{2}t^2 + \frac{\rho'}{2}(t-p)^2 \Big\}
\]
with the parameters $\omega=\omega_{ij}$, $a =\lambda$ if $i=j$ or $a = 0$ otherwise, and $p = z_{ij}-y'_{ij}/\rho'$. Since the objective function can be rewritten as $\frac{\rho'+a}{2}\big(t - \frac{\rho' p - \sgn(t)\omega}{\rho' + a}\big)^2$, letting $t_1 = \frac{\rho' p-\omega}{\rho'+a}$ and $t_2 = \frac{\rho' p +\omega}{\rho'+a}$, the solution is as follows:
\begin{align*}
	t^*
	= \left\{ \begin{array}{ll}
   		t_2,& \mbox{if}\ t_2\leq 0\\
		0,  & \mbox{if}\ t_1<0 < t_2\\
   		t_1,& \mbox{if}\ t_1\geq 0
   	\end{array} \right\} 
	= \frac{{\rm shrink}(\rho'p, \omega)}{ \rho'+a},
\end{align*}
where ${\rm shrink}(\beta, \alpha) = \sgn(\beta)(|\beta| - \alpha)_+$ is a shrinkage operator of $\beta$ corresponding to $\alpha$.  
Hence, the optimal solution $\hat C$ is given by the following
\begin{align}\label{update:C}
	\hat C = R \odot \mbox{shrink} \left(\rho' Z -Y', \Omega \right),
\end{align}
where $R $ has the diagonals $1/(\rho' + \lambda)$ and the off-diagonals $1/\rho'$, and ${\rm shrink(B, A)}$ is the elementwise operator of ${\rm shrink}(\beta, \alpha)$.

\subsubsection{Updating the Error Matrix}
The solution of $\min L_2(E)$ depends on the function $\phi$. If $\phi(E)$ is one of the three functions $\|E\|_1$, $\|E\|_{2,1}$, or $\|E\|_F^2$, the solution is closed-form with $\Delta = X-XZ-Y''/\rho''$,
\begin{align}\label{update:E}
	\hat E = \left\{\begin{array}{ll}
		{\rm shrink} \left( \Delta, \alpha/\rho'' \right), & \mbox{if }\ \phi(E) =\|E\|_1;\\
		\Delta\diag(\beta_1,\cdots,\beta_n), & \mbox{if }\ \phi(E) =\|E\|_{2,1};\\
		\Delta/( 1+2 \alpha/\rho'' ), , & \mbox{if }\ \phi(E) =\|E\|_F^2,
	\end{array}\right.
\end{align}
Here, the first form is similar to that given by \cite{Beck2009A}. In the second form, $\beta_i =(\|\delta_i\|_2 - \alpha/\rho'' )_+/ \|\delta_i\|_2$ with the columns $\delta_i$ of $\Delta$, as shown by \cite{Yang2011A}.

\subsubsection{Updating the Relaxation Variable}

Fixing $\hat C$ and $\hat E$, $L(\hat C, \hat E, Z)$ is a quadratic function of $Z$.
Hence, its minimizer is unique and is given by the solution of the equation $\frac{\partial}{\partial Z} L(\hat C, \hat E, Z) = 0$. That is, 
\begin{align*}
	\beta(Z-GG^T) + \rho' (Z-\hat C) + \rho'' X^T(XZ+\hat E-X) -Y'+X^TY'' = 0.
\end{align*}
Thus, the minimizer of (\ref{noiseADMM:C}) is 
\begin{align} \label{update:Z}
	\hat Z = \big( (\beta + \rho') I+\rho'' X^TX  \big)^{-1}\big(\beta GG^T + \rho' \hat C+\rho''X^T(X-\hat E)+Y'-X^TY''\big).
\end{align}

\begin{algorithm}[t]
\caption{Minimal subspace segmentation via relaxed optimization (MSS$_-$RO)}\label{alg:noise}
Input: $X$, $K$, $d$, accuracy parameters $\tau$ and $\epsilon$, and max iteration numbers $t_{\max}, \ell_{\max}, k_{\max}$\\
Output: $J$ and $C$.
\begin{algorithmic}[1]
\STATE Initially set $\Omega=\Omega_c$, $Z = C = I$, $Y' = 0$, and $Y'' = 0$, and save $\Omega_{\rm old}=\Omega$. 
\STATE Repeat the following produce for at most $t_{\max}$ times.
\STATE \hspace{15pt} Solve (\ref{opt:noise_rcsc}) with fixed $\Omega$ via the following inner iterations:
\STATE \hspace{15pt} For $\ell=1,\cdots,\ell_{\max}$
\STATE \hspace{30pt} Save $C_{\rm old} = C$ and run the ADMM iteration for updating $C$:
\STATE \hspace{30pt} For $k = 1, 2, \cdots, k_{\max}$
\STATE \hspace{45pt} Save $Z_{\rm old} = Z$ and update $C$, $E$, and $Z$ as (\ref{update:C}), (\ref{update:E}), and (\ref{update:Z}), respectively.
\STATE \hspace{45pt} Modify the multipliers $Y'$ and $Y''$ as (\ref{noiseADMM:Y'}) and (\ref{noiseADMM:Y''}). 
\STATE \hspace{45pt} If $\|Z-Z_{\rm old}\|_F<\tau$, set $C=Z$, and terminate the iteration.
\STATE \hspace{30pt} End
\STATE \hspace{30pt} If $\|C - C_{\rm old}\|_F < \epsilon$, terminate. Otherwise, update $G$ as (\ref{G}).
\STATE \hspace{15pt} End
\STATE \hspace{15pt} Construct the graph $A$ by Algorithm \ref{alg:G} and update $\Omega$ and $J$ by Algorithm \ref{alg:omega} .
\STATE \hspace{15pt} If $\Omega = \Omega_{\rm old} $, terminate the repeat.  
\STATE End
\end{algorithmic}
\end{algorithm}

The whole iterative procedure for solving (\ref{opt:noise_rcsc}) is summarized in Algorithm \ref{alg:noise}, which combines the closed-forms of solutions of the subproblems (\ref{noiseADMM:ZE}) and (\ref{noiseADMM:C}), the ADMM iteration, and the updating of $\Omega$ if necessary. The number of repeats $t_{\max}$ can be small. In some examples, setting $t_{\max}=1$ also gives a good solution.

\section{Experiments} \label{sec:experiment}

We evaluate the performance of our algorithms on synthetic data sets without noise and two kinds of real-world data sets for face recognition and motion detection. The synthetic data vectors are sampled from the union of several known subspaces $\{\S_k\}$ such that the segment $X_k$ of samples from $\S_k$ is minimal and $\spann(X_k) = \S_k$ for each $k$. The minimal subspaces are intersected and some samples are close to but do not belong to the intersections of these subspaces. Two key issues known to affect subspace segmentation are subspace intersections and samples located near intersected subspaces. The synthetic data are tested to show how the proposed algorithms perform when 
the subspaces are heavily intersected and some samples are close to the intersected subspaces.
Two real-world data sets are used to evaluate how the proposed algorithms perform on noisy data. Our algorithms are also compared with five algorithms for subspace clustering:
LRR \cite{LRR2013}, 
CLAR \cite{Kang2015Robust}, 
SSC \cite{Elhamifar2013Sparse}, 
LRSSC \cite{wang2013provable}, 
and SoftS3C \cite{Li2017Structured}.\footnote{We omit a comparison with the hard version proposed in the same paper \cite{Li2017Structured} since the soft version SoftS3C performed slightly better than the hard version in our experiments.}  
The reported results of these algorithms were obtained using the codes provided by algorithm owners or downloaded from open sources. The parameters are set as suggested by the algorithm owners or carefully chosen by us. The original LRSSC was slightly modified in the experiments on real-world data sets to achieve better results.

\subsection{Evaluation Criteria}

We use the following four measurements to evaluate the quality of the computed solutions from the partition error, the deviation from block-diagonal form, and the connection of diagonal blocks, according to the ideal minimal segments $J^*$.
 
(1) The error of partition $J$. This error is defined by the percentage of misassigned samples in partition $J$ compared with the true minimal partition $J^*$,
\begin{align}\label{ErrParti}
	{\rm ErrParti}(J) =  \min_{\pi} \frac{1}{n}\Big\{n -\sum_{k=1}^K |J_{\pi(k)}\cap J_k^*|\Big\},
\end{align}
where $\{\pi(1),\cdots,\pi(K)\}$ is a permutation of $\{1,\cdots,K\}$. 

(2) Block-diagonal deviation. \cite{wang2013provable} used the metric $\frac{\|C_{{\rm off}(J^*)}\|_1}{\|C\|_1}$ to define the relative deviation. However, relatively large diagonals may cause a relatively small value of this function, which may lead to an incorrect gloss for the deviation since such a small value does not imply a small deviation from the ideal block-diagonal form. We modify the metric to the following
\begin{align}\label{BdiagDevi} 
	{\rm BdiagDevi}(C) = \frac{\|C_{{\rm off}(J^*)}\|_1}{\|C\|_1-\|c\|_1} 
\end{align}
by removing the diagonals from the denominator, and use it to measure the deviation of $C$ from the ideal block-diagonal form.

(3) Intra-block connection. The Gini Index was used by \cite{Hurley2009Comparing} to measure the sparsity of a vector.  \cite{wang2013provable} adopted it as a sparsity metric ${\rm GiniIndex}(C_{J^*})$ of all the diagonal blocks $C_{J^*} = \{C_{J_k^*}\}$, that is, the sparsity of the vector of entries in the $K$ diagonal blocks $\{C_{J^*_k}\}$. We slightly modify it to measure the intra-block connection of $C$ by $1-{\rm GiniIndex}(C_{J^*})$. 
That is, sorting all the $M = \sum_{k=1}|J_k^*|^2$ entries of the diagonal blocks $C_{J^*} = \{C_{J^*_k}\}$ as $\tilde c_1\leq \cdots\leq \tilde c_{M}$, the intra-block connection is defined as follows
\begin{align}\label{IntraBConn}
	{\rm IntraBConn}(C) = \sum_{\ell=1}^{M} \frac{|\tilde c_\ell|}{\|C_{J^*}\|_1}\left( \frac{2(M - \ell) +1}{M} \right).
\end{align}

(4) $K$-block-diagonal structure. It was shown by \cite{vonluxburg2007a} that $C$ is a block-diagonal matrix with $K$ blocks if and only if the (normalized) Laplacian $L$ of $(|C|+|C|^T)/2$ has only $K$ zero eigenvalues. In the approximate case, $C$ is a block-diagonal matrix with $K$ diagonal blocks approximately if $L$ has $K$ small eigenvalues and its other eigenvalues are distinguishable from the $K$ smallest ones. The spectral clustering is just the $K$ partition of the (normalized) rows of the orthonormal matrix with $K$ unit eigenvectors corresponding to the smallest eigenvalues as its columns. By classical subspace perturbation theory \cite{Stewart1990Matrix}, the stability of the spectral clustering can be characterized by the gap between the $K$-th and $(K+1)$st smallest eigenvalues $\lambda_K(L)$ and $\lambda_{K+1}(L)$ of $L$. Hence, we use the relative gap 
\begin{align}\label{KblockDiag}
	{\rm KblockDiag}(L) = \frac{\lambda_{K+1}(L)-\lambda_K(L)}{\lambda_{K+1}(L)} 
\end{align}
to measure the stability of the spectral clustering, similar to $\frac{\lambda_{K+1}(L)-\lambda_K(L)}{\lambda_K(L)-\lambda_{K-1}(L)}$ used by \cite{Lauer2009Spectral}. Here, ${\rm RelGap}(L)\in[0,1]$.\footnote{The original definition defined by \cite{Lauer2009Spectral} could be arbitrarily large.} 

Each metric function above has the same range of $[0,1]$. The first two functions measure the approximation of the $K$ partition to the minimal partition $J^*$ and the approximation of $C$ to have a block-diagonal form as an MSDR. The last two functions measure the. connection of the diagonal blocks from two different viewpoints. Generally, smaller values of ${\rm ErrParti}(J)$ and ${\rm BdiagDev}(C)$ and relatively larger values of ${\rm IntraBCon}(C)$ and ${\rm KblockDiag}(L)$ mean a better solution $(J,C)$ for learning the minimal sample subspaces.

\subsection{Synthetic Data without Noise} \label{sec:syndata}

\begin{table}[t]                          
\centering                                
\resizebox{15cm}{!}{
\begin{tabular}{|c@{\ }c@{\ }c|@{\ }c@{\ }|r@{\, }r@{\, }r@{\, }r@{\, }r|r@{\, }r@{\, }r@{\, }r@{\, }r|r@{\, }r@{\, }r@{\, }r@{\, }r|}  \hline  \hline 
\multicolumn{4}{|c|}{} & \multicolumn{5}{c|}{$r = 10$}& \multicolumn{5}{c|}{$r = 14$}& \multicolumn{5}{c|}{$r = 20$}\\\hline
\multirow{4}{*}{\rotatebox{90}{Average}}&\multirow{4}{*}{\rotatebox{90}{\!dimension\,}}&\multirow{4}{*}{\rotatebox{90}{\!of $\S_k\!\cap\!\S_\ell$\,}}
&$d_c$ &  4 &  5  &  6  &  7 &  8 	&  7 &  8  &  9  & 10 & 11 & 11 & 12  & 13  & 14 & 15  \\\cline{4-19} 
&&&$\tau_\sigma=0.9$& 0.69 & 1.54 & 2.68 & 4.13 & 6.01 & 2.10 & 3.21 & 4.51 & 6.12 & 8.01 & 4.00 & 5.24 & 6.65 & 8.17 & 10.03  \\ 
&&&0.990& 0.03 & 0.59 & 2.03 & 4.00 & 6.00 & 0.78 & 2.11 & 4.00 & 6.00 & 8.00 & 2.25 & 4.02 & 6.00 & 8.00 & 10.00  \\ 
&&&0.999& 0.00 & 0.20 & 2.00 & 4.00 & 6.00 & 0.27 & 2.00 & 4.00 & 6.00 & 8.00 & 2.01 & 4.00 & 6.00 & 8.00 & 10.00  \\ 
\hline\hline 
\multirow{5}{*}{\rotatebox{90}{\!Percentage of\!}}&\multirow{5}{*}{\rotatebox{90}{\!samples with\,}}&\multirow{5}{*}{\rotatebox{90}{\!small dist$(x_i)$}}
&$[0.00, 0.05)$&  0.0 &  0.0 &  0.0 &  0.3 &  3.9 &  0.0 &  0.0 &  0.0 &  0.0 &  0.5 &  0.0 &  0.0 &  0.0 &  0.0 &  0.0  \\ 
&&&$[0.05, 0.10)$&  0.0 &  0.0 &  0.2 &  1.9 & 10.9 &  0.0 &  0.0 &  0.1 &  0.5 &  3.5 &  0.0 &  0.0 &  0.0 &  0.0 &  0.2  \\ 
&&&$[0.10, 0.30)$&  1.1 &  4.8 & 16.1 & 40.1 & 63.1 &  1.3 &  4.6 & 13.8 & 32.0 & 59.0 &  0.7 &  2.3 &  6.3 & 15.1 & 32.1  \\ 
&&&$[0.30, 0.50)$& 17.7 & 37.1 & 53.9 & 49.2 & 21.1 & 27.6 & 47.4 & 62.3 & 59.1 & 35.5 & 31.9 & 49.3 & 65.2 & 71.8 & 63.6  \\ 
&&&$[0.50, 1.00)$& 81.2 & 58.1 & 29.8 &  8.6 &  1.0 & 71.2 & 48.0 & 23.9 &  8.3 &  1.5 & 67.3 & 48.4 & 28.5 & 13.1 &  4.2  \\ 
\hline\hline 
\end{tabular}}
\caption{Synthetic data: subspace intersection with neighboring samples} \label{tab:data_quality}
\end{table}

The synthetic samples are generated such that the sample spanned subspaces are intersected with each other. We randomly choose an $r$-dimensional subspace in $\R^m$ given $r$, or equivalently, choose an orthogonal basis matrix $U\in \R^{m\times r}$. Then, we arbitrarily choose $K$ subspaces $\{\S_k\}$ with bases $U_k = UP_k$, $k=1,\cdots,K$, where each $P_k \in\R^{r\times d_k}$ is also an orthonormal matrix with a given column number $d_k<r$. The parameters $\{d_k\}$ and $r$ determine the subspace complexity and subspace segmentation difficulty. Generally, a smaller sequence $\{d_k\}$ makes segmentation easier. We choose different kinds of $\{d_k\}$ for generating sample sets with variant complexities of the segmentation. The samples are randomly chosen from each subspace $\S_k$ in the form $x = U_ky$, where the entries of $y$ are independent and identically distributed in the uniform distribution. Thus, the sample set can be written as $\{X_k=U_kY_k\}$. As shown in Theorem \ref{thm:ndgt}, these samples are intersected nondegenerately and hence, the set $\{X_k\}$ is the unique minimal segmentation, and we also have $\S_k = \spann(X_k)$ for each $k$. 

There are two kinds of complexities associated with subspace segmentation. One is the degree of subspace intersection. We define this degree by the ratios $\{d_k/r\}$. Generally, a smaller $d_k$ implies better sample construction in this way if $r$ is fixed. The other kind of complexity involves the existence of samples near another subspace and the percentage of samples within a short distance. The minimal distance of a sample $x_i\in\S_k$ to other subspaces is measured by 
\[
	{\rm dist}(x_i) = \min_{\ell\neq k}{\rm dist}(x,\S_\ell) = \min_{\ell\neq k}\frac{\|x_i-U_{\ell}(U_{\ell}^Tx_i)\|}{\|x_i\|}.
\]

We estimate the dimension of $\S_s\cap\S_t$ as the number of singular values of $U_s^TU_t$ satisfying $\sigma_i(U_s^TU_t)\geq \tau_\sigma$ for a $\tau_\sigma\approx 1$, which takes into account the proximity of two subspaces, excluding the intersection. In the top block of Table \ref{tab:data_quality}, we list the average dimensions of the pairwise-intersected subspaces of $K=5$ minimal sample subspaces with an equal dimension $d_c$ for 100 repeated experiments. The average dimension of the pair-wise subspace intersection
is approximately $(2d_c-r)_+$. We test three values of $r$ and, for each $r$, five values of $d_c$. In the bottom block, we also list the percentage of samples with the distances ${\rm dist}(x_i)$ in a given interval. Each time, we randomly choose $n_k=1000$ samples in each subspace. The quality of the samples improves as $r$ increases if the dimension of the intersected subspaces is approximately fixed. 

We run the MSS algorithms with MCG using $\Omega_c$ only (MCG(c)), the alternative optimization (AO) or the hybrid optimization (HO) on the synthetic data sets with the same settings of $r$ and $d_c$ as shown in Table \ref{tab:data_quality}. To reduce the computation time, we choose $n_k = 50$ samples in each subspace and repeat the experiments $10$ times for a total of 150 tests. For each $r$, the MSS algorithms succeed in recovering the true minimal sample subspaces on all the tests with four smaller values of $d_c$, but fail when $d_c$ is set to be the largest value. We choose a large $d_c$ for each $r$ to show the case when the MSS algorithms cannot retrieve the minimal subspaces.
In approximately 29\% of the 120 successful tests of the MSS algorithms, the MCG(c) method successfully gives a true solution where further alternative or hybrid optimization is not required. The MCG(c) solutions for the remaining tests are also suitable as initial guesses for further alternative optimization, and 22\% of them can be further improved to retrieve the minimal subspaces using the AO strategy within at most 5 AO steps. The other testings require the HO strategy to eventually obtain the true solution within at most 5 HO steps.
Table \ref{tab:RRSS} shows the convergence behavior of the MSS algorithms with the three strategies MCG(c), AO, and HO.

\begin{table}[h]                          
\centering                                
\begin{tabular}{|c|c|c|c|}  \hline  \hline 
Stage & \multicolumn{3}{c|}{Convergence behavior in the successful tests}  \\\hline 
MCG(c) 	& 35 succeed & \multicolumn{2}{c|}{85 fail} \\\hline 
AO 		& at most 5 iterations & 19 succeed &  66 fail \\\hline 
HO		&\multicolumn{2}{c|}{at most 5 iterations}& 66 succeed  \\
\hline\hline 
\end{tabular}
\caption{Convergence behavior of the MSS algorithms} \label{tab:RRSS} 
\end{table}

Table \ref{tab:efficiency} further characterizes the behaviur of the MSS method with the different strategies MCG(c), AO, and HO, measured by the average values of ErrParti, BdiagDevi, IntraBConn, and KblockDiag on 10 repeated tests. As in previous experiments, the strategies AO and HO improve the connection and estimation accuracy of the $K$-block structure on both the minimal partition and MSDR. The error functions ErrParti and BdiagDevi decrease quickly, while IntraBConn and KblockDiag increase in most cases when the subspaces are not overly intersected with each other. In the case with the largest $d_c$ for $r=10$, 14, or 20, the local minimizer of MSS$_-$MCG(c) is far from the ideal solution, due to heavy subspace intersection. The average dimension of the subspace intersections is large compared with the dimension of each subspace. The relative dimensions of the intersected subspaces are approximately 75\%, 72\%, and 67\% of the dimension of the sample space for $r=10,14,20$, respectively. The local optimal solution of MSS$_-$MCG(c) cannot serve as a good initial guess for the MSS$_-$AO or MSS$_-$HO.

In Table \ref{tab:efficiency}, we also compare the MSS method with four other algorithms, LRR, SSC, SoftS3C, and LRSSC, on the same data. LRR gives relatively strong intra-block connection but very poor $K$-block-diagonal structure and large approximation errors to minimal partitions or MSDRs. SoftS3C slightly improves upon SSC based on the four measurements, and LRSSC performs better than LRR, SSC, and SoftS3C on the approximation of $K$-block-diagonal form, but weakens the connections within diagonal blocks.
If the subspace intersection is very weak, the solutions of SSC, SoftS3C, and LRSSC approximate the minimal partitions well, but the block-diagonal structures are unclear since the solutions have large values of BdiagDevi and relatively small values of KblockDiag. The connections within diagonal blocks are also relatively weak. When $r$ is relatively large or the subspace intersection slightly increases, LRR, SSC, SoftS3C and LRSSC still fail to give acceptable solutions. We note that in the noiseless case, the objective function of CLAR is equal to 
$\log\det(I+C^TC) =\sum_{i=1}^n\log\big(1+\lambda_i(VV^T+H^TH)\big)$
by Lemma \ref{lma:C}.
It is obvious that CLAR gives the same solution with $H=0$ as LRR in the noiseless case. 
Hence, we omit a comparison with CLAR here. 

\begin{table}[!t]                          
\centering                                
\resizebox{15cm}{!}{
\begin{tabular}{|@{\,}c@{\,}| c@{\ }c@{\ }c@{\ }c@{\ }c|c@{\ }c@{\ }c@{\ }c@{\ }c|c@{\ }c@{\ }c@{\ }c@{\ }c|c@{\ }c@{\ }c@{\ }c@{\ }c|}  \hline  \hline 
 & \multicolumn{5}{c|}{ErrParti} 
 & \multicolumn{5}{c|}{BdiagDevi}
 & \multicolumn{5}{c|}{IntraBConn}
 & \multicolumn{5}{c|}{KblockDiag} 
 \\\hline
$r=10,\ d_c=$ &   4 &   5 &   6 &   7 &   8 &   4 &   5 &   6 &   7 &   8 &   4 &   5 &   6 &   7 &   8 &   4 &   5 &   6 &   7 &   8\\\hline
MSS$_-$MCG(c)   & 0.00 & 0.00 & 0.02 & 0.27 & 0.60 
                & 0.02 & 0.11 & 0.37 & 0.66 & 0.74 
                & 0.57 & 0.58 & 0.54 & 0.43 & 0.37 
                & 0.92 & 0.78 & 0.39 & 0.04 & 0.01
                \\
MSS$_-$AO       & 0.00 & 0.00 & 0.01 & 0.16 & 0.60 
                & 0.00 & 0.00 & 0.02 & 0.30 & 0.72
                & 0.58 & 0.58 & 0.57 & 0.44 & 0.24 
                & 1.00 & 0.99 & 0.97 & 0.79 & 0.41  
                \\
MSS$_-$HO       & 0.00 & 0.00 & 0.00 & 0.00 & 0.72
                & 0.00 & 0.00 & 0.00 & 0.01 & 0.65
                & 0.58 & 0.58 & 0.58 & 0.57 & 0.27 
                & 1.00 & 1.00 & 1.00 & 0.99 & 0.45
                \\ \hline
LRR             & 0.14 & 0.27 & 0.50 & 0.59 & 0.67
                & 0.65 & 0.70 & 0.73 & 0.76 & 0.78 
                & 0.57 & 0.57 & 0.57 & 0.57 & 0.57 
                & 0.04 & 0.02 & 0.01 & 0.01 & 0.01
                \\
SSC             & 0.00 & 0.01 & 0.09 & 0.39 & 0.62 
                & 0.08 & 0.25 & 0.42 & 0.55 & 0.65 
                & 0.11 & 0.11 & 0.09 & 0.08 & 0.07 
                & 0.65 & 0.31 & 0.10 & 0.03 & 0.01
                \\
SoftS3C         & 0.00 & 0.01 & 0.08 & 0.37 & 0.61 
                & 0.11 & 0.26 & 0.41 & 0.53 & 0.65 
                & 0.14 & 0.12 & 0.10 & 0.08 & 0.07 
                & 0.59 & 0.32 & 0.12 & 0.05 & 0.03
                \\
LRSSC           & 0.00 & 0.01 & 0.06 & 0.30 & 0.62 
                & 0.02 & 0.12 & 0.31 & 0.47 & 0.61 
                & 0.05 & 0.05 & 0.05 & 0.04 & 0.03 
                & 0.85 & 0.46 & 0.14 & 0.03 & 0.02
                \\ \hline
$r=14,\ d_c=$ &   7 &   8 &   9 &  10 &  11 &   7 &   8 &   9 &  10 &  11 &   7 &   8 &   9 &  10 &  11 &   7 &   8 &   9 &  10 &  11\\\hline
MSS$_-$MCG(c)   & 0.00 & 0.00 & 0.09 & 0.35 & 0.59 
                & 0.22 & 0.41 & 0.62 & 0.70 & 0.74 
                & 0.57 & 0.54 & 0.46 & 0.40 & 0.37 
                & 0.63 & 0.36 & 0.09 & 0.02 & 0.01
                \\
MSS$_-$AO       & 0.00 & 0.00 & 0.01 & 0.20 & 0.56 
                & 0.00 & 0.00 & 0.03 & 0.40 & 0.70 
                & 0.57 & 0.57 & 0.56 & 0.44 & 0.31 
                & 1.00 & 1.00 & 0.95 & 0.59 & 0.37
                \\
MSS$_-$HO       & 0.00 & 0.00 & 0.00 & 0.00 & 0.70 
                & 0.00 & 0.00 & 0.00 & 0.03 & 0.66 
                & 0.57 & 0.57 & 0.57 & 0.57 & 0.32 
                & 1.00 & 1.00 & 1.00 & 0.94 & 0.37
                \\ \hline
LRR             & 0.08 & 0.29 & 0.45 & 0.56 & 0.64 
                & 0.70 & 0.72 & 0.74 & 0.76 & 0.77 
                & 0.57 & 0.57 & 0.57 & 0.57 & 0.57 
                & 0.02 & 0.01 & 0.01 & 0.00 & 0.00
                \\
SSC             & 0.00 & 0.02 & 0.13 & 0.42 & 0.59 
                & 0.29 & 0.41 & 0.51 & 0.60 & 0.66 
                & 0.12 & 0.11 & 0.10 & 0.09 & 0.08 
                & 0.33 & 0.16 & 0.05 & 0.02 & 0.01
                \\
SoftS3C         & 0.00 & 0.03 & 0.12 & 0.38 & 0.61
                & 0.31 & 0.41 & 0.50 & 0.59 & 0.65 
                & 0.14 & 0.12 & 0.11 & 0.09 & 0.08 
                & 0.32 & 0.18 & 0.08 & 0.04 & 0.03
                \\
LRSSC           & 0.00 & 0.02 & 0.14 & 0.46 & 0.61 
                & 0.15 & 0.30 & 0.44 & 0.55 & 0.63 
                & 0.07 & 0.07 & 0.06 & 0.05 & 0.05 
                & 0.51 & 0.21 & 0.06 & 0.02 & 0.01
                \\ \hline
$r= 20,\ d_c=$ &  11 &  12 &  13 &  14 &  15 &  11 &  12 &  13 &  14 &  15 &  11 &  12 &  13 &  14 &  15 &  11 &  12 &  13 &  14 &  15\\\hline
MSS$_-$MCG(c)   & 0.01 & 0.04 & 0.14 & 0.30 & 0.51
                & 0.55 & 0.64 & 0.68 & 0.71 & 0.73 
                & 0.49 & 0.44 & 0.41 & 0.38 & 0.37 
                & 0.19 & 0.08 & 0.04 & 0.01 & 0.01
                \\
MSS$_-$AO       & 0.00 & 0.00 & 0.02 & 0.15 & 0.43 
                & 0.00 & 0.02 & 0.05 & 0.50 & 0.68 
                & 0.56 & 0.56 & 0.55 & 0.47 & 0.38 
                & 0.99 & 0.97 & 0.92 & 0.34 & 0.24
                \\
MSS$_-$HO       & 0.00 & 0.00 & 0.00 & 0.00 & 0.49 
                & 0.00 & 0.00 & 0.00 & 0.24 & 0.59 
                & 0.56 & 0.56 & 0.56 & 0.56 & 0.43 
                & 1.00 & 1.00 & 1.00 & 0.63 & 0.24
                \\ \hline
LRR             & 0.11 & 0.18 & 0.31 & 0.51 & 0.58 
                & 0.72 & 0.73 & 0.75 & 0.76 & 0.77 
                & 0.56 & 0.56 & 0.56 & 0.56 & 0.56 
                & 0.02 & 0.01 & 0.01 & 0.01 & 0.00
                \\
SSC             & 0.01 & 0.04 & 0.12 & 0.26 & 0.50 
                & 0.43 & 0.50 & 0.56 & 0.61 & 0.65 
                & 0.13 & 0.12 & 0.12 & 0.11 & 0.10 
                & 0.17 & 0.11 & 0.05 & 0.02 & 0.01
                \\
SoftS3C         & 0.01 & 0.04 & 0.11 & 0.24 & 0.47 
                & 0.43 & 0.50 & 0.55 & 0.60 & 0.65 
                & 0.15 & 0.14 & 0.13 & 0.12 & 0.11 
                & 0.19 & 0.13 & 0.08 & 0.05 & 0.03
                \\
LRSSC           & 0.01 & 0.05 & 0.13 & 0.30 & 0.52 
                & 0.33 & 0.44 & 0.51 & 0.57 & 0.62 
                & 0.09 & 0.08 & 0.08 & 0.07 & 0.07 
                & 0.27 & 0.14 & 0.05 & 0.03 & 0.01
                \\ \hline\hline 
\end{tabular}}
\caption{Results of the four measurements on the solutions of the compared algorithms} \label{tab:efficiency} 
\end{table}

\subsection{Real-world Data}\label{sec:real data}

We use the benchmark databases Extended YaleB for face clustering  \cite{GeBeKr01} and Hopkin 155 for motion segmentation \cite{Tron2007A} to evaluate the performance of Algorithm \ref{alg:noise} (MSS$_-$RO) on noisy data. 
Because ignorable noise exists in real-world data, segmentation is obtained by applying spectral clustering on a constructed graph $A$, rather than the solution $C$ itself. Special postprocessing of the computed solution $C$ might be required for constructing the graph $A$, which aims to strengthen the block-diagonal structure of $C$ or increase class similarities. SSC and LRR/CLAR adopt different postprocessing approaches as shown below, while SoftS3C uses $A=(|C|+|C^T|)/2$.

The postprocessing of a solution $C$ for SSC cuts of the small entries of $C$ in absolute value to obtain a sparser $C_\gamma$ such that each column has the norm $\|C_\gamma e_j\|\approx\gamma\|C e_j\|$ with a given positive $\gamma$ close to 1. Each column of $C_\gamma$ should be further normalized to have the largest entry equal to one in absolute value. Let $\tilde C_\gamma$ be the normalized $C_\gamma$. Then, graph $A$ is set as the symmetric part of $\tilde C_\gamma$, or equivalently, $A = |\tilde C_\gamma|+|\tilde C_\gamma^T|$. In the experiments on the two databases, we set $\gamma = 1$ for Extended YaleB and $\gamma = 0.7$ for Hopkins155 as suggested by SSC.

LRR or CLAR truncates the SVD of $C$ to $U_\sigma\Sigma_\sigma V_\sigma^T$ by removing the singular values smaller than threshold $\sigma$ and the corresponding singular vectors. Then, the rows $\{u_i\}$ of $U_\sigma\Sigma_\sigma^{1/2}$ are used to construct the graph $A=(a_{ij})$ as per
$a_{ij} = (\frac{\langle u_i,u_j\rangle}{\|u_i\|\|u_j\|})^s$ with $s \geq 1$. We set $s=4$, $\sigma=10^{-4}\|C\|_2$ for LRR and $s=4$, $\sigma=10^{-6}$ for CLAR, as suggested by their authors.

We use the strategy of double truncating to construct the graph shown in Algorithm \ref{alg:G} with the same settings $\gamma = 1$, $\sigma = 0$, and $s = 1$ for Extended YaleB, and $\gamma = 0.8$, $\sigma = 0.001$, and $s = 4$ for Hopkins155. Notably, the strategy of double truncating does not work well for SSC or LRR/CLAR. Since no suggestions about graph constructing were given for LRSSC in the literature, We test the three approaches adopted in SSC, LRR, or our method for LRSSC, including its two modified versions mLRSSC and aLRSSC that will be mentioned latter, and report the best results.

\subsubsection{Facial Image Clustering}
The Extended YaleB database consists of 2432 facial images ($192\times 168$ pixels) from 38 individuals under 64 illumination conditions. Due to various illumination and shadow conditions, these images have a relatively large amount of noise and corruption. It makes sense to use images from the same individual as a groundtruth class, and such images in each class are sampled from the same subspace approximately. It was pointed out by \cite{basri2003lambertian} that a facial image lives in a 9 dimensional subspace. Thus, each of the subspaces has dimension 9, and we estimate $d = 9K$ if we have images from $K$ individuals.

 \begin{table}[t]
\centering
\resizebox{15.2cm}{!}{
\begin{tabular}{|@{\,}c@{\,}|c@{\ \,}c@{\ \,}c@{\ \,}c|c@{\ \,}c@{\ \,}c@{\ \,}c|c@{\ \,}c@{\ \,}c@{\ \,}c|c@{\ \,}c@{\ \,}c@{\ \,}c|c@{\ \,}c@{\ \,}c@{\ \,}c|}
\hline\hline
\multirow{2}{*}{Algorithm} &\multicolumn{4}{c|}{$K= 2$} &\multicolumn{4}{c|}{$K= 3$} &\multicolumn{4}{c|}{$K= 5$} &\multicolumn{4}{c|}{$K= 8$} &\multicolumn{4}{c|}{$K=10$}\\\cline{2-21}
    & $I_0$ & $I_1$ & $I_2$ & $I_3$ & $I_0$ & $I_1$ & $I_2$ & $I_3$ 
    & $I_0$ & $I_1$ & $I_2$ & $I_3$ & $I_0$ & $I_1$ & $I_2$ & $I_3$ 
    & $I_0$ & $I_1$ & $I_2$ & $I_3$\\\hline
LRR & 39.3  & 13.5  & 41.7  &  5.5  &  9.1  & 13.0  & 71.2  &  6.7 
    &  1.5  &  6.3  & 74.1  & 18.1  &  0.0  &  0.0  & 55.1  & 44.9 
    &  0.0  &  0.0  & 33.3  & 66.7\\
CLAR & 44.2 & 24.5  & 31.3  &  0.0  & 16.6  & 17.1  & 66.3  &  0.0 
    &  6.8  & 14.4  & 78.8  &  0.0  &  0.7  &  6.6  & 92.6  &  0.0 
    &  0.0  &  0.0  & 100.0 &  0.0\\
SSC & 58.3  & 21.5  & 16.6  &  3.7  & 27.6  & 22.8  & 42.1  &  7.5 
    &  5.5  & 18.6  & 63.2  & 12.7  &  0.0  &  4.4  & 76.5  & 19.1 
    &  0.0  &  0.0  & 66.7  & 33.3\\
SoftS3C& 67.5 & 20.2& 11.7  &  0.6  & 39.9  & 24.0  & 36.1  &  0.0 
    & 20.9  & 24.9  & 50.0  &  4.2  &  2.2  & 10.3  & 66.9  & 20.6 
    &  0.0  &  0.0  & 66.7  & 33.3\\
mLRSSC& 73.6& 15.3  & 11.0  &  0.0  & 46.4  & 23.8  & 28.8  &  1.0 
    & 28.3  & 20.4  & 45.9  &  5.3  & 11.0  &  8.1  & 57.4  & 23.5 
    &  0.0  &  0.0  & 66.7  & 33.3\\
MSS$_-$RO(c)& 82.2 & 12.9 &4.9 &0.0 & 63.0  & 16.6  & 18.8  &  1.7 
    & 53.3  & 20.6  & 22.8  &  3.3  & 28.7  & 28.7  & 38.2  &  4.4 
    &  0.0  & 33.3  & 66.7  &  0.0\\
MSS$_-$RO&82.2& 12.9&  4.9  &  0.0  & 63.7  & 18.5  & 16.1  &  1.7 
    & 57.1  & 25.0  & 15.5  &  2.3  & 44.1  & 19.1  & 35.3  &  1.5 
    & 33.3 &  0.0 & 66.7 &  0.0\\
\hline\hline
 \end{tabular}}
 \caption{Percentage (\%) of computed segmentations with errors belonging to each of the intervals $I_0 = [0, 0.005]$, $I_1 = (0.005, 0.01]$, $I_2 = (0.01, 0.1]$, and $I_3 = (0.1, 0.5]$.}\label{tab:face}
 \end{table}

The testing sets are chosen as follows. A total of 38 individuals are divided into 4 groups---each of the first three groups contains 10 individuals and the last group contains the remaining 8 individuals. In each group, we choose $K$ individuals and test the segmentation of all $64K$ images from the chosen individuals. Since there are $C_p^K$ combinations of $K$ individuals among $p$ individuals, we have $3C_{10}^K+C_{8}^K$ tests for a fixed $K$. Since the five values of $K$ are set as 2, 3, 5, 8, and 10, we have 163, 416, 812, 136, and 3 tests for the five settings of $K$, respectively, and we have 1530 tests using different sizes of sample sets for this database. We downsample the large images to $48\times 42$ pixels and vectorize them as 2016-dimensional vectors $\{x_j\}$ to obtain a reasonable computational complexity, as done by \cite{Elhamifar2013Sparse}. The samples are  normalized to have the unit norms $\|x_j\|_2 = 1$. The self-expressive error is measured by the $\ell_1$-norm $\phi=\|\cdot\|_1$. We set the parameters $\lambda = 5,\alpha = \frac{20}{\|X\|_1^*}$, and $\beta = 5$, where $\|X\|_1^* = \max_j \|x_j\|_1 $. 
The strategy for setting $\alpha$ is the same as that used for SSC and SoftS3C.

\begin{figure}[h]
	\centering
	\includegraphics[width = 0.45\textwidth]{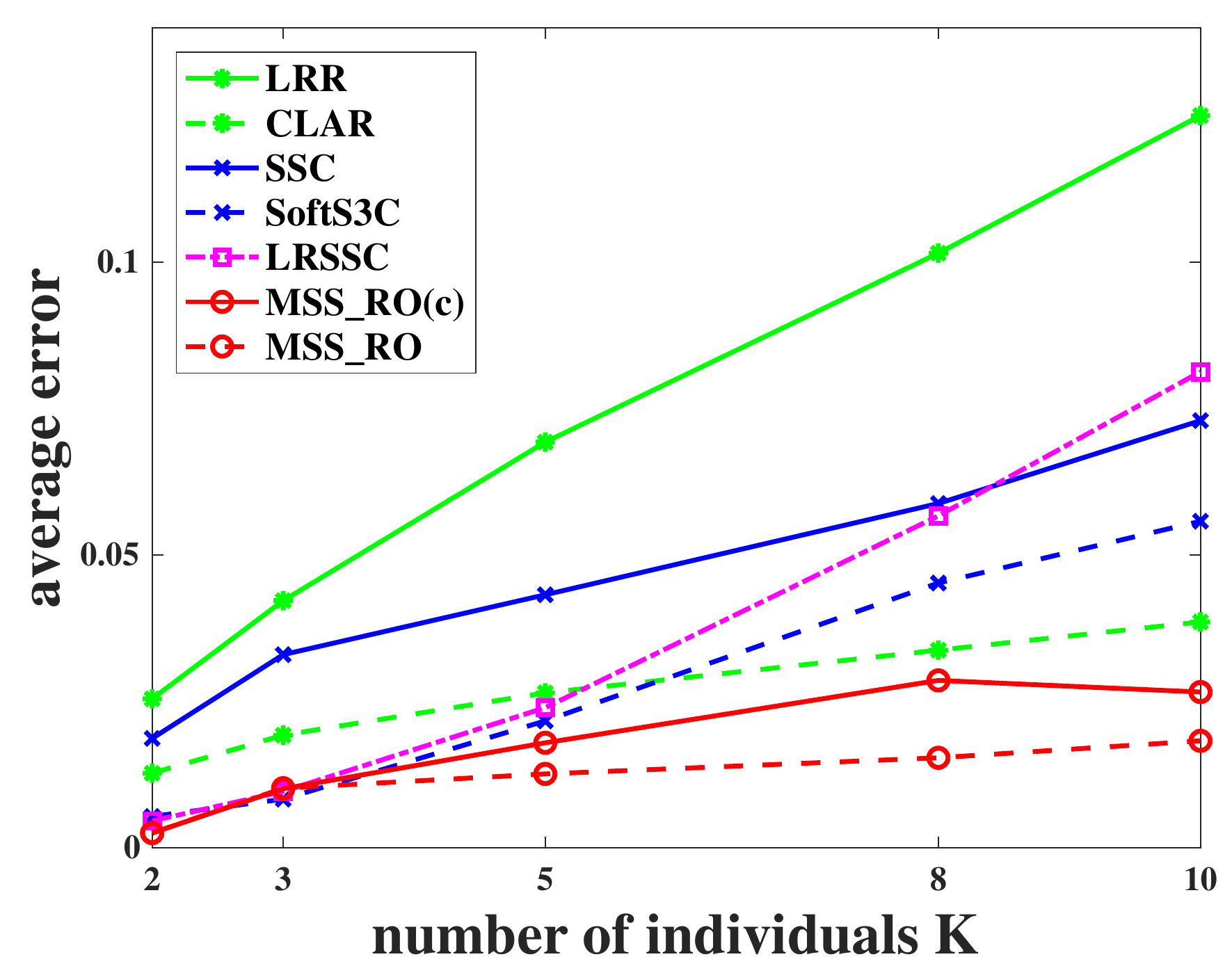}
	\includegraphics[width = 0.45\textwidth]{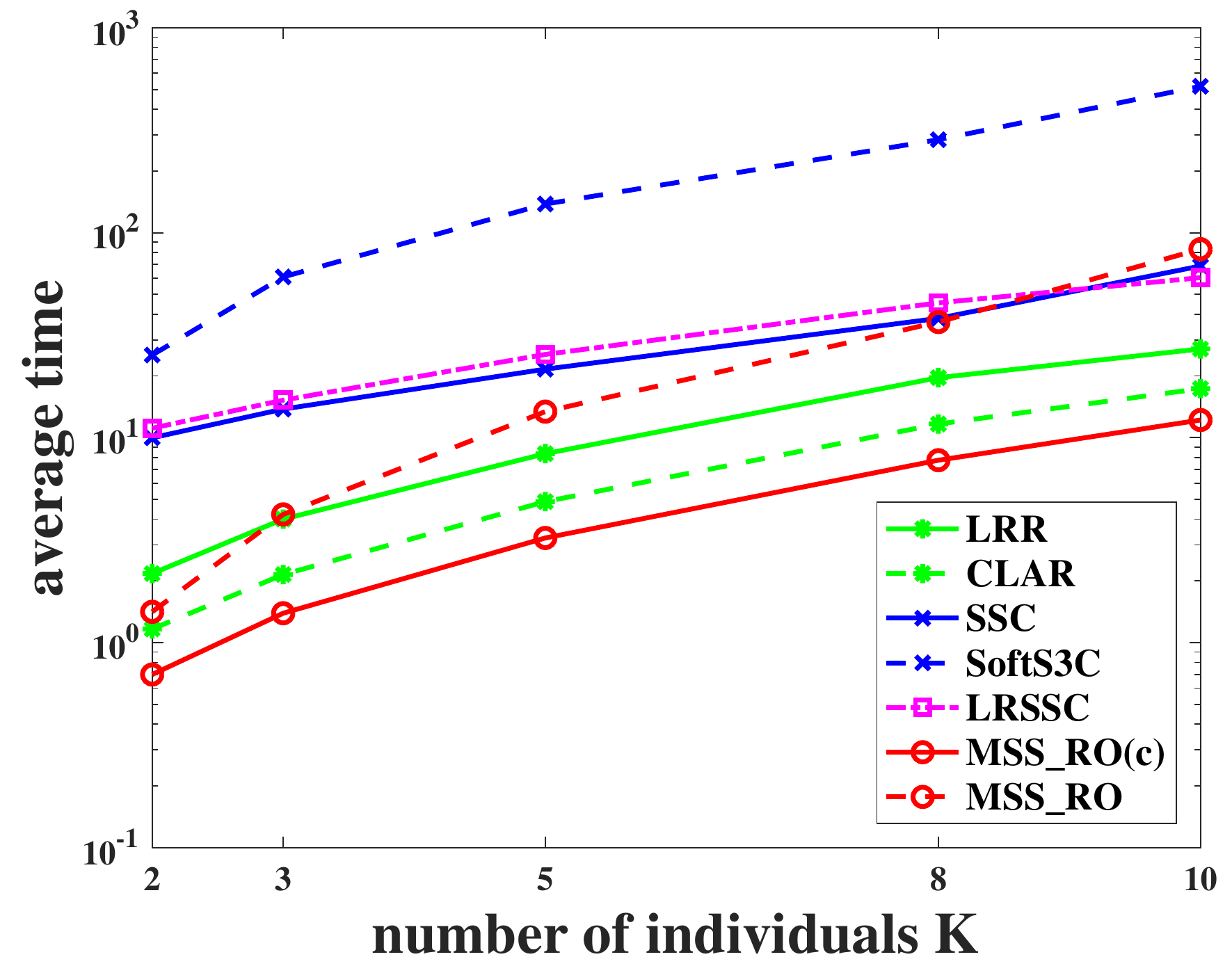}
	\caption{Percentage of average clustering errors (left) and average computation time in seconds (right) of the algorithms for detecting facial images of $K$ individuals with $K= 2, 3, 5, 8, 10$, respectively.}
	\label{fig:face_mr_t}
\end{figure}

Due to the existence of noise, the six algorithms cannot correctly retrieve the true classes completely. To show the efficiency of these algorithms, we separate the computed clustering into four groups according to the partition error ErrParti belonging to one of four intervals $I_0 = [0, 0.005]$, $I_1 = (0.005, 0.01]$, $I_2 = (0.01, 0.1]$, and $I_3 = (0.1, 0.5]$. We point out that LRSSC adopts $\|X-XC\|_F^2$ as a penalty term and imposes the zero restriction on diagonals as in SSC for combining LRR with SSC. Such a combination cannot give better results than LRR or SSC on this data set. We modify LRSSC by changing the Frobenius norm ($\ell_2$-norm) penalty to the $\ell_1$-norm penalty $\|X-XC\|_1$.
The modified LRSSC (marked as mLRSSC) performs better than LRR, SSC, and CLAR, as shown in Table \ref{tab:face}. When $K=2$, mLRSSC can correctly detect the two individuals in approximately 2/3 tests. When $K$ is slightly increased to 3, its detection percentage decreases to 39.9\%, which is also higher than that of the other four methods. When more individuals are detected, the detecting error, that is, the partition error, quickly enlarges. For example, the values of ErrParti are greater than 0.01 for all the solutions of LRR, SSC, mLRSSC, SoftS3C and CLAR when $K=10$. 
Our MSS$_-$RO(c) and MSS$_-$RO algorithms perform very well in this experiment. Essentially, the MSS$_-$RO(c) method provides a better solution than those of LRR, SSC, mLRSSC, SoftS3C and CLAR in most cases, especially with large $K$. The percentage of MSS$_-$RO(c) segmentations with errors in $I_0$ are $82.2\%$, $63.0\%$ and $53.3 \%$ for $K=2$, 3, or 5 individuals, respectively, which is much higher than those of LRR, SSC, mLRSSC, SoftS3C and CLAR. Even in the more complicated case for detecting $K=10$ individuals, 1/3 of the tests achieve a detection error less than 0.01. Within at most two updates of $\Omega$ ($\ell_{\max}\leq 3$), MSS$_-$RO presents its final results in this example. Table \ref{tab:mr_yaleB} gives direct performance comparison of the subspace clustering methods.

\begin{table}[h]
	\centering
	\begin{tabular}{|c|c|c|c|c|c|c|c|c|c|c|}		
		\hline \hline
		\multirow{2}{*}{Algorithm} & \multicolumn{2}{|c|}{$K=2$}&  \multicolumn{2}{|c|}{$K=3$} & \multicolumn{2}{|c|}{$K=5$} & \multicolumn{2}{|c|}{$K=8$} & \multicolumn{2}{|c|}{$K=10$} \\
		\cline{2-11}
		 & mean & median & mean & median & mean & median & mean & median & mean & median \\
		\hline
		LRR & 2.54 & 0.78 & 4.22 & 2.60 & 6.92 & 5.63 & 10.15 & 9.28 & 12.50 & 12.66 \\
		CLAR & 1.27 & 0.78 & 1.92 & 1.56 & 2.64 & 2.19 & 3.36 & 3.03 & 3.85 & 3.44\\
		SSC  & 1.86 & 0.00 & 3.30 & 0.52 & 4.32 & 2.81 & 5.88 & 4.59 & 7.29 & 5.47\\
		SoftS3C & 0.53 & 0.00 & 0.83 & 0.52 & 2.16 & 1.25 &  4.52 & 2.25 & 5.57 & 2.34\\
		mLRSSC & 0.46 & 0.00 & 0.96 & 0.52 & 2.39 & 1.25 & 5.67 & 3.42 & 8.12 & 5.78 \\
		MSS$_-$RO(c) & 0.25 & 0.00 & 1.01 & 0.00 & 1.79 & 0.31 & 2.86 & 0.78 & 2.66 & 1.09 \\
		MSS$_-$RO & 0.25 & 0.00 & 1.01 & 0.00 & 1.26 & 0.31 & 1.53 & 0.59 & 1.82 & 1.09 \\
		\hline \hline
	\end{tabular}
	\caption{ Percentage of average and median errors (\%) of different algorithms for detecting facial images of $K$ individuals with $K=2,3,5,8,10$, respectively.} \label{tab:mr_yaleB}
\end{table}

The computational cost of MSS$_-$RO is also competitive. MSS$_-$RO(c) is faster than the other algorithms in the five cases for variant $K$, without updating the coarsest $\Omega_c$. If $\Omega$ is updated within several iterations, the computation time of MSS$_-$RO increases but remains smaller than SoftS3C. On the right side of Figure \ref{fig:face_mr_t}, we plot the computation time in seconds. We also plot the average clustering error with the same number of individuals on the left and do not separate the solutions into several groups as before.\footnote{The results of SoftS3C for $K=5,8,10$ reported in \cite{Li2017Structured} are 
somehow better than those we obtained using the codes provided by the authors, but still worse than the results of MSS$_-$RO. 
}

\subsubsection{Motion Segmentation}

The Hopkins 155 database contains 155 videos of rigidly moving objects in which 120 videos have two objects and 35 videos have three objects. To separate the objects in each video, the feature points (2D coordinates) of the objects are first extracted from the frames of each video. The 2D coordinates of the same feature point in the sequence of frames form a long sample vector of length $2F$, where $F$ is the number of frames in the video.  Because the objects are rigidly moving, the $2F$-dimensional vectors $\{y_i\}$ of the same object belong to an affine subspace with a dimension of at most 3, as pointed out by \cite{costeira1998a}. The affine data set $\{y_i\}$ corresponding to an object can be modified to be a subspace data set $\{x_i\}$, where each $x_i$ has an additional constant entry 1. The objects in a video are detected by segmenting $\{x_i\}$. Since each of the linear subspaces has a dimension of at most 4, we set $d=4K$ if the number of objects (subspaces) is $K$. In this data set, $K=2$ or $K=3$.
SSC and SoftS3C also have affine versions (termed aSSC or aSoftS3C, respectively) in which the sum of the entries in each column of $C$ is restricted to be one. We report the results of these affine versions on $\{y_i\}$ because of their improved performance over their original versions on $\{x_i\}$. LRR, CLAR, and the noise-handling version MMS$_-$RO of MSS work on $\{x_i\}$. Similar to the original SSC/SoftS3C, LRSSC does not work well on $\{x_i\}$. We modify LRSSC again (termed  aLRSSC) by adding a restriction on $C$ as in SSC/SoftS3C so that it can works on the affine data $\{y_i\}$ and gives better results similar to the affine versions of SSC/SoftS3C.
In our algorithm MSS$_-$RO, we use $\phi=\|\cdot\|_F^2$ for measuring the self-expressive error, and set the parameters 
\begin{align} \label{eq:paras_motion}
    \lambda = 10n/K, \ \alpha = 50n/K, \ \beta = 0.05n/K,
\end{align} 
where $n$ is the number of sample vectors, which vary from 39 to 556 in this experiment. 

\begin{figure}[t]
	\centering
	\includegraphics[width=0.45\textwidth]{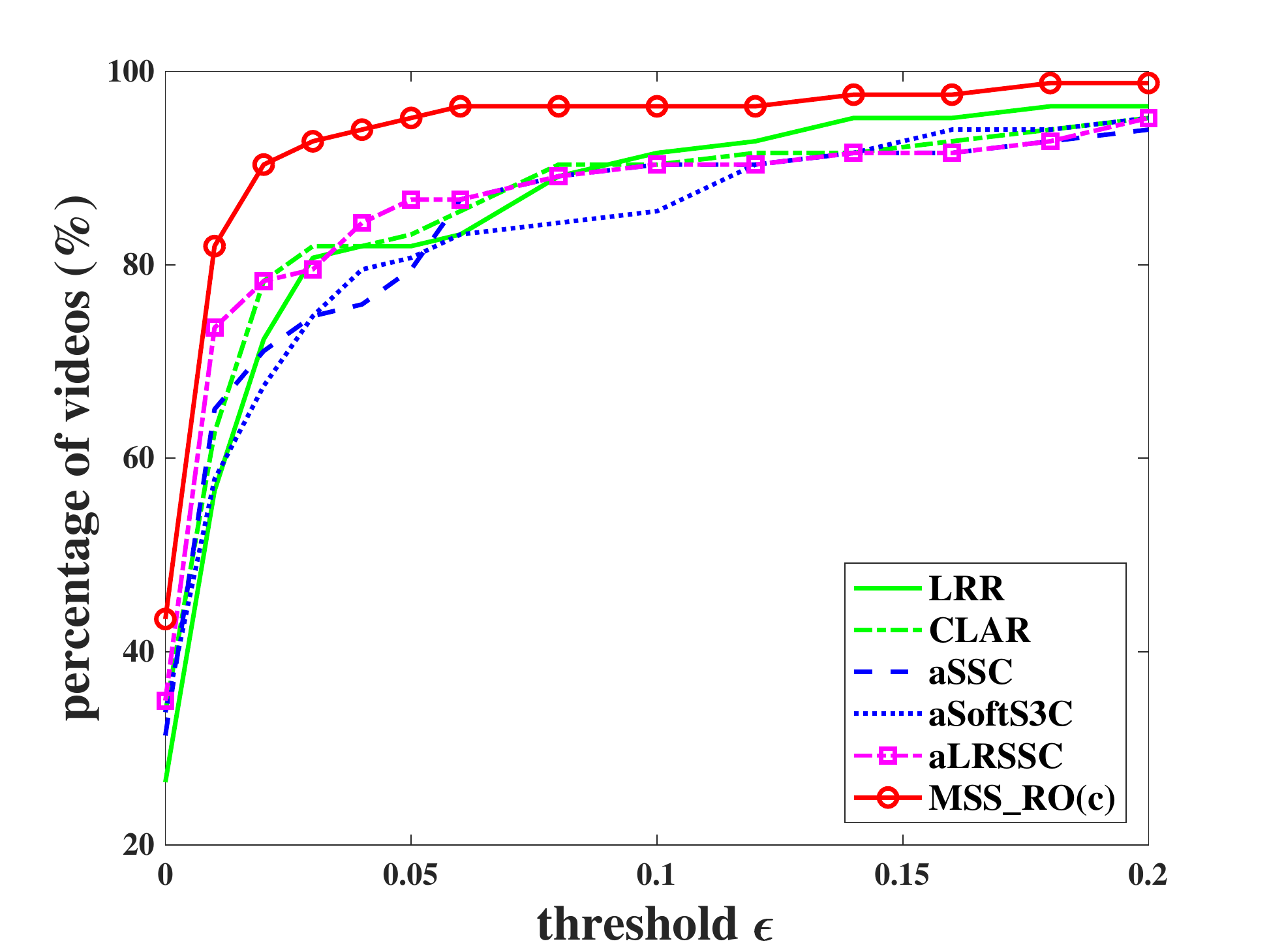}
	\includegraphics[width=0.45\textwidth]{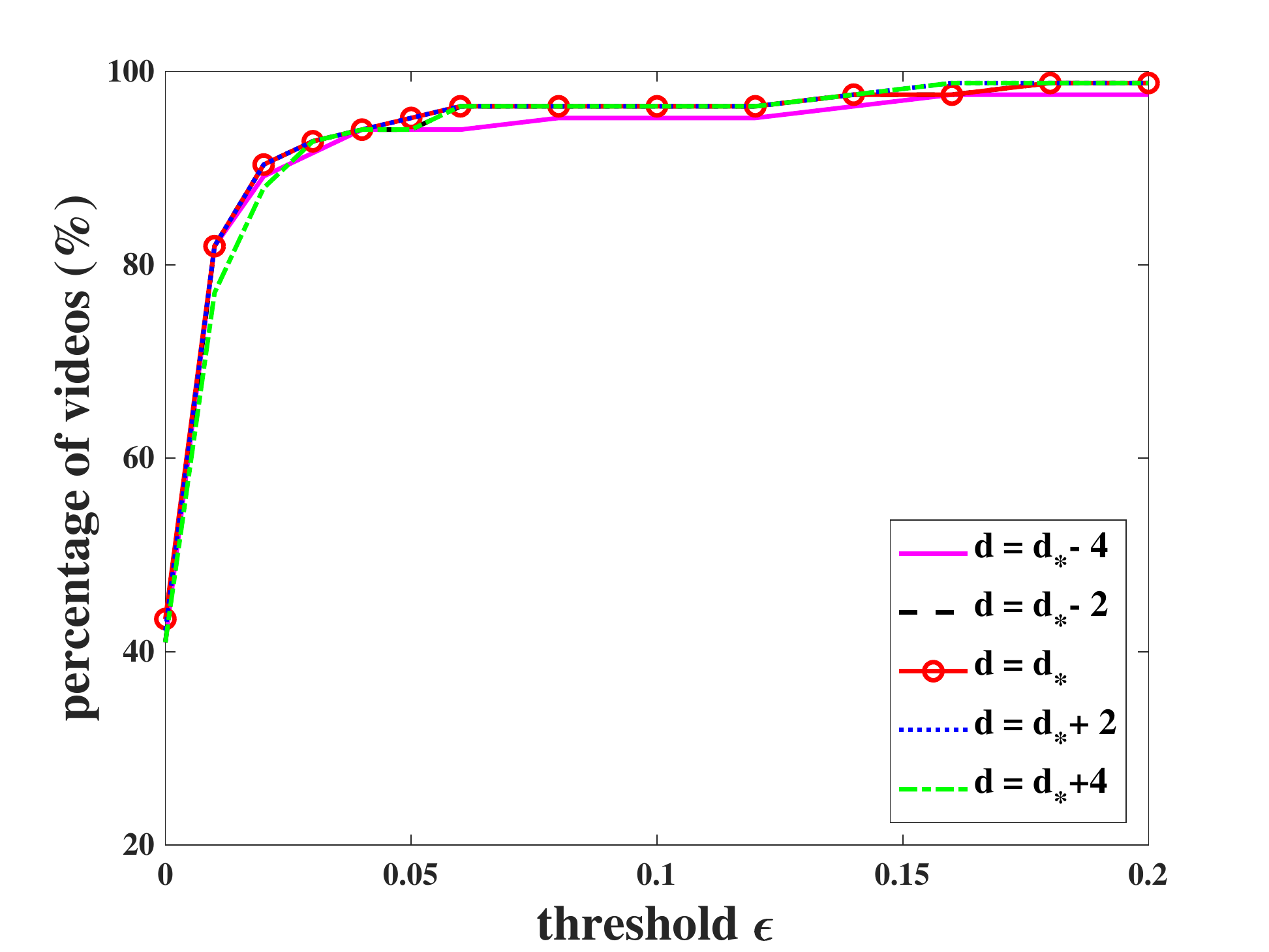}
	\caption{Percentage of videos whose detection error less than given threshold $\epsilon$ in the set of 83 `hard' videos. The left one compares different methods, where correct $d$ is set for MSS$_-$RO. The right one shows the sensitivity of MSS$_-$RO to $d$ that varies around the correct sum of subspace dimensions $d_*$.} 
	\label{fig:motion}
\end{figure} 

In this data set, there are 72 videos whose objects are easily detected---all of the compared algorithms can correctly detect all the objects. The remaining 83 videos are relatively hard to detect---no one can be completely detected by all the compared algorithms. In this experiment, we focus on performances of these algorithms on these 83 challenging videos, by checking the distribution of the detecting errors.
The left panel of Figure \ref{fig:motion} plots the percentage of videos whose detecting errors are not larger than the threshold value $\epsilon$ for each of the six algorithms LRR, CLAR,  aSSC, aSoftS3C, aLRSSC, and MSS$_-$RO. 
The LRR and CLAR solutions have detection errors relatively smaller than those of aSSC and aSoftSSC. Our LRSSC modification causes it to perform slightly better than aSSC. MSS$_-$RO(c) still significantly outperforms all of the compared algorithms. Table \ref{tab:mr_motion} gives a direct performance comparison of the subspace clustering methods.

\begin{table}[!hbp]
	\centering
	\begin{tabular}{|c|c|c|c|c|c|c|}		
		\hline \hline
		\multirow{2}{*}{Algorithm} & \multicolumn{2}{|c|}{$K=2$}&  \multicolumn{2}{|c|}{$K=3$} & \multicolumn{2}{|c|}{all videos}  \\
		\cline{2-7}
		& mean & median & mean & median & mean & median  \\
		\hline
		LRR & 1.36 & 0.00 & 2.51 & 0.00 & 1.62 & 0.00  \\
		CLAR & 1.32 & 0.00 & 2.60 & 0.51 & 1.61 & 0.00 \\
		aSSC  & 1.53 & 0.00 & 4.40 & 0.56 & 2.18 & 0.00 \\
		aSoftS3C & 1.64 & 0.00 & 4.41 & 0.56 & 2.20 & 0.00\\
		aLRSSC & 1.22 & 0.00 & 3.67 & 0.21 & 1.77 & 0.00 \\
		MSS$_-$RO(c) & 0.50 & 0.00 & 0.85 & 0.00 & 0.58 & 0.00  \\
		MSS$_-$RO & 0.43 & 0.00 & 0.86 & 0.19 & 0.52 & 0.00  \\
		\hline \hline
	\end{tabular}
	\caption{Percentage of average and median errors (\%) of different algorithms for detecting videos with $K=2$, $K=3$ and all videos, respectively.} \label{tab:mr_motion}
\end{table} 

Compared with algorithms in the literature for subspace learning, our MSS method needs an estimated value $d$ of the true sum $d_*$ of subspace dimensions as an additional prior. 
For some real-world data sets like YaleB and Hopkins155, the correct $d_*$ is practically known. The right panel of Figure \ref{fig:motion} also numerically show the sensitivity of MSS$_-$RO when it uses estimated $d$ slightly varying around the correct ones. We set $d = d_*\pm \delta$, $\delta=0,2,4$, where $d_* = 4K$ is the correct one. The data set is the 83 challenging videos as previously. As in the left panel of Figure \ref{fig:motion}, we also plot the percentage curve versus the threshold $\epsilon$ for each $d$ in the right panel. It seems that MSS$_-$RO is robust on $d$ in this data set. We also checked the sensitivity of MSS$_-$RO on the set of other 72 videos. Only two videos are wrongly detected by MSS$_-$RO.

\section{Conclusion}
Subspace learning is a challenging task not only in theoretical analysis but also in modeling and computing for applications. In applications, ground-truth subspaces may be different from those that can be mathematically defined based on finite samples. Mathematically, these subspaces may be heavily intersected with each other, and some samples may be difficult to separated if they are proximal to the intersected subspaces. The existence of noise further complicates the problem. In this paper, we provided the concept of minimal sample subspace and considered the segmentation as a union of minimal sample subspaces, together with a so-called pure subspace that is mostly nonexistent in applications. However, even for the minimal subspace segmentation, which is now well-defined, the problem is also difficult and complicated since the MSS may be not unique. We gave sufficient conditions for addressing this uniqueness, and built some solid theoretical bases to support our proposed optimization modeling for conditionally recovering the minimal sample subspaces even if these subspaces themselves are heavily intersected. However, there are still some theoretical problems that need to be addressed, such as the sensitivity of the MSS problem. Perturbation analysis should be given to address  the reliability of the retrieved MSS and also help us to understand segmentation on noisy samples. We consider this a difficult but interesting topic for further research.

We proposed several algorithms for solving the MSS problem with or without noise. However, because of the complexity of the optimization problems, a globally optimal solution is not always guaranteed, though we did endeavor to obtain a globally optimal solution of this nonconvex problem as much as possible. In our experiments, finding a local minimum seldom occurs when the subspace intersection is slight, but happens more frequently with increasing intersection. It is unclear how subspace intersection comes into the effect of local minima. The computational complexity of the basic algorithm MSS$_-$MCG is $O(n^3)$. This disadvantage hinders its applications to large numbers of samples. Since the ideal solution is basically low-rank, sufficiently utilizing the low-rank structure may be useful for reducing the computational complexity of MSS$_-$MCG. We will continue our work on this important topic.

\section*{Acknowledgment}
The work was supported in part by NSFC projects 11571312 and 91730303, and National Basic Research Program of China (973 Program)  2015CB352503.

\bibliographystyle{unsrt} 
\bibliography{ref}

\end{document}